%% file: main-IJCAI23.tex
\documentclass{article}
\usepackage{ijcai23}

\usepackage{times}
\usepackage{soul}
\usepackage{url}
\usepackage[hidelinks]{hyperref}
\usepackage[utf8]{inputenc}
\usepackage[small]{caption}
\usepackage{subcaption}
\usepackage{graphicx}
\usepackage{amsmath}
\usepackage{amsthm}
\usepackage{booktabs}
\usepackage{algorithm}
\usepackage{algorithmic}
\usepackage[switch]{lineno}
\usepackage{natbib}  

\usepackage{todonotes}
\usepackage{xcolor}
\usepackage{setspace}
\usepackage{multirow}
\usepackage{bm,bbm}
\usepackage{paralist}
\usepackage{enumitem}
\usepackage{amsmath}
\usepackage{amssymb}
\usepackage{amsthm}
\usepackage{mathtools}
\usepackage{amsfonts}
\usepackage{booktabs}       
\usepackage{nicefrac}  
\usepackage{amsthm}
\usepackage{graphicx}
\usepackage[export]{adjustbox}

\newtheorem{corollary}{Corollary}
\newtheorem{problem*}{Problem}

\newtheorem{theorem}{Theorem}
\newtheorem{lemma}{Lemma}

\newtheorem{definition}{Definition}
\newtheorem{proposition}{Proposition}

\DeclareMathOperator*{\argmax}{argmax}
\DeclareMathOperator*{\argmin}{argmin}
\newcommand{\cA}{\mathcal{A}}  
\newcommand{\cC}{\mathcal{C}}

\newcommand{\cL}{\mathcal{L}}
\newcommand{\cM}{\mathcal{M}} \newcommand{\cN}{\mathcal{N}}

 \newcommand{\cY}{\mathcal{Y}}
 \newcommand{\cX}{\mathcal{X}}

\newcommand{\EE}{\mathbb{E}} \newcommand{\RR}{\mathbb{R}}

\newcommand*{\defeq}{\stackrel{\text{def}}{=}}

\newcommand{\fpx}{p^{\leftrightarrow}_{\bm{x}}}
\newcommand{\fp}[1]{p^{\leftrightarrow}_{#1}}

\def\eqref#1{equation~\ref{#1}}
\def\Eqref#1{Equation~\ref{#1}}

\usepackage{mathtools}
\usepackage{esvect}

\makeatletter
\newcommand{\oset}[3][0ex]{%
  \mathrel{\mathop{#3}\limits^{
    \vbox to#1{\kern-2\ex@
    \hbox{$\scriptstyle#2$}\vss}}}}
\makeatother
\newcommand{\optimal}[1]{\oset{\scalebox{.5}{$\star$}}{#1}}

\newcommand{\btheta}{{\bm{\theta}}}

\title{On the Fairness Impacts of Private Ensembles Models}

\author{Cuong Tran$^{1}$  
\And Ferdinando Fioretto$^{2}$ \\
\affiliations
  $^1$
  Syracuse University\\
  $^2$
  University of Virginia\\
\emails
cutran@syr.edu,
nandofioretto@gmail.com
}

\begin{document}
\maketitle

\begin{abstract}
The Private Aggregation of Teacher Ensembles (PATE) is a machine learning framework that enables the creation of private models through the combination of multiple "teacher" models and a "student" model. The student model learns to predict an output based on the voting of the teachers, and the resulting model satisfies differential privacy. PATE has been shown to be effective in creating private models in semi-supervised settings or when protecting data labels is a priority. 
This paper explores whether the use of PATE can result in unfairness, and demonstrates that it can lead to accuracy disparities among groups of individuals. The paper also analyzes the algorithmic and data properties that contribute to these disproportionate impacts, why these aspects are affecting different groups disproportionately, and offers recommendations for mitigating these effects.
\end{abstract}

\section{Introduction}
\label{sec:introduction}
The widespread adoption of machine learning (ML) systems in decision-making processes have raised concerns about bias and discrimination, as well as the potential for these systems to leak sensitive information about the individuals whose data is used as input. These issues are particularly relevant in contexts where ML systems are used to assist in decisions processes impacting individuals' lives, such as criminal assessment, lending, and hiring. 

Differential Privacy (DP) \citep{dwork:06} is an algorithmic property that bounds the risks of disclosing sensitive information of individuals participating in a computation.
In the context of machine learning, DP ensures that algorithms can learn the relations between data and predictions while preventing them from memorizing sensitive information about any specific individual in the training data. While this property is appealing, it was recently observed that DP systems may induce biased and unfair outcomes for different groups of individuals \citep{NEURIPS2019_eugene,Fioretto:NeurIPS21a,Cuong:IJCAI21}. 
The resulting outcomes can have significant impacts on individuals with negative effects on financial, criminal, or job-hiring decisions \citep{fioretto2021decision}.
{\em While these surprising observations have become apparent in several contexts, their causes are largely understudied.}

This paper makes a step toward filling this important gap and investigates the unequal impacts that can occur when training a model using Private Aggregation of Teacher Ensembles (PATE), a state-of-the-art privacy-preserving ML framework \citep{papernot2018scalable}. PATE involves combining multiple agnostic models, referred to as \emph{teachers}, to create a \emph{student} model that is able to predict an output based on noisy voting among the teachers. This approach satisfies differential privacy and has been demonstrated to be effective for learning high-quality private models in semi-supervised settings. The paper examines which algorithmic and data properties contribute to disproportionate impacts, why these aspects are affecting different groups of individuals disproportionately, and proposes a solution for mitigating these effects. 

In summary, the paper makes several key contributions:
{\bf (1)} It introduces a fairness measure that extends beyond accuracy parity and assesses the direct impact of privacy on model outputs for different groups.
{\bf (2)} It examines this fairness measure in the context of PATE, a leading privacy-focused ML framework. 
{\bf (3)} It identifies key components of model parameters and data properties that contribute to disproportionate impacts on different groups during private training. 
{\bf (4)} It investigates the circumstances under which these components disproportionately affect different groups.
{\bf (5)} Finally, based on these findings, the paper proposes a method for reducing these unfair impacts while maintaining high accuracy.

The empirical advantages of privacy-preserving ensemble models over other frameworks, such as DP-SGD \citep{abadi:16,DBLP:conf/nips/GhaziGKMZ21,uniyal2021dpsgd}, make this work a significant and widely relevant contribution to understanding and addressing the disproportionate impacts observed in semi-supervised private learning systems. As far as we are aware, this is the first study to examine the causes of disparate impacts in privacy-preserving ensemble models.



\begin{figure*}[tb]
    \centering
    \includegraphics[width=0.9\linewidth]{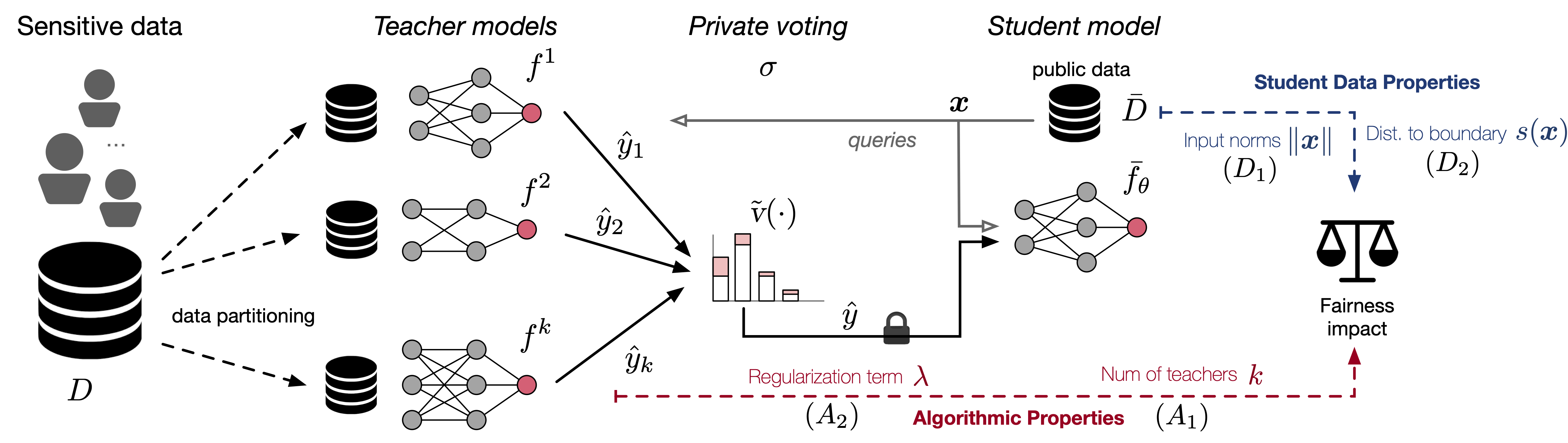}
    \caption{Illustration of PATE and aspects contributing to fairness.}
    \label{fig:scheme}
\end{figure*}

\section{Related Work} 
\label{app:related_work}


The relationship between privacy and fairness has been a topic of recent debate, as recently surveyed by \cite{Fioretto:IJCAI22a}, with several researchers raising questions about the tradeoffs 
involved \citep{ekstrand:18}. \cite{cummings:19} specifically studied the tradeoffs between 
differential privacy and equal opportunity, a fairness criterion that requires a classifier to have 
equal true positive rates for different groups. They demonstrated that it is not possible to 
simultaneously achieve $(\epsilon,0)$-differential privacy, satisfy equal opportunity, and have 
accuracy better than a constant classifier. Additionally, it has been proven that when training data 
has a long-tailed distribution, it is impossible to develop a private learning algorithm that has 
high accuracy for minority groups \citep{pmlr-v180-sanyal22a}. These findings led to asking if fair models can be created while preserving sensitive information, and have 
spurred the development of various approaches such as those presented in 
\citep{jagielski:18,mozannar2020fair,Fioretto:NeurIPS21a,tran2020differentially,tran2021differentially,ferdin2020lagrangian}. 

\citet{pujol:20} were the first to show, empirically, that 
decision tasks made using DP datasets may disproportionately affect some groups of individuals over others. These studies were complemented theoretically by \citet{Cuong:IJCAI21}.
Similar observations were also made in the context of model learning. \citet{NEURIPS2019_eugene} empirically observed that the accuracy of a DP model trained using DP-Stochastic Gradient Descent (DP-SGD) decreased disproportionately across groups causing larger negative impacts to the underrepresented groups. \citet{farrand2020neither} and \citet{uniyal2021dpsgd} reached similar conclusions and showed that this disparate impact was not limited to highly imbalanced data. 

This paper builds on this body of work and their important empirical observations.  It provides an analysis of the causes of unfairness in the context of private learning ensembles, a significant privacy-enhancing ML system, and introduces guidelines for mitigating these effects.

\section{Preliminaries: Differential Privacy}
\label{sec:preliminaries}

Differential privacy (DP) is a strong privacy notion stating that the probability of any output does not change much when a record is added or removed from a dataset, limiting the amount of information that the output reveals about any individual.  
The action of adding or removing a record from a dataset $D$, resulting in a new dataset $D'$, defines the notion of \emph{adjacency}, denoted $D \sim D'$.
\begin{definition}[\cite{dwork:06}]
  \label{dp-def}
  A mechanism $\cM \!:\! \mathcal{D} \!\to\! \mathcal{R}$ with domain $\mathcal{D}$ and range $\mathcal{R}$ satisfies $(\epsilon, \delta)$-differential privacy, if, for any two adjacent inputs $D \sim D' \!\in\! \mathcal{D}$, and any subset of output responses $R \subseteq \mathcal{R}$:
  \[
      \Pr[\cM(D) \in R ] \leq  e^{\epsilon} 
      \Pr[\cM(D') \in R ] + \delta.
  \]
\end{definition}
\noindent 
Parameter $\epsilon > 0$ describes the \emph{privacy loss} of the algorithm, with values close to $0$ denoting strong privacy, while parameter 
$\delta \in [0,1)$ captures the probability of failure of the algorithm to satisfy $\epsilon$-DP. 
The global sensitivity $\Delta_\ell$ of a real-valued 
function $\ell: \mathcal{D} \to \mathbb{R}$ is defined as the maximum amount 
by which $\ell$ changes  in two adjacent inputs:
\(
  \Delta_\ell = \max_{D \sim D'} \| \ell(D) - \ell(D') \|.
\)
In particular, the Gaussian mechanism, defined by
\(
    \mathcal{M}(D) = \ell(D) + \mathcal{N}(0, \Delta_\ell^2 \, \sigma^2), 
\)
\noindent where $\mathcal{N}(0, \Delta_\ell^2\, \sigma^2)$ is 
the Gaussian distribution with $0$ mean and standard deviation 
$\Delta_\ell^2\, \sigma^2$, satisfies $(\epsilon, \delta)$-DP for 
$\delta \!>\! \frac{4}{5} \exp(-(\sigma\epsilon)^2 / 2)$ 
and $\epsilon \!<\! 1$ \citep{dwork:14}.

\section{Problem Settings and Goals}
\label{sec:problem}

This paper considers a \emph{private} dataset $D$ consisting of $n$ individuals' data $(\bm{x}_i, y_i)$, with $i \!\in\! [n]$, drawn i.i.d.~from an unknown distribution $\Pi$. Therein, $\bm{x}_i \!\in\! \cX$ is a sensitive feature vector containing a protected group attribute $\bm{a}_i \!\in\! \cA \!\subset\! \cX$, and $y_i \!\in\! \cY = [C]$ is a $C$-class label. 
{For example, consider a classifier that needs to predict criminal defendants’ recidivism. The data features $\bm{x}_i$ may describe the individual's demographics, education, and crime committed, the protected attribute $\bm{a}_i$ may describe the individual's gender or ethnicity, and $y_i$ whether the individual has high risk to reoffend.} 

This paper studies the fairness implications arising when training private semi-supervised transfer learning models. 
The setting is depicted in Figure \ref{fig:scheme}. We are given an ensemble of \emph{teacher} models $\bm{T} \!=\! \{f^j\}_{j=1}^k$, with each $f^j \!:\! \cX \!\to\! \cY$ trained on a non-overlapping portion $D_i$ of $D$. This ensemble is used to transfer knowledge to a \emph{student} model $\bar{f}_\btheta \!:\! \cX \!\to\! \cY$, where $\btheta$ is a vector of real-valued parameters. 

The student model $\bar{f}$ is trained using a \emph{public} dataset $\bar{D} \!=\! \{\bm{x}_i\}_{i=1}^m$ with samples drawn i.i.d.~from the same distribution $\Pi$ considered above but whose labels are unrevealed. 
We focus on learning {classifier} $\bar{f}_\btheta$ using  knowledge transfer from the teacher model ensemble $\bm{T}$ while guaranteeing the privacy of each individual's data $(\bm{x}_i, y_i) \!\in\! D$. 
The sought model is learned by minimizing the regularized empirical risk function with loss $\ell \!:\! \cY \times \cY \!\to\! \mathbb{R}_+$:
\begin{align}
\label{eq:ERM}
    \btheta^* &= \argmin_{\bm{\btheta}} \cL(\btheta; \bar{D}, \bm{T})  + \lambda \| \btheta \|^2 \\ 
      &= \sum_{\bm{x} \in \bar{D}} 
    \ell\left( \bar{f}_\btheta(\bm{x}), \textsl{v}\left(\bm{T}(\bm{x}) \right)\right)  + \lambda \| \btheta \|^2,
\end{align}
where $\textsl{v} \!:\! \cY^k \!\to\! \cY$ is a \emph{voting scheme} used to
decide the prediction label from the ensemble $\bm{T}$, with 
$\bm{T}(\bm{x})$ used as a shorthand for $\{f^j(\bm{x})\}_{j=1}^k$, and $\lambda>0$ is a regularization term. 

We focus on DP classifiers that protect the disclosure of the individual's data and analyzes the fairness impact (as defined below) of privacy on different groups of individuals. 

\paragraph{Privacy.}
\emph{Privacy} is achieved by using a DP version $\tilde{\textsl{v}}$ of the voting function $\textsl{v}$: 
\begin{equation}
\label{eq:noisy_max}
\textstyle    
  \tilde{\textsl{v}}(\bm{T}(\bm{x})) \!=\! \argmax_c \{ \#_c(\bm{T}(\bm{x})) \!+\! \cN(0, \sigma^2)\}
\end{equation}
which perturbs the reported counts $\#_c(\bm{T}(\bm{x}))\!=\!|\{j\!:\!j \!\in\![k], f^j(\bm{x}) \!=\! c\}|$ for class $c \!\in\! \cC$ with zero-mean Gaussian and standard deviation $\sigma$.  
The overall approach, called \emph{PATE} \citep{papernot2018scalable}, guarantees $(\epsilon, \delta)$-DP, with privacy loss scaling with the magnitude of the standard deviation $\sigma$ and the size of the public dataset $\bar{D}$. 
A detailed review of the privacy analysis of PATE is reported in Appendix C of \citep{Tran:PPAI22}.
Throughout the paper, the privacy-preserving parameters of the model $\bar{f}$ trained with noisy voting $\tilde{\textsl{v}}(\bm{T}(\bm{x}))$ are denoted with $\tilde{\bm{\btheta}}$. 

\paragraph{Fairness.}
One widely used metric for measuring utility in private learning is the \emph{excess risk} \citep{ijcai2017548}, which is defined as the difference between the private and non-private risk functions:
\begin{equation}
    \label{def:excessiver_risk}
  R(S, \bm{T}) \defeq \EE_{\tilde{\btheta}}  
  \left[ \cL(\tilde{\btheta}; S, \bm{T}) \right] 
    - \cL({\btheta}^*; S, \bm{T}),
\end{equation}
where the expectation is taken over the randomness of the private mechanism, $S$ is a subset of $\bar{D}$, $\tilde{\btheta}$ is the private student model's parameters, and 
${\btheta}^* \!=\! \argmin_\btheta \cL(\btheta; \bar{D}, \bm{T}) + \lambda \| \btheta\|^2$.

In this paper, the unfairness introduced by privacy in the learning task is measured using the difference in excess risks of each protected subgroup. This notion is significant because it captures the unintended impact of privacy on task accuracy for a given group, and it relates to the concept of accuracy parity, a standard metric in fair and private learning. More specifically, the paper focuses on measuring the excess risk 
$R(\bar{D}_{\leftarrow a}, \bm{T})$ for groups $a \in \cA$, where $\bar{D}_{\leftarrow a}$ is the subset of $\bar{D}$ containing only samples from a group $a$. 
We use the shorthand $R(\bar{D}_{\leftarrow a})$ to refer to $R(\bar{D}_{\leftarrow a}, \bm{T})$ and assume that the private mechanisms are non-trivial, i.e., they minimize the population-level excess risk $R(\bar{D})$.

\begin{definition}
Fairness is measured as the highest excess risk difference among all groups:
\begin{equation}
\label{eq:risk_gap} 
  \xi(\bar{D}) = \max_{a, a' \in \cA} 
   R(\bar{D}_{\leftarrow a}) - R(\bar{D}_{\leftarrow a'}).
\end{equation}
    
\end{definition}
Notice how this definition of fairness relates to the concept of accuracy parity \citep{NEURIPS2019_eugene}, which measures the disparity of task accuracy across groups, when the adopted loss $\ell$ is a 0/1-loss. All the experiments in the paper use, in fact, this 0/1-loss, while the theoretical analysis considers general differentiable loss functions. Additional details regarding this fairness definition and its relations with other fairness notions can be found in Appendix A of \citep{Tran:PPAI22}.

\section{PATE Fairness Analysis: Roadmap}
\label{sec:roadmap}

The objective of this paper is to identify the factors that cause unfairness in PATE and understand why they have this effect. The following sections isolate these key factors, which will be divided into two categories: \emph{algorithm parameters} and \emph{public student data characteristics}. The theoretical analysis assumes that, for a group $a \in \cA$, the group loss function $\cL(\btheta; D_{\leftarrow a}, \bm{T})$ is convex and $\beta_a$-smooth with respect to the model parameters $\btheta$ for some $\beta_a \geq 0$. However, the evaluation does not impose any restrictions on the form of the loss function. A detailed description of the experimental settings can be found in Appendix D, and the proofs of all theorems are included in Appendix A of \citep{Tran:PPAI22}.

\paragraph{A fairness bound.}
We start by introducing a bound on the model disparity, which will be crucial for identifying the algorithm and data characteristics that contribute to unfairness in PATE. 
Throughout the paper, we refer to the  quantity 
$\Delta_{\tilde{\btheta}} \defeq \| \tilde{\btheta} - \btheta^* \|$ as to \emph{model deviation due to privacy}, or simply \emph{model deviation}, as it captures the effect of the private teachers' voting on the student learned model. Here, ${\btheta}^*$ and $\tilde{\btheta}$ represent the parameters of student model $\bar{f}$ learned using a clean or noisy voting scheme, respectively.

\begin{figure}[t]
    \centering
    \includegraphics[width=180pt]{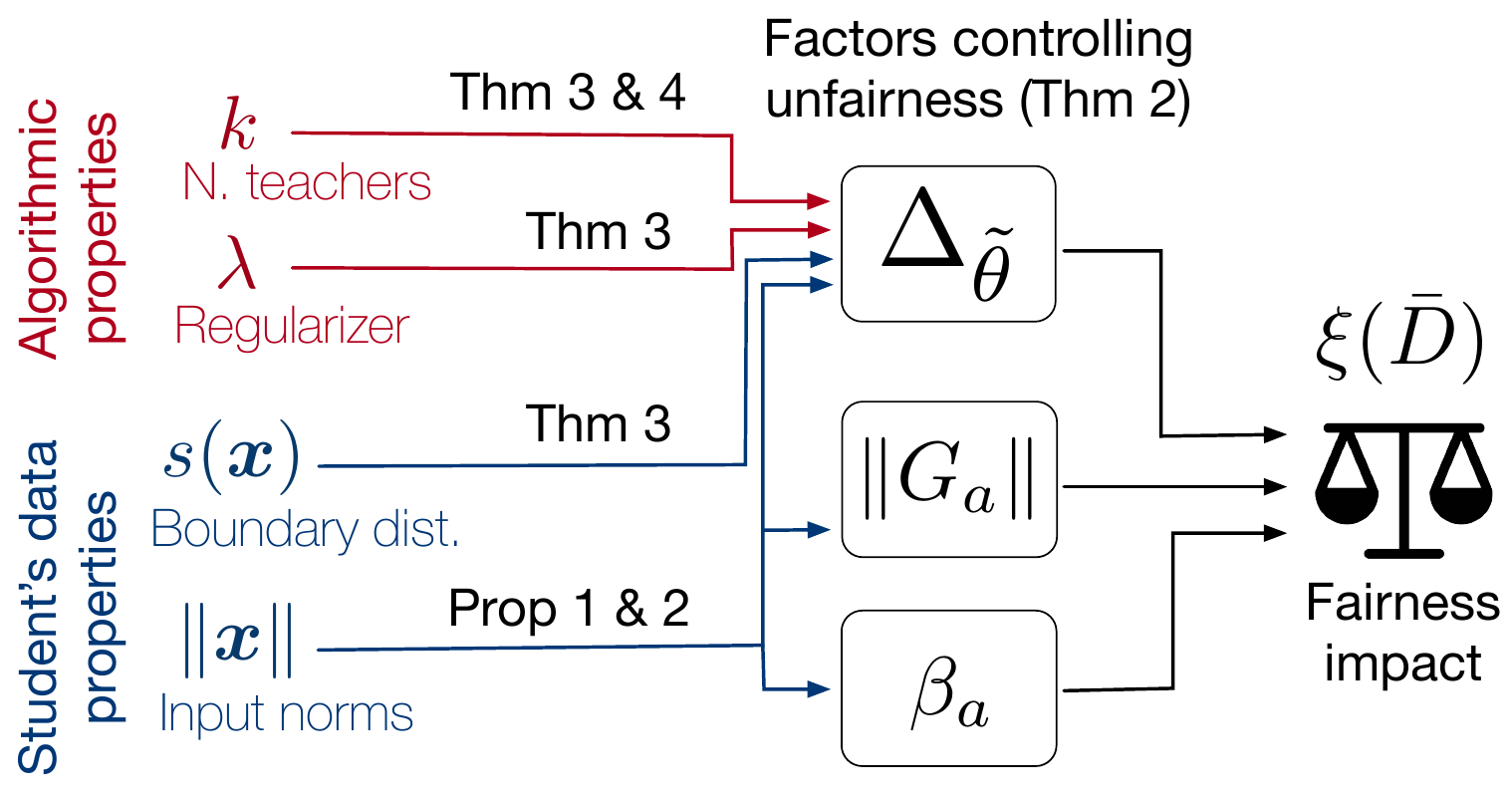}
    \caption{Factors impacting PATE fairness.}
    \label{fig:causal}
\end{figure}

\begin{theorem}
\label{thm:2}
The model fairness is upper bounded as: 
\begin{equation}
\label{eq:fair_ub_1}
\xi(\bar{D}) \leq 2 \max_a \| G_a \|\; \mathbb{E}\left[\Delta_{\tilde{\btheta}} \right]
       + \nicefrac{1}{2} \max_a {\beta_a}\; \mathbb{E}\left[ \Delta_{\tilde{\btheta}}^2 \right],
\end{equation}
where $G_a =  \EE_{\bm{x} \sim \bar{D}_{\leftarrow a}}\left[
  \nabla_{\btheta^*} 
        \ell(\bar{f}_{\btheta^*}(\bm{x}),y) \right]$ is the gradient of the group loss evaluated at $\btheta^*$, and $\Delta_{\tilde{\btheta}}$ and  $\Delta_{\tilde{\btheta}}^2$ capture the first and second order statistics of the model deviation. 
\end{theorem}
The above illustrates that the model unfairness is proportionally 
regulated by three direct factors: 
{\bf (1)} the model deviation $\Delta_{\tilde{\btheta}}$, 
{\bf (2)} the maximum gradient norm $\max_a \|G_a\|$ among all groups, and 
{\bf (3)} the largest smoothness parameter $\max_a \beta_a$ among all groups. 

The paper delves into which {\bf A}lgorithms' parameters and {\bf D}ata characteristics affect the factors that contribute to model unfairness. 
Within the {\bf A}lgorithm's parameters, in addition to the privacy variable  $\epsilon$ (captured by the noise parameter $\sigma$), the paper identifies two factors having a direct impact on fairness: {($\bm{A_1}$)} the regularization term $\lambda$ associated with the student risk function and {($\bm{A_2}$)} the size $k$ of the teachers' ensemble.
Regarding the public student {\bf D}ata's characteristics, the paper shows that {($\bm{D_1}$)} the magnitude of the sample input norms $\|\bm{x}\|$ and {($\bm{D_2}$)} the distance of a sample to the decision boundary (denoted $s(\bm{x})$) are key factors that can exacerbate the excess risks induced by the student model. 
The relationships between these factors and how they impact model fairness are illustrated in Figure~\ref{fig:causal}.

Several aspects of the analysis in this paper rely on the following definition. 
\begin{definition}
Given a data sample $(\bm{x}, y) \!\in\! D$, for an ensemble $\bm{T}$ and voting scheme $\textsl{v}$, the \emph{flipping probability}  is: 
\[
    \fp{\bm{x}} \defeq  \Pr\left[ 
    \tilde{\textsl{v}}(\bm{T}(\bm{x})) \neq \textsl{v}(\bm{T}(\bm{x})) 
    \right].
\]
\end{definition}
\noindent It connects the \emph{voting confidence} of the teacher ensemble with the perturbation induced by the private voting scheme and will be useful in the fairness analysis introduced below.

The theoretical results presented in the following sections are supported and corroborated by empirical evidence from tabular datasets (UCI Adults, Credit card, Bank, and Parkinsons) and an image dataset (UTKFace). These results were obtained using feed-forward networks with two hidden layers and nonlinear ReLU activations for both the ensemble and student models for tabular data, and CNNs for image data. All reported metrics are the average of 100 repetitions used to compute empirical expectations and report 0/1 losses, \emph{which capture the concept of accuracy parity}. While the paper provides a brief overview of the empirical results to support the theoretical claims, extended experiments and more detailed descriptions of the datasets can be found in Appendix D of \citep{Tran:PPAI22}.

\section{Algorithm's Parameters}
\label{sec:alg_params}

This section analyzes the algorithm's parameters that affect the disparate impact of the student model outputs. The fairness analysis reported in this section assumes that the student model loss $\ell(\cdot)$ is convex and \emph{decomposable}:

\begin{definition}
\label{def:1}
A function $\ell(\cdot) $ is \emph{decomposable} if there exists a parametric function $h_{\btheta} \!:\! \cX \!\to\! \RR$, a constant real number $c$, and a function
$z \!:\! \RR \!\to\! \RR$, such that, for $\bm{x} \!\in\! \cX$, 
and $y \!\in\! \cY$:
\begin{equation}
\label{eq:decomposable}
    \ell(f_{\btheta}(\bm{x}), y) = z(h_{\btheta}(\bm{x})) 
    + c \, y\, h_{\btheta}(\bm{x}).
\end{equation}
\end{definition}

A number of loss functions commonly adopted in ML, including the logistic loss (used in our experiments) 
or the least square loss function, are decomposable \citep{patrini2014almost}. 
Additionally, while restrictions are commonly imposed on the loss functions to render the analysis tractable, our findings are empirically validated on non-linear models.

It is important to recall that the model deviation is a central factor that proportionally controls the unfairness of PATE (Theorem~\ref{thm:2}). 
In the following, we provide a useful bound on the model deviation and highlight its relationship with key algorithm parameters.

\begin{theorem}
\label{thm:3}
Consider a student model $\bar{f}_{\btheta}$ trained with a convex and decomposable loss 
function $\ell(\cdot)$. Then, the first order statistics of the model deviation is upper bounded as:
\begin{equation}
    \mathbb{E}\Big[ \Delta_{\tilde{\btheta}} \Big] 
    \leq \frac{|c|}{m\lambda} \left[ \sum_{\bm{x} \in \bar{D}} p^{\leftrightarrow}_{\bm{x}} \| G_{\bm{x}}^{\max}\| \right],
\end{equation}
where $c$ is a real constant and
$G_{\bm{x}}^{\max}= \max_{\btheta}\| \nabla_{\btheta} h_{\btheta}(\bm{x}) \|$ 
represents the maximum gradient norm distortion introduced by a 
sample $\bm{x}$. Both $c$ and $h$ are defined as in Equation~\ref{eq:decomposable}. 
\end{theorem}


\noindent 
The proof relies on $\lambda$-strong convexity of the loss function ${\cL(\cdot)} + \lambda \| \btheta \| $ (see Appendix B of of \citep{Tran:PPAI22}) and its tightness is demonstrated empirically in Appendix D.2 of \citep{Tran:PPAI22}. 
Theorem~\ref{thm:3} reveals how the student model changes due to privacy and relates it with two mechanism-dependent components: {\bf (1)} the regularization term $\lambda$ of the empirical risk function 
$\cL(\btheta, \bar{D}, \bm{T})$ (see \Eqref{eq:ERM}), and  {\bf (2)} the flipping probability $\fpx$, which, as it will be shown later, is heavily controlled by the size $k$ of the teacher ensemble. 
These mechanisms-dependent components and the focus of this section, while data-dependent components, including those related to the maximum gradient norm distortion $G^{\max}_{\bm{x}}$ are discussed to Section \ref{sec:data_prop}.

\paragraph{$\bm{A_1}$:~The impact of the regularization term $\lambda$.} The first
 immediate observation of Theorem \ref{thm:3} is that variations of
 the regularization term $\lambda$ can increase or decrease the difference between the private and non-private student model parameters. 
 Since the model deviation $\EE[ \Delta_{\tilde{\btheta}}]$ has adirect relationship with the fairness goal (see the first term of RHS of \Eqref{eq:fair_ub_1} in Theorem \ref{thm:2}) {\em the regularization term affects the disparate impact of the privacy-preserving student model}. 
 These effects are further illustrated in Figure \ref{fig:2} (top).
\begin{figure}[!bt]
    \centering
    \includegraphics[width=\linewidth,height=80pt]{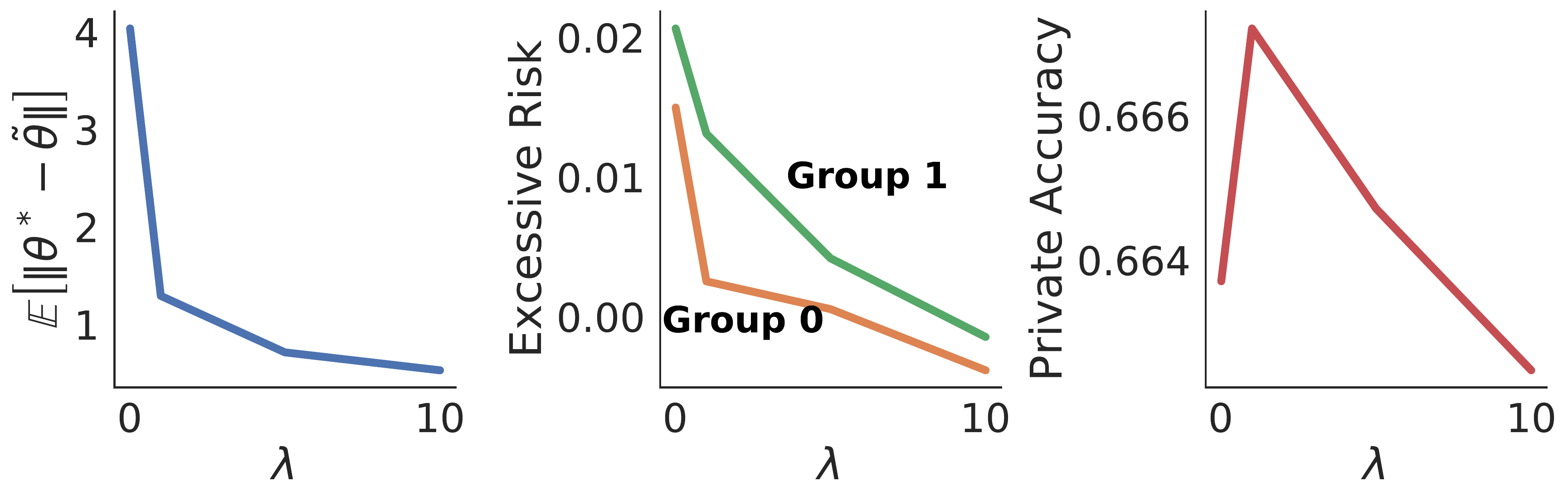}
    \includegraphics[width=\linewidth,height=80pt]{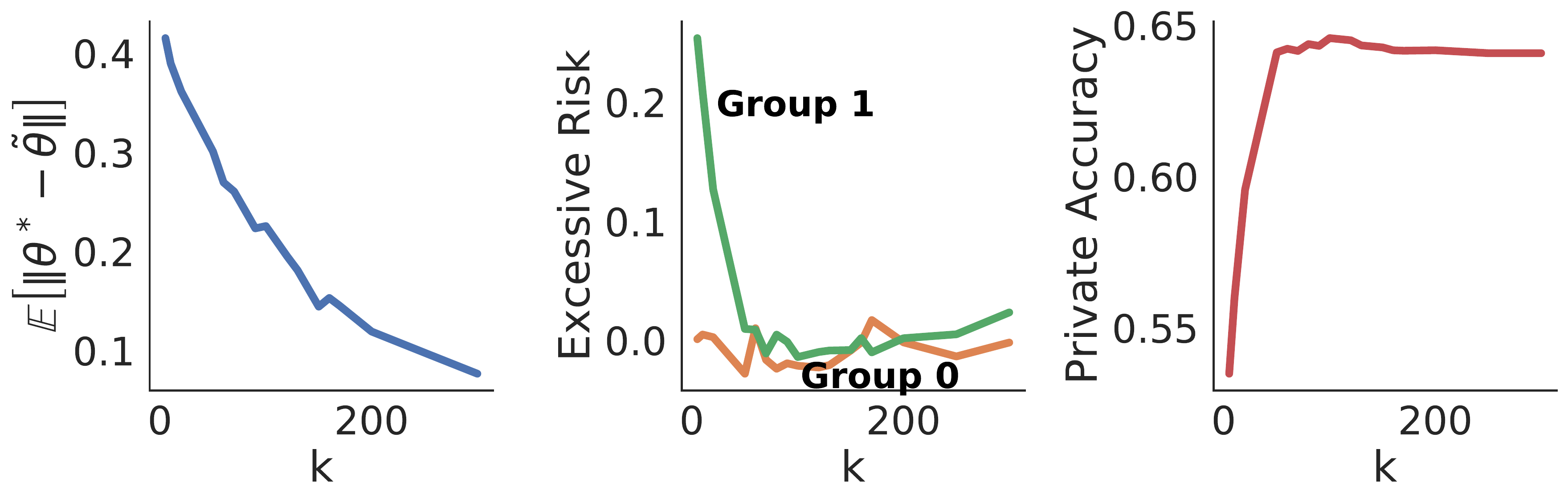}
    \caption{Credit card dataset with $\sigma\!=\!50, k\!=\!150$ (top) and $\lambda\!=\!100$ (bottom). 
    Expected model deviation (left), 
    excess risk (middle), and model accuracy (right) as a function 
    of the regularization term (top) and ensemble size (bottom).}
    \label{fig:2}
\end{figure}
The figure shows how increasing $\lambda$ reduces the expected difference between the privacy-preserving and original model parameters 
$\EE[\Delta_{\tilde{\btheta}} ]$ (left),  as well as the excess risk $R(\bar{D}_{\leftarrow a})$ 
difference between groups $a=0$ and $a=1$ (middle). 
Note, however, that while larger $\lambda$ values may reduce the model unfairness, they can hurt the  model's accuracy, as shown in the right plot. 
The latter is an intuitive and recognized effect of large regularizers \citep{mahjoubfar2017deep}.

\paragraph{$\bm{A_2}$:~The impact of the teachers ensemble size $k$.} 
Next, we consider the relationship between the ensemble size $k$ and the resulting private model's fairness. The following result relates the size of the ensemble with its voting confidence.
\begin{theorem}
\label{thm:4}
For a sample $\bm{x} \!\in\! \bar{D}$ let the teacher models 
outputs $f^i(\bm{x})$ be in agreement, $\forall i \in [k]$. 
The flipping probability $\fpx$ is given by 
\(
     \fpx = 1 - \Phi(\frac{k}{\sqrt{2} \sigma}),
\)
where $\Phi(\cdot)$ is the CDF of the standard Normal distribution 
and $\sigma$ is the standard deviation in the Gaussian mechanism.
\end{theorem}
The proof is based on the properties of independent Gaussian random variables.
This analysis shows that the ensemble size $k$ (as well as the privacy parameter $\sigma$) directly affects the outcome of the teacher voting and, therefore, the model deviation and its disparate impact. 
The theorem shows that larger $k$ values correspond to smaller flipping probability $\fpx$. {\em In conjunction with Theorem~\ref{thm:2}, this suggests that the model deviation due to privacy and the excess risks for various groups are inversely proportional to the ensemble size $k$.}

\begin{figure}[t]
    \centering
    \includegraphics[width=120pt]{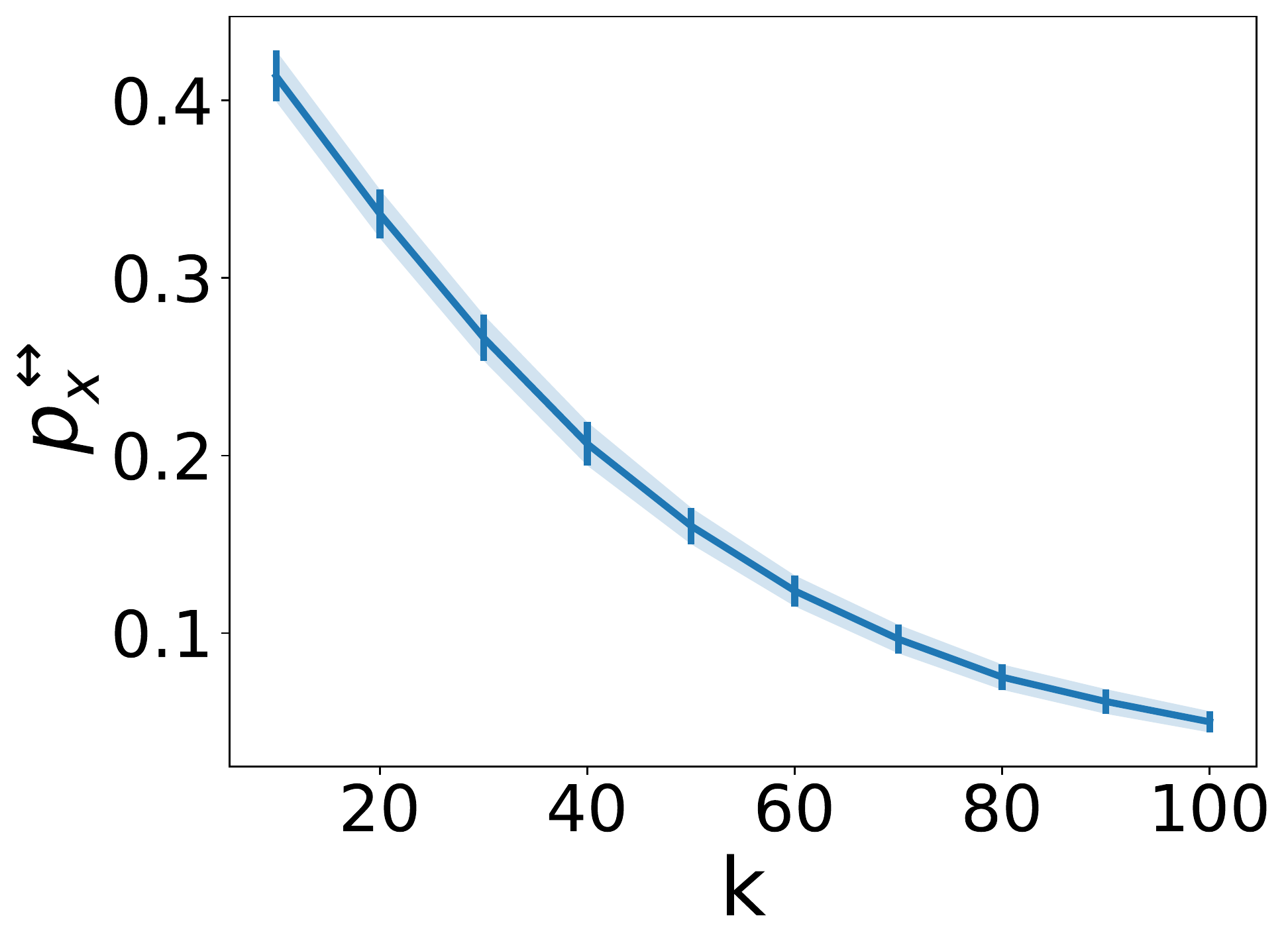}
    \includegraphics[width=120pt]{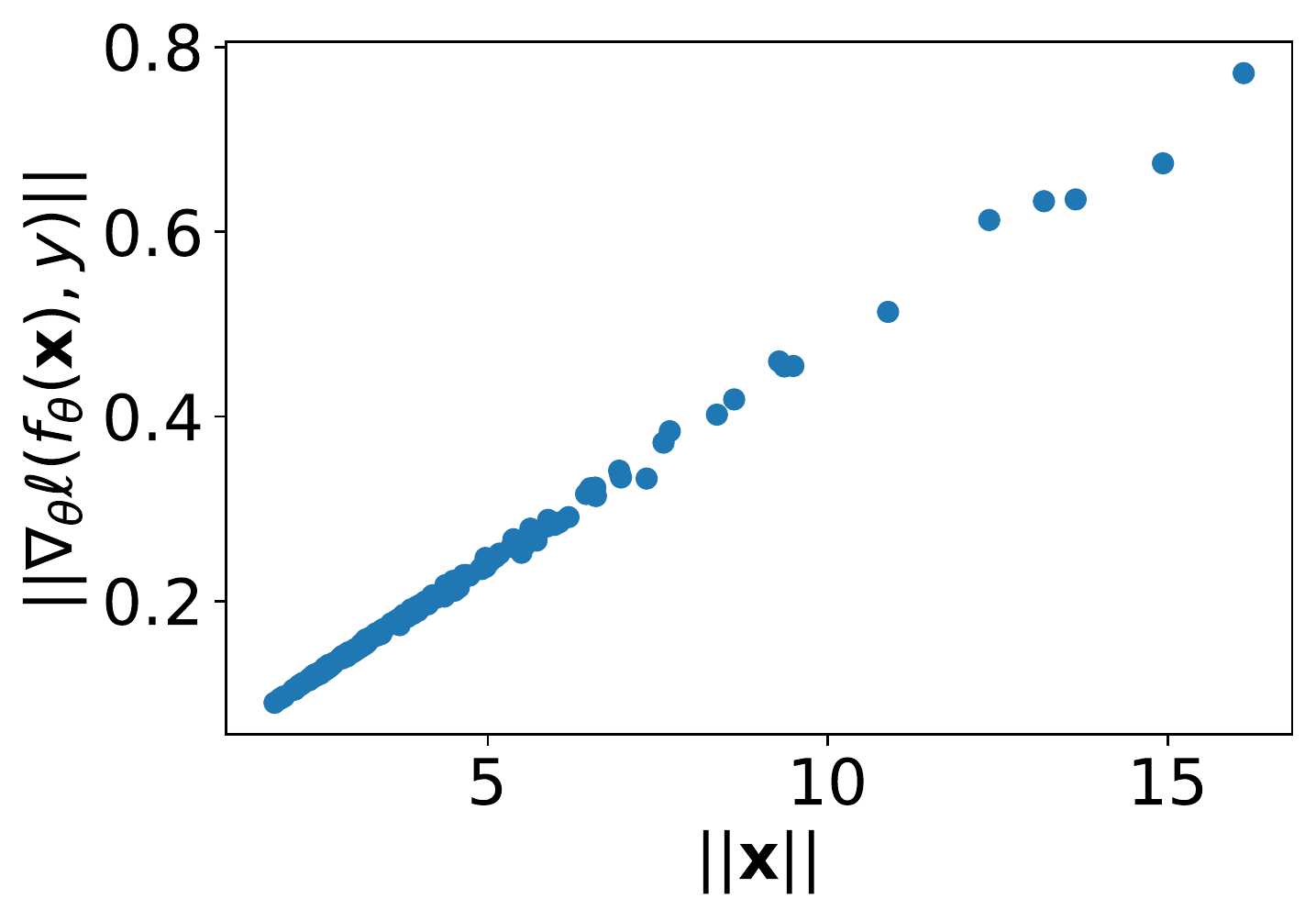}
    \caption{\small Credit-card: Average flipping probability $\fpx$ 
    for samples $\bm{x} \in \bar{D}$ as a function of the ensemble size $k$ (left) and the relation between gradient and input norms (right).}
    \label{fig:flipping_pr_k}
    \label{fig:grad_inp_corr}
\end{figure}
Figure \ref{fig:flipping_pr_k} (top) illustrates the relationship between the number $k$ of teachers and the flipping probability $\fpx$ of the ensemble,  indicating that larger ensembles result in smaller flipping probabilities. 
It is worth noting that in these experiments, {\em different teachers may have different agreements on each sample}, thus this result generalizes the one presented in Theorem~\ref{thm:4}. Additionally, Figure \ref{fig:2} (bottom) 
shows that increasing $k$ reduces the expected model deviation (left), reduces the group excess risk difference (middle), and increases the accuracy of the model $\bar{f}$  (right). 
Similar to theregularization term $\lambda$, large values $k$ can decrease the accuracy of the (private and non-private) models. 
This behavior is related to the bias-variance tradeoff imposed on the growing ensemble with less training data available to each teacher.

\noindent 
This section concludes with a useful corollary of Theorem~\ref{thm:3}. 
\begin{corollary}[Theorem~\ref{thm:3}]
\label{cor:1}
For a \emph{logistic regression} classifier $\bar{f}_{\btheta}$, the
model deviation is upper bounded as:
\begin{equation}
\label{eq:8}
 \mathbb{E}\left[ \Delta_{\tilde{\btheta}} \right] 
    \leq 
    \frac{1}{m\lambda} \left[ \sum_{\bm{x} \in \bar{D}} \fp{\bm{x}} \| \bm{x} \| \right]. 
 \end{equation}
\end{corollary}

This result highlights the presence of a relationship between gradient norms and input norms, which is further illustrated in Figure~\ref{fig:grad_inp_corr} (bottom). The plot shows a strong correlation between inputs and their associated gradient norms. The result also shows that samples with large norms can significantly impact fairness, emphasizing the importance of considering the characteristics of the student data, which are the subject of study in the next section.

{In summary, the regularization parameter $\lambda$ and the ensemble size $k$ are two key algorithmic parameters that, by bounding the model deviation $\Delta_{\tilde{\theta}}$, can control the disparate impacts of the student model. These relations are further illustrated in the causal graph in Figure \ref{fig:scheme}.}

\section{Student's Data Properties}
\label{sec:data_prop}

Having examined the algorithmic properties of PATE affecting fairness, 
this section turns on analyzing the role of certain characteristics of the student data in regulating the disproportionate impacts of of the algorithm.
The results below will show that the norms of the student's data samples and their distance to the decision boundary can significantly impact the excess risk in PATE. This is particularly interesting as it dispels the notion that unfairness in these models is solely due to imbalanced training data.
The following is a second corollary of Theorem \ref{thm:3} and bounds the second order statistics of the model deviation to privacy.
\begin{corollary}[Theorem~\ref{thm:3}]
\label{cor:2}
Given the same settings and assumption of Theorem \ref{thm:3}, it follows:
\begin{equation}
\label{eq:9}
    \mathbb{E}\left[ \Delta_{\tilde{\btheta}}^2 \right] 
    \leq \frac{|c|^2}{m \lambda^2} \left[ \sum_{\bm{x} \in \bar{D}} p^{\leftrightarrow 2}_{\bm{x}} \| G_{\bm{x}}^{\max}\|^2 \right].
\end{equation}
\end{corollary}
\noindent
Note that, similarly to what shown by Corollary \ref{cor:1}, when $\bar{f}_\btheta$ is a logistic regression model, the gradient norm $ \| G^{\max}_{\bm{x}}\|$ above can be substituted with the input norm $\|\bm{x}\|$. 

The rest of the section focuses on logistic regression models, however, as our experimental results illustrate, the observations extend to complex nonlinear models as well.

\paragraph{($\bm{D_1}$): The impact of the data input norms.}
First notice that the norm $\|\bm{x}\|$ of a sample $\bm{x}$ strongly influences the model deviation controlling quantity $\Delta_{\tilde{\btheta}}$ as already observed by Corollaries \ref{cor:1} and \ref{cor:2}. 
\begin{figure}[t]
    \centering
    \includegraphics[width=160pt,height=90pt]{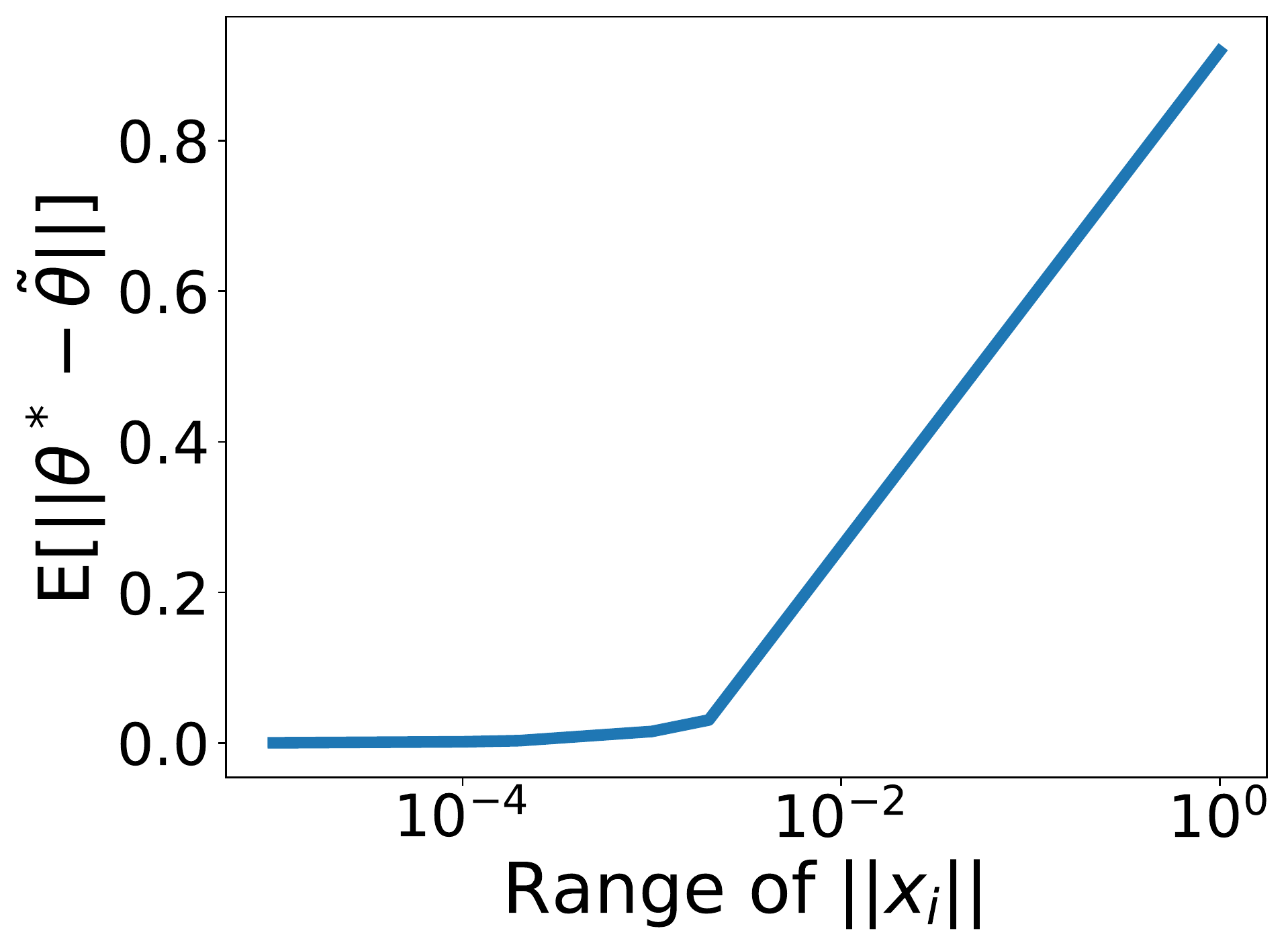}\\
    \vspace{-2pt}\hspace*{-10pt}
    \includegraphics[width=160pt,height=90pt]{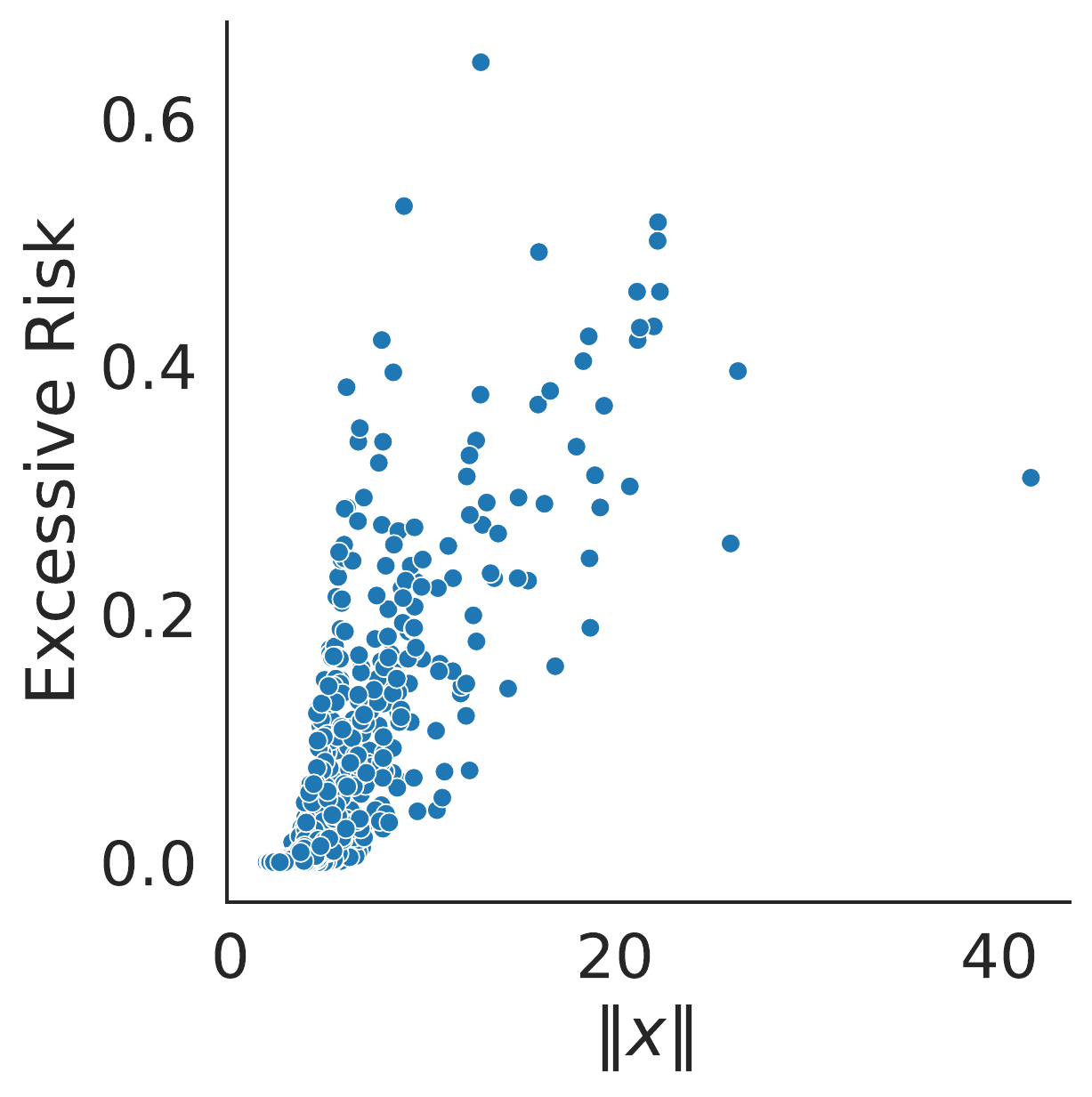}\\
    \caption{\small {\sl Credit}: Relation between input norms and model deviation~(top) and 
    excess risk (bottom).}
    \label{fig:impact_norm_2_exp_diff}
\end{figure}
This aspect is further highlighted in Figure \ref{fig:impact_norm_2_exp_diff} (top), which illustrates that samples with high input norms have a significant impact on the model deviation. {\em As a result, these samples may contribute to the unfairness of the model, as per Theorem \ref{thm:2}.} 

Next, recall that the group gradient norms $G_a$ have a proportional effect on the upper bound of the model unfairness, as shown in Theorem \ref{thm:2}. These norms also have an effect on the excess risk $R(\bar{D}_{\leftarrow a})$, as shown in Lemma \ref{a:thm:1}, Appendix B of \citep{Tran:PPAI22}
The following results reveal a connection between the gradient norm for a sample $\bm{x} \in \bar{D}$ and its associated input norm, and how these factors relate to the unfairness observed in the student model.

\begin{proposition}\label{ex:grad_logreg}
Consider a  logistic regression binary classifier $\bar{f}_{\btheta}$
with cross entropy loss function $\ell$.
For a given sample $(\bm{x}, a, y) \in \bar{D}$, the gradient
 $\nabla_{\btheta^*}\ell(\bar{f}_{\btheta^*}(\bm{x}),y)$ is given by:
\[
\nabla_{\btheta^*}\ell(\bar{f}_{\btheta^*}(\bm{x}),y) 
= (\bar{f}_{\btheta^*}(\bm{x}) -y \big) \otimes \bm{x},
\]
where $\otimes$ expresses the Kronecker product. 
\end{proposition}
\noindent 
Thus, the relation above suggests that the \emph{input norm} of data samples play a key role in controlling their associated excess risk, and, thus, that of the group in which they belong to. 
This aspect can be appreciated in Figure \ref{fig:impact_norm_2_exp_diff} (bottom), which shows a strong correlation between the input norms and excess risk. This observation is significant because it challenges the common belief that unfairness is solely caused by imbalances in group sizes. Instead, it suggests that the properties of the data itself directly contribute to unfairness.

Finally, note that the smoothness parameter $\beta_{a}$ reflects the local flatness of the loss function in relation to samples from a group $a$. An extension of the results from \cite{shi2021aisarah} is provided to derive $\beta_a$ for logistic regression classifiers, further illustrating the connection between the input norms $\|\bm{x}\|$ of a group $a \in \cA$ and the smoothness parameters $\beta_{a}$.

\begin{proposition}\label{ex:hessian_logreg}
Consider again a binary logistic regression as in Proposition 
\ref{ex:grad_logreg}. The smoothness parameter $\beta_a$ for a group $a \in \cA$   is given by: $ \beta_a = 0.25  \max_{\bm{x} \in D_a} \| \bm{x} \|^2.$
\end{proposition}
\noindent
Therefore, Propositions \ref{ex:grad_logreg} and \ref{ex:hessian_logreg} show that groups with large (small) inputs' norms tend to have large (small) gradient norms and smoothness parameters. Since these factors influence the model deviation, they also affect the associated excess risk, leading to larger disparate impacts. 
An extended analysis of the above claim is provided in Appendix D.7 of \citep{Tran:PPAI22}.

\begin{figure}[t]
    \centering
    \includegraphics[width=160pt, height=90pt]{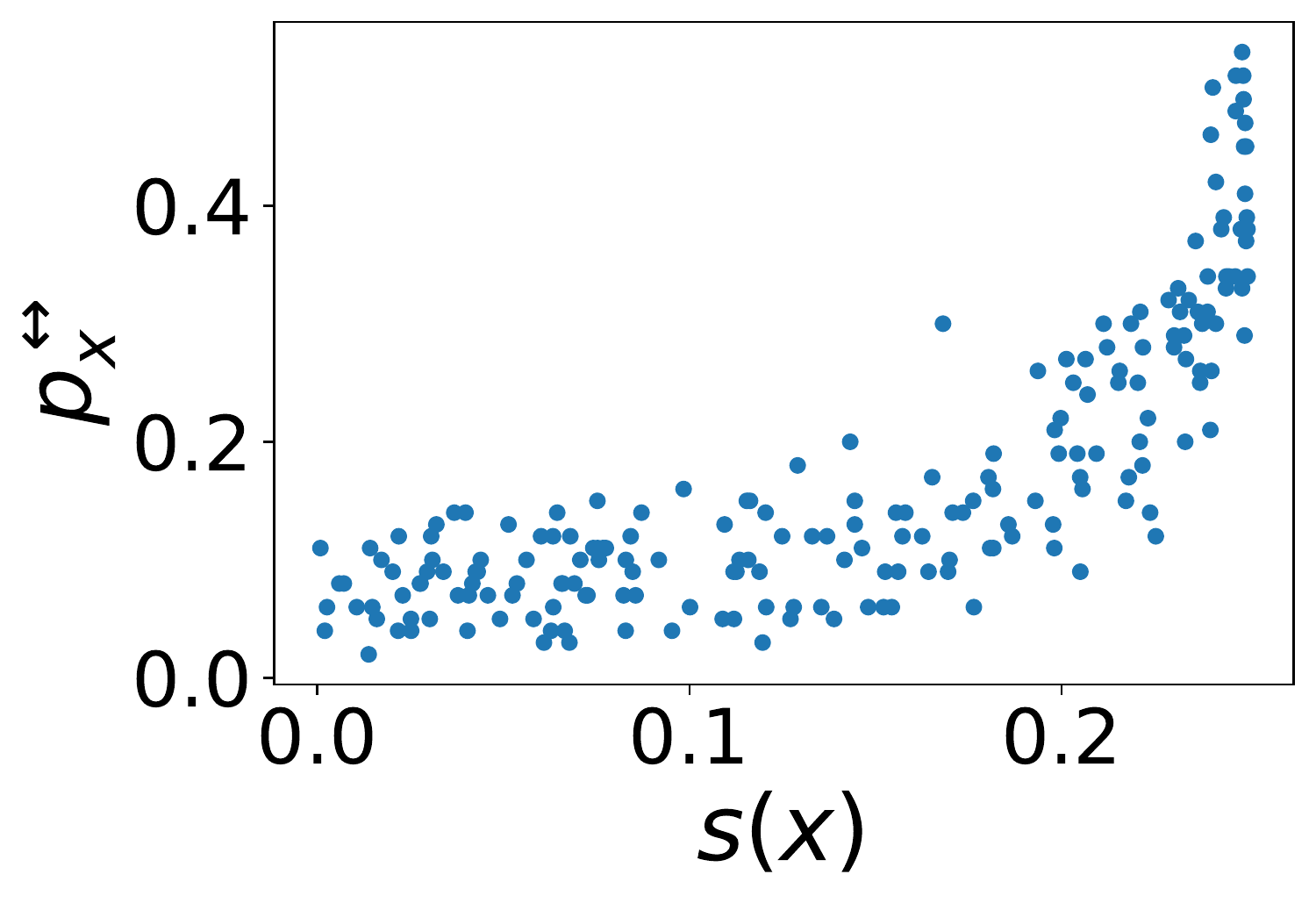}\\
    \vspace{-2pt}
    \includegraphics[width=160pt, height=90pt]{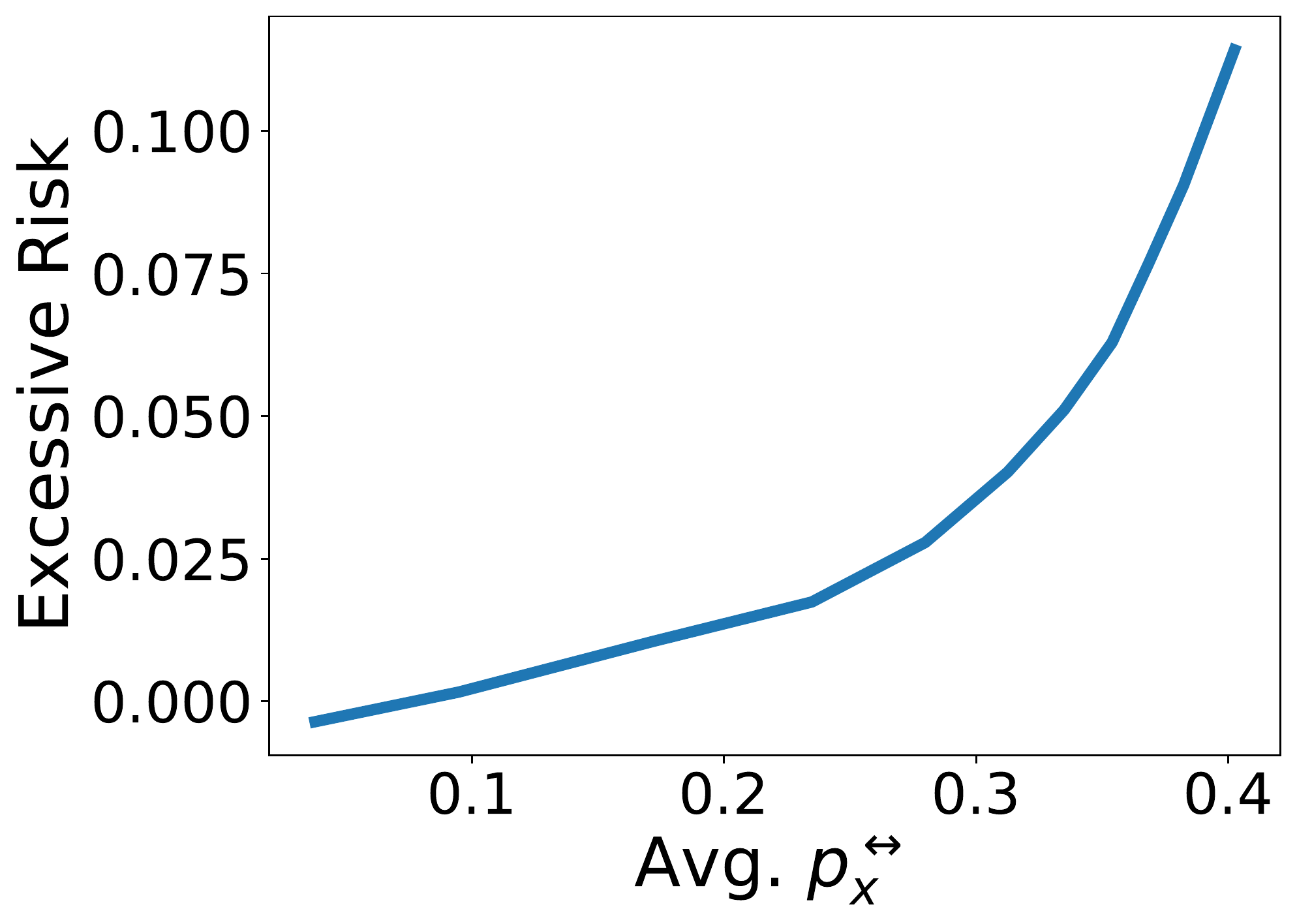}
  \caption{\small {\sl Credit}: Spearman correlation between closeness to boundary $s(\bm{x})$ and flipping probability $\fpx$ (top) and relation between input norms and excess risk (bottom).}
  \label{fig:corr_a_s}
\end{figure}
\paragraph{($\bm{D_2}$): The impact of the distance to decision boundary.}
As mentioned in Theorem \ref{thm:3}, the flipping probability $\fpx$ of a sample $\bm{x} \in \bar{D}$ directly controls the model deviation  $\Delta_{\tilde{\btheta}}$. 
Intuitively, samples close to the decision boundary are associated 
to small ensemble voting confidence and vice-versa. Thus, groups 
with samples close to the decision boundary will be more sensitive to the noise induced by the private voting process. 
To illustrate this intuition the paper reports the concept of \emph{closeness to boundary}. 
\begin{definition}[\cite{tran2021differentially}]
\label{def:dist_boundary}
Let $f_{\btheta}$ be a $C$-classes classifier trained using 
data $\bar{D}$ with its true labels. The closeness to the decision boundary $s(\bm{x})$ is defined as:
\(
s(\bm{x}) \defeq 1- \sum_{c=1}^C f_{\btheta^*, c} (\bm{x})^2,
\)
where  $f_{\btheta,c}$ denotes the softmax probability for class $c$.
\end{definition}
\begin{figure*}
\includegraphics[width=0.20\linewidth,height=80pt,valign=M]{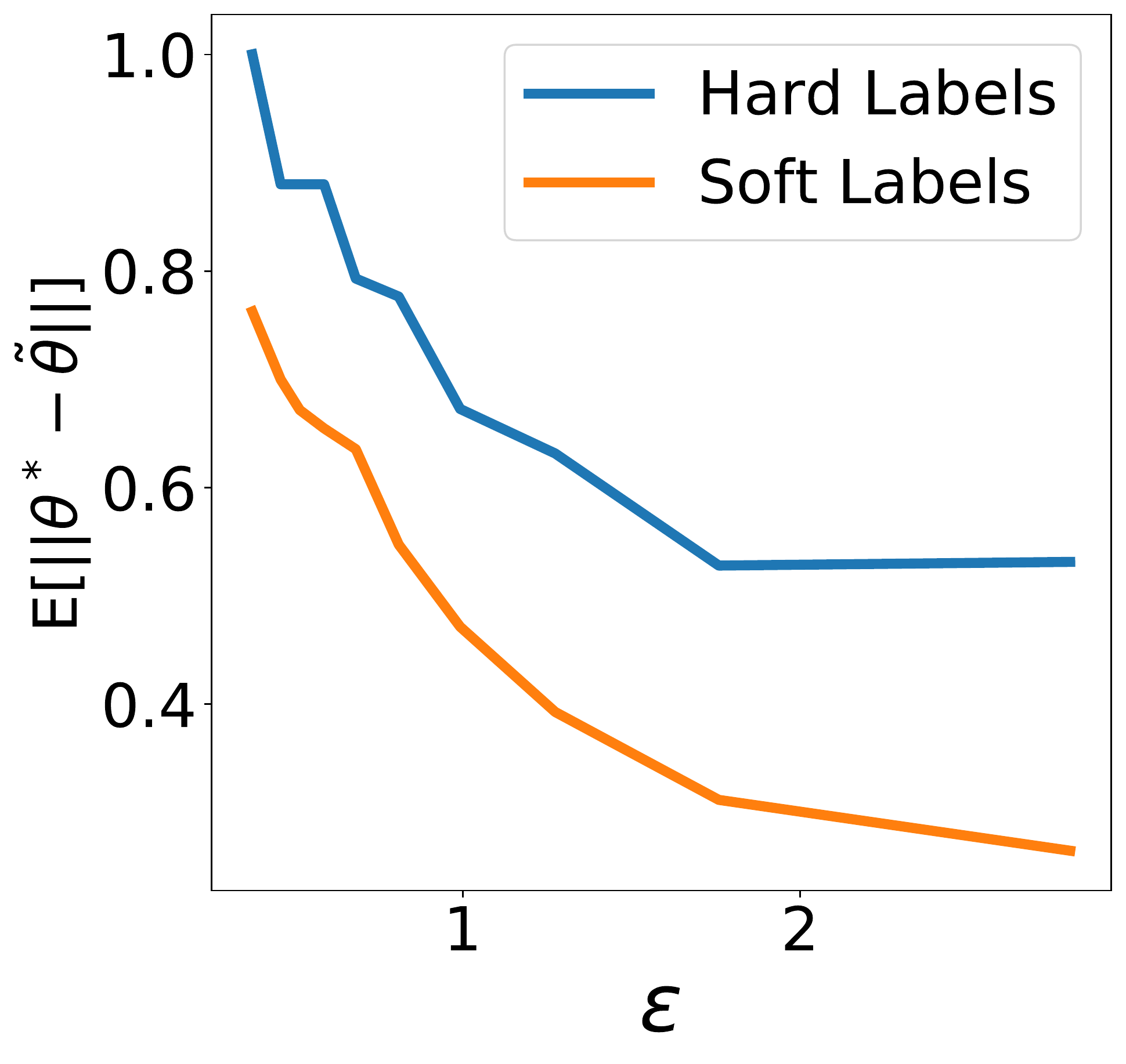}
\includegraphics[width=0.35\linewidth,height=130pt,valign=M]{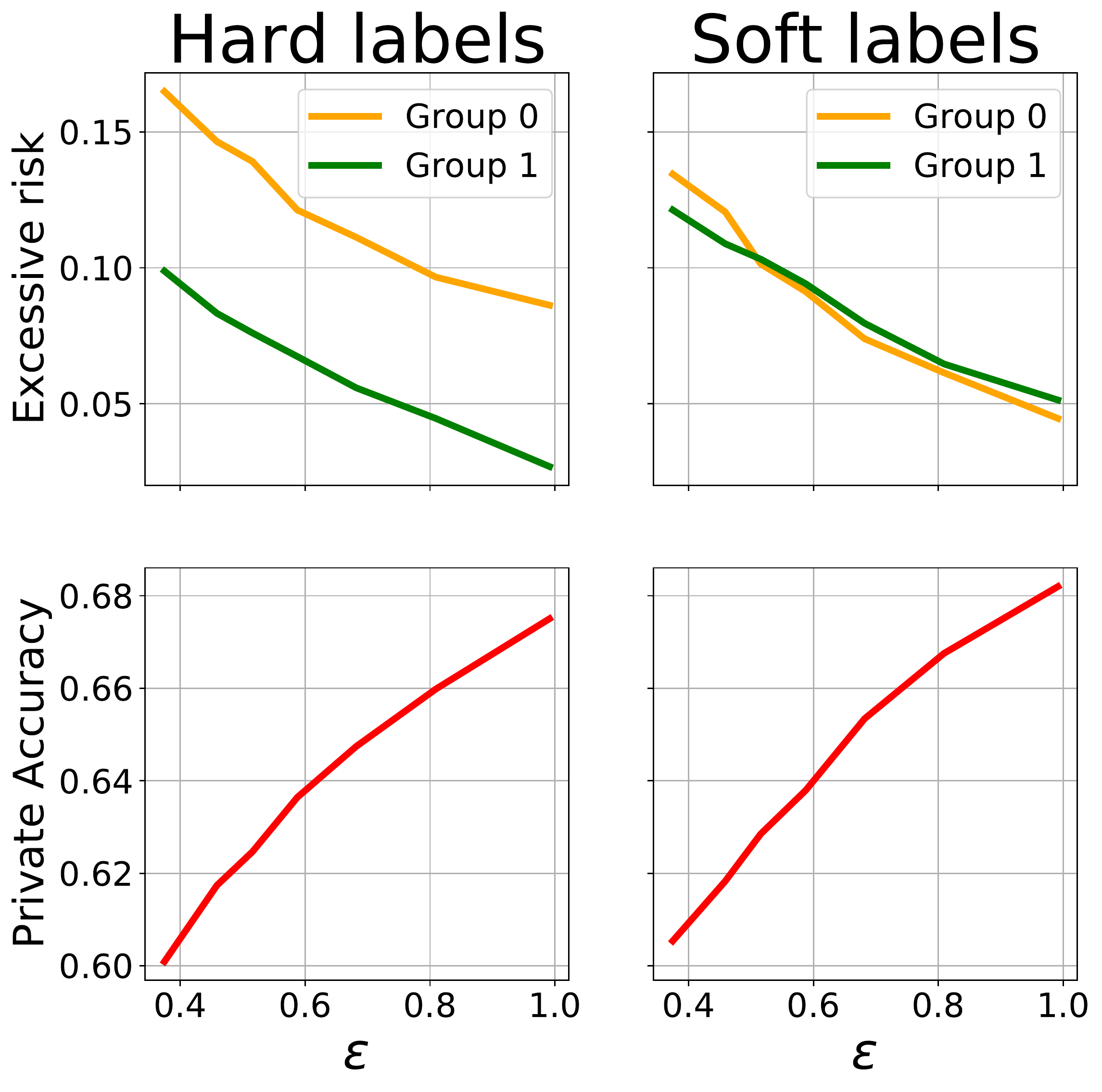}
\includegraphics[width=0.35\linewidth,height=130pt,valign=M]{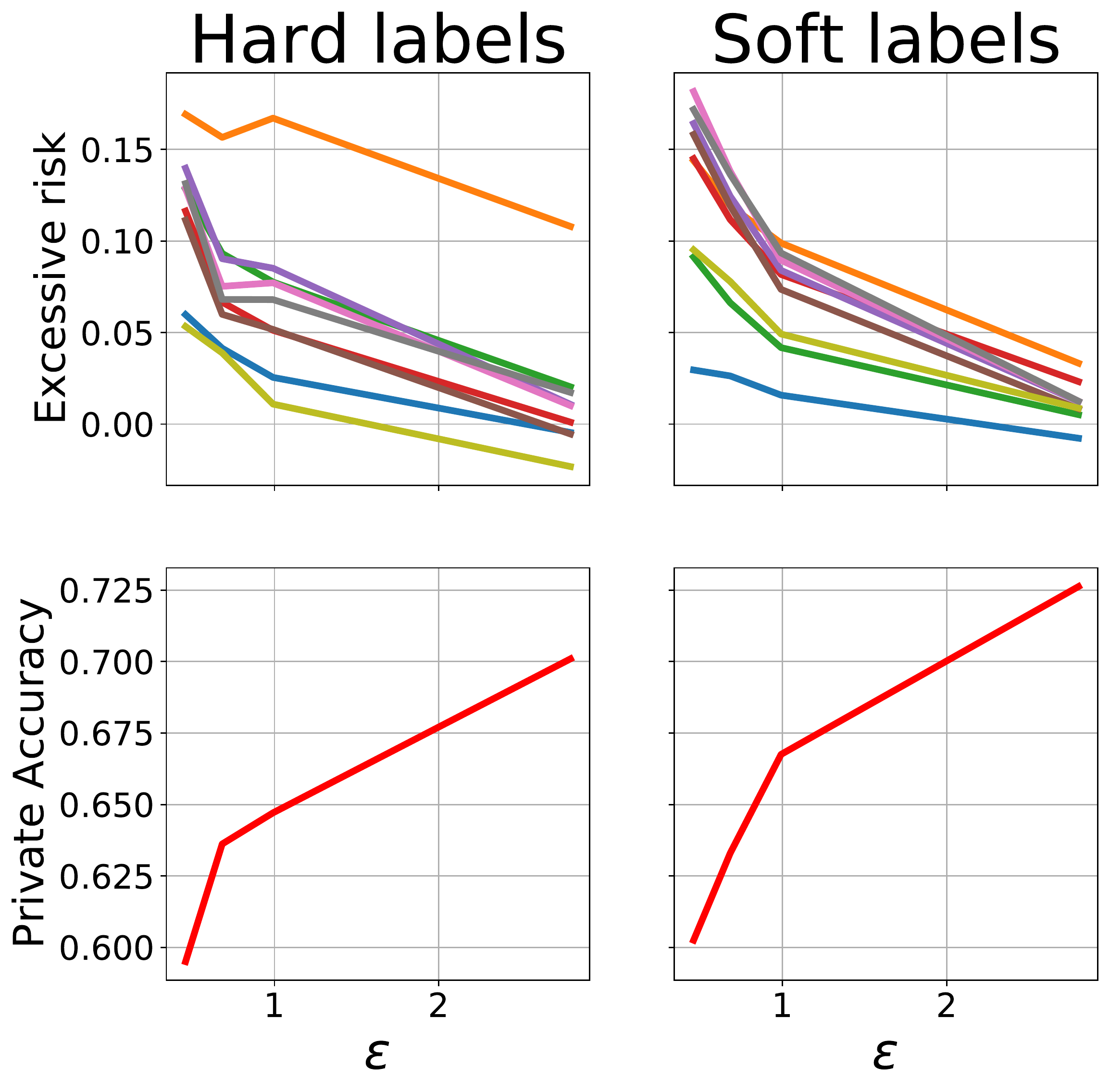}
\includegraphics[width=0.07\linewidth,height=55pt,valign=M]{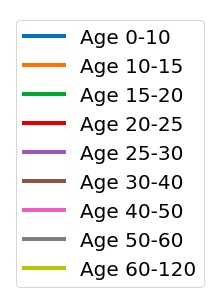}

\caption{Training privately PATE with hard and soft labels: 
Model deviation at varying of the privacy loss (left) on Credit dataset and excess risk at varying of the privacy loss for Credit (middle) and UTKFace (right) datasets. 
}
\label{fig:mitigation_solution}
\end{figure*}

\noindent 
The above discussion relates large (small) values of $s(\bm{x})$ to projections of point $\bm{x}$ that are close (distant) to the model decision boundary.
{\em The concept of closeness to decision boundary provides a way to indirectly quantify the flipping probability of a sample.} Empirically, the correlation between the distance of sample $\bm{x}$ to the decision boundary and its flipping probability $\fpx$ is illustrated in Figure \ref{fig:corr_a_s} (top). 
The plots are once again generated using a neural network with nonlinear objective and the relation holds for all datasets analyzed. 
The plot indicates that the samples that are close to the decision boundary have a higher probability of ``flipping'' their label, leading to a worse excess risk and unfairness. Finally, Figure \ref{fig:corr_a_s} (bottom) further illustrates the strong proportional effect of the flipping probability on the excess risk.

{To summarize, the norms $\|\bm{x}\|$ of a group's samples and their associated distance to boundary $s(\bm{x})$ are two key characteristics of the student data that influence fairness through their control of the model deviation $\Delta_{\tilde{\theta}}$, the smoothness parameters $\beta_a$, and the group gradients $G_a$, (see Figure \ref{fig:causal} for a schematic representation).}

\section{Mitigation solution}
\label{sec:mitigation}
The previous sections have identified a number of algorithmic and data-related factors that can influence the disparate impact of the student model.
These factors often affect the model deviation $\Delta_{\tilde{\btheta}}$, which is related to the excess risk of different groups (as shown in Theorem \ref{thm:2}), whose difference we would like to minimize. 
With this in mind, this section proposes a strategy to reduce the deviation of the private model parameters. 
To do so, we exploit the idea of \emph{soft labels} instead of traditional 
\emph{hard labels} in the voting process. 
Hard labels may be significantly affected by small perturbations due to noise, especially when the teachers have low confidence in their votes.
For example, consider a binary classifier where for a sample $\bm{x}$, $\nicefrac{k}{2}+1$ teachers vote label $0$ and 
$\nicefrac{k}{2}-1$, label $1$, for some even ensemble size $k$. 
If perturbations are introduced to these counts to esnure privacy, the process may incorrectly report label ($\hat{y}=1$) with high probability, causing 
causing the student model's private parameters to deviate significantly from the non-private ones.
This issue can be partially addressed by the introduction of soft labels:
\begin{definition}
\label{def:soft_labels} The \emph{soft label} of a sample $\bm{x}$ is 
$\alpha(\bm{x})= \left( \nicefrac{ \#_c(\bm{T}(\bm{x}))}{k}\right)^C_{c=1}$
and its private counterpart $\tilde{\alpha}(\bm{x})$ adds Gaussian noise $\mathcal{N}(0, \sigma^2)$ in the numerator of $\alpha(\bm{x})$. 
\end{definition}

To exploit soft labels, the training step of the student model uses loss 
\(
\ell'(\bar{f}_{\btheta}(\bm{x}), \tilde{\alpha}) = \sum_{c=1}^C \tilde{\alpha}_c \ell(f_{\btheta}(\bm{x}), c),
\)
which can be considered as a weighted version of the original loss function 
$\ell(\bar{f}_{\btheta}(\bm{x}), c)$ on class label $c$, whose weight 
is its confidence $\tilde{\alpha}_c $. 
Note that $\ell'(\bar{f}_{\btheta}(\bm{x}), \tilde{\alpha}) 
= \ell(\bar{f}_{\btheta}(\bm{x}))$ 
when all teachers in the ensemble chose the same label.
{The privacy loss for this model is equivalent to that of classical 
PATE. The analysis is reported in Appendix C of \citep{Tran:PPAI22}}.

The effectiveness of this scheme is demonstrated in Figure \ref{fig:mitigation_solution}. 
The experiment settings are reported in detail in  \citep{Tran:PPAI22} (Appendix) and reflect those described in Section \ref{sec:roadmap}.
The left subplot shows the relation between the model deviation 
$\mathbb{E}\left[ \Delta_{\tilde{\btheta}} \right]$ at varying of the privacy loss $\epsilon$ (dictated by the noise level 
$\sigma$). Notice how the student models trained using soft labels  
reduce their model deviation ($E[ \Delta_{\tilde{\btheta}}]$) when compared to the counterparts that use hard labels. 

The middle and right plots of Figure \ref{fig:mitigation_solution} 
show the impact of the proposed solution on the private student model in terms of the utility/fairness tradeoff. 
The top subplots illustrate the group excess risks $R(\bar{D}_{\leftarrow a})$ associated with each group $a \in \cA$ for Credit (left) and UTKFace (right) datasets, respectively. 
The bottom subplots shows the accuracy of the models, which include a simple ReLU network for the tabular data and a CNN for the image dataset.
Recall that the fairness goal $\xi(\bar{D})$ is captured by the gap between excess risk curves in the figures.
Notice how soft labels can reduce the disparate impacts in private 
training (top).
Notice also that while fairness is improved there is seemingly no 
cost in accuracy. On the contrary, using soft labels produces comparable or better models than the counterparts produced with hard labels. 

Additional experiments, including illustrating the behavior of the 
mitigating solution at varying of the number $k$ of teachers are 
reported in Appendix D of \citep{Tran:PPAI22} and the trends are all 
consistent with what is described above.
{Note also that the proposed solution preserves the original privacy budget. In contrast, mitigating solutions that would consider explicitly the number of teachers $K$ or the smoothness parameter $\lambda$ will inevitably introduce further privacy/fairness tradeoffs as would require costly privacy-preserving hyper-parameter optimization \citep{papernot2021hyperparameter}.

Finally, an important benefit of this solution is that it \emph{does not} uses the protected group information ($a \in \cA$) during training. Thus, it is applicable in challenging situations when it is not feasible to collect or use protected features (e.g., under the General Data Protection Regulation (GDPR)  \cite{lahoti2020fairness}). 
\emph{These results are significant. They suggest that this mitigating solution can be effective for improving the disparate impact of private model ensembles without sacrificing accuracy.}

\section{Discussion, Limitations, and Conclusions}
\label{sec:limitations}

This study highlights two key messages. First, the proposed mitigating solution relates to concepts in robust machine learning. In particular, \citet{papernot2016distillation} showed that training a classifier with soft labels can increase its robustness against adversarial samples. This connection is not coincidental, as the deviation of the model is influenced by the voting outcomes of the teacher ensemble (as demonstrated in Theorems \ref{thm:2} and \ref{thm:3}). In the same way that robust ML models are insensitive to input perturbations, an ensemble that strongly agrees will be less sensitive to noise and vice versa. This raises the question of the relationship between robustness and fairness in private models, which is an important open question. Second, we also note that more advanced voting schemes, such as interactive GNMAX \citep{papernot2018scalable}, may produce different fairness results. While this is an interesting area for further analysis, these sophisticated voting schemes may introduce sampling bias (e.g., interactive GNMAX may exclude samples with low ensemble voting agreement), which could create its own fairness issues.

Given the growing use of privacy-preserving learning tasks in consequential decisions, this work represents a significant and widely applicable step towards understanding the causes of disparate impacts in differentially private learning systems.

\section*{Ethical Statement}

Private Aggregation of Teacher Ensembles (PATE) is a private machine learning framework that aims to protect the privacy of data labels while still enabling effective learning in semi-supervised settings. However, this framework has been shown to introduce accuracy disparities among individuals and groups, potentially leading to unfairness. In order to address this issue, it is important to analyze the algorithmic and data properties that contribute to these disproportionate impacts and to develop guidelines to mitigate these effects. It is also crucial to consider the potential impacts on diverse groups and to strive for fairness in the development and application of PATE. Ensuring the ethical use of this privacy-preserving framework is essential in order to ensure that it is not used to perpetuate or exacerbate existing inequalities.

\section*{Acknowledgements}
This research is partially supported by NSF grants 2232054, 2133169	and NSF CAREER Award 2143706. Fioretto is also supported by an Amazon Research Award and a Google Research Scholar Award. Its views and conclusions are those of the authors only.

\bibliographystyle{named}
\bibliography{iclr2023_conference}

\newpage
\appendix
\include{appendix}

\end{document}

%% file: appendix.tex


\onecolumn

\setcounter{theorem}{3}
\setcounter{corollary}{2}
\setcounter{lemma}{1}
\setcounter{proposition}{2}

\begin{center}
\noindent{\vspace{0.25em}\LARGE 
\textbf{\sc On the Fairness Impacts of Private Ensembles Models}\\[0.75em]
\textbf{\sc Supplemental Material}}
\end{center}

\setcounter{theorem}{0}
\setcounter{corollary}{0}
\setcounter{lemma}{0}
\setcounter{proposition}{0}

\section{A thorough dicussion on fairness metric adopted}
\label{sec:discussion}
We provide more justification on the fairness definition adopated in the paper in this section. First our fairness metric is based on the concept of \emph{excessive risk} which is widely used as the benchmark metric to measure the utility of private learning \cite{pathak2010multiparty,wang2018empirical,cattan2022fine}. The excessive risk is also referred as \emph{utility} in other work \cite{NEURIPS2021_7c6c1a7b,DBLP:journals/corr/abs-1906-12056}. The excessive risk measures the difference in accuracy between private and non-private models over the population. We thus extend this concept to group levels. Due to the popularity of excessive risk used in private learning analysis as the benchmark for utility, a trivial way to measure fairness is to quantify the difference among excessive risks of different groups.

As a concrete example, consider the scenario of two groups A and B. Suppose under non-private training, the accuracy of A and B are 90\% and 80\% respectively. However, under private training with $\epsilon = 1.0$, the private accuracy of A and B nows are both 70\%. Although, under the standard accuracy parity metrics the private model returns a perfect fairness here, under our fairness definition the private model is unfair.  In particular, the excessive risk for group A is 20\% and for group B is 10\%. The private learning somehow penalizes group A much more than group B.

\section{Missing Proofs}
\label{app:missing_proofs}
This section contains the missing proofs associated with the theorems and corollaries
presented in the main paper. The theorems are restated for completeness.


First we provide the upper bound on the excess risk per group $a \in \cA$ in the following Lemma \ref{a:thm:1}. This helps to understand what factors control the excess risk for a particular group. 
\begin{lemma}
\label{a:thm:1}
The excess risk $R(\bar{D}_{\leftarrow a})$ of a group $a \in \cA$ is upper bounded as:
\begin{equation}
\label{eq:ER_ub}
R(D_{\leftarrow a}) \leq 
  \|G_a \| \mathbb{E}\left[ \Delta_{\tilde{\btheta}} \right] 
  + \nicefrac{1}{2}\; \beta_a \mathbb{E}\left[ \Delta_{\tilde{\btheta}}^2 \right],
\end{equation} 
where $G_a =  \EE_{\bm{x} \sim \bar{D}_{\leftarrow a}}\left[
  \nabla_{\btheta^*} 
        \ell(\bar{f}_{\btheta^*}(\bm{x}),y) \right]$ is the gradient of the group loss evaluated at $\btheta^*$, and $\Delta_{\tilde{\btheta}}$ and  $\Delta_{\tilde{\btheta}}^2$
capture the first and second order statistics of the model deviation.
\end{lemma}

\begin{proof}
By $\beta_a$ smoothness assumption on the loss function  defined over a group $a \in \cA$  it follows that:
\begin{align}
    \cL(\tilde{\btheta}; D_{\leftarrow a}, \bm{T})  \leq   \cL(\btheta^*; D_{\leftarrow a}, \bm{T})  + \big( \tilde{\btheta} - \btheta^* \big)^T G_a   + \frac{\beta_a}{2} \| \tilde{\btheta}  - \btheta^*\|^2.
\end{align}
By taking the expectation on both sides of the above equation w.r.t.~the randomness of the noise, we obtain:
\begin{align}
    \mathbb{E}[ \cL(\tilde{\btheta}; D_{\leftarrow a}, \bm{T}] &\leq \cL(\btheta^*; D_{\leftarrow a}, \bm{T})  + G_a^T \mathbb{E}[(\tilde{\btheta} - \optimal{\btheta})] 
     + \frac{\beta_a}{2} \mathbb{E}[\|\tilde{\btheta} - \optimal{\btheta}\|^2]  \label{eq:upper_bound}\\
    & \leq \cL(\btheta^*; D_{\leftarrow a}, \bm{T}) + \| G_a\|  \mathbb{E}\left[ \Delta_{\tilde{\btheta}}\right]  + \frac{1}{2} \beta_a \mathbb{E}[\Delta_{\tilde{\btheta}}^2],
    \label{eq:final_ineq}
\end{align}
where the last inequality is by Cauchy-Schwarz inequality on vectors. Next, by substituting $ R(\bar{D}_{\leftarrow a})  =\mathbb{E}[ \cL(\tilde{\btheta}; D_{\leftarrow a}, \bm{T}]   -\cL(\btheta^*; D_{\leftarrow a}, \bm{T})   $  into  \Eqref{eq:final_ineq}  we obtain the Lemma statement.

\end{proof}

\begin{theorem}
\label{a:thm:2}
The model fairness is upper bounded as: 
\begin{equation}
\label{eq:fair_ub}
\xi(\bar{D}) \leq \max_a 2 \| G_a \| \mathbb{E}\left[\Delta_{\tilde{\btheta}} \right]
                       + \max_a \nicefrac{1}{2}\;{\beta_a} \mathbb{E}\left[ \Delta_{\tilde{\btheta}}^2 \right].
\end{equation}  
\end{theorem}

\begin{proof}

By convexity assumption on the loss function  defined over a group $a \in \cA$  it follows that:
\begin{align}
   \cL(\btheta^*; D_{\leftarrow a}, \bm{T}) + \big( \tilde{\btheta} - \btheta^* \big)^T \bm{G}_a \leq  \cL(\tilde{\btheta}; D_{\leftarrow a}, \bm{T}) 
\end{align}
By taking the expectation on both sides of the above equation w.r.t.~the randomness of the noise, we obtain:
\begin{align}
\label{eq:lower_bound}
    \mathbb{E}[ \cL(\tilde{\btheta}; D_{\leftarrow a}, \bm{T})] & \geq  \cL(\btheta^*; D_{\leftarrow a}, \bm{T}) + \mathbf{E}[\big( \tilde{\btheta} - \btheta^* \big)^T]  \bm{G}_a 
\end{align}

By combining Equation \ref{eq:lower_bound}
 and Equation \ref{eq:upper_bound} we obtain the following:
 
 \begin{align}
  \mathbf{E}[\big( \tilde{\btheta} - \btheta^* \big)^T]  \bm{G}_a      \leq  R(\bar{D}_{\leftarrow a})  \leq  \mathbb{E}\left[ \big( \tilde{\btheta} - \btheta^*\big)^T \right] \bm{G}_a + \frac{\beta_a}{2} \mathbb{E}\left[ \Delta_{\tilde{\btheta}}^2 \right] 
 \end{align}
 
 Based on the definition of fairness in Equation \ref{eq:risk_gap}, it follows that:
 
 \begin{align}
     \xi(\bar{D})  & = \max_{a, a' \in \cA}  R(\bar{D}_{\leftarrow a}) - R(\bar{D}_{\leftarrow a'}) \leq \max_{a, a' \in \cA}  \mathbf{E}[\big( \tilde{\btheta} - \btheta^* \big)^T] \big( G_a - G_{a'}\big) + \max_{a \in \cA }\frac{\beta_a}{2} \mathbb{E}\left[ \Delta_{\tilde{\btheta}}^2 \right]  \\
     &  \leq \max_a  2 \|G_a\|  \mathbf{E}[ \| \tilde{\btheta} - \btheta^* \| ] + \max_a \frac{\beta_a}{2} \mathbb{E}\left[ \Delta_{\tilde{\btheta}}^2 \right]  = 2 \max_a  \|G_a\| \mathbb{E}\left[ \Delta_{\tilde{\btheta}} \right]  +  \max_a \frac{\beta_a}{2} \mathbb{E}\left[ \Delta_{\tilde{\btheta}}^2 \right] 
 \end{align}
\end{proof}

\begin{theorem}
\label{a:thm:3}
Consider a student model $\bar{f}_{\btheta}$ trained with a convex and decomposable loss 
function $\ell(\cdot)$. Then, the expected 
difference between the private and non-private model parameters is 
upper bounded as follows:
\begin{equation}
    \mathbb{E}\left[  \Delta_{\tilde{\btheta}}\right] 
    \leq \frac{|c|}{m\lambda} \left[ \sum_{\bm{x} \in \bar{D}} p^{\leftrightarrow}_{\bm{x}} \| G^{\max}_{\bm{x}}\| \right],
\end{equation}
where $c$ is a real constant and
$G^{\max}_{\bm{x}}= \max_{\btheta}\| \nabla_{\btheta} h_{\btheta}(\bm{x}) \|$ 
represents the maximum gradient norm distortion introduced by a 
sample $\bm{x}$. Both $c$ and $h$ are defined as in Equation 
\eqref{eq:decomposable}. 
\end{theorem}

Proof of Theorem \ref{a:thm:3} requires the following Lemma \ref{app:lem:1} from \cite{opt_paper} on the property of strongly convex functions.
\begin{lemma}[\citet{opt_paper}]
\label{app:lem:1}
Let $\cL(\btheta)$ be a differentiable function. 
Then $\cL(\btheta)$ is $\lambda$-strongly convex \emph{iff} for all vectors
 $\btheta, \btheta'$: 
\begin{equation}
    (\nabla_{\btheta}\cL  - \nabla_{\btheta'} \cL)^T (\btheta - \btheta') \geq \lambda \| \btheta -\btheta' \|^2.
\end{equation}
\end{lemma}

\begin{proof}[Proof of Theorem \ref{a:thm:3}]
Let us denote with $\hat{y}_i =   \textsl{v}(\bm{T}(\bm{x}_i))$ to indicate the non-private voting label associated with $\bm{x}_i$ and $\tilde{y}_i =   \tilde{\textsl{v}}(\bm{T}(\bm{x}_i))$ for the private voting label counterpart.
The regularized empirical risk function with the non-private voting labels from \Eqref{eq:ERM} can be rewritten as follows:

\begin{align}
\cL & = \frac{1}{m} \sum_{i=1}^m \ell(\bar{f}_{\btheta}(\bm{x}_i),\hat{y}_i)  +\lambda \| \btheta \|\\
& = \frac{1}{m} \sum_{i=1}^m \big[z(h_{\btheta}(\bm{x}_i)) + c \hat{y}_i h_{\btheta}(\bm{x}_i) \big] +\lambda \| \btheta \|^2,
\label{eq:hat_L}
\end{align}
where the second equality is due to the decomposable loss assumption.
Likewise, define $\tilde{\cL}$ to be the regularized empirical risk function  with  private voting labels $\tilde{y}_i$:
\begin{align}
\label{eq:tilde_L}
\tilde{\cL}  = \frac{1}{m} \sum_{i=1}^m \big[z(h_{\btheta}(\bm{x}_i)) + c \tilde{y}_i h_{\btheta}(\bm{x}_i) \big] +\lambda \| \btheta \|^2,
\end{align}

\noindent Based on Equation \ref{eq:hat_L}  and Equation \ref{eq:tilde_L}, it follows that: $\tilde{\cL} = \cL + \Delta_{\cL}$ where $\Delta_{\cL} = \frac{c}{m} \sum_{i=1}^m(\tilde{y}_i -\hat{y}_i) h_{\btheta}(\bm{x}_i)$.

\noindent 
Furthermore, since each individual loss function  $\ell(\bar{f}_{\btheta}(\bm{x}_i),\tilde{y}_i)$ or $\ell(\bar{f}_{\btheta}(\bm{x}_i),\hat{y}_i)$ is convex  for all $i$ from the given assumption, then $\tilde{\cL}$ and $\cL$ both are $\lambda$-strongly convex. 

Next, from the definition of $\  \tilde{\btheta} = \argmin_{\btheta} \tilde{\cL}$,  and  $\  \optimal{\btheta} = \argmin_{\btheta} \cL  $ it follows that: 

\begin{equation}
\nabla_{\tilde{\btheta}} \tilde{\cL} =\boldsymbol{0} \ \mbox{and }  \nabla_{\optimal{\btheta}} \cL =\boldsymbol{0}. 
\label{eq:4}
\end{equation}

By Lemma \ref{app:lem:1}, it follows that: 

\begin{equation}
    \big( \nabla_{\tilde{\btheta}}\tilde{\cL} - \nabla_{\optimal{\btheta}} \tilde{\cL} \big)^T (\tilde{\btheta} - \optimal{\btheta}) \geq \lambda \| \tilde{\btheta} - \optimal{\btheta} \|^2. 
    \label{eq:5}
\end{equation}

Now since $\nabla_{\tilde{\btheta}}\tilde{\cL} = \bm{0}$ by  \Eqref{eq:4}, we can rewrite Equation \ref{eq:5} as

\begin{equation}
    \big( - \nabla_{\optimal{\btheta}} \tilde{\cL} \big)^T (\tilde{\btheta} - \optimal{\btheta}) \geq \lambda \| \tilde{\btheta} - \optimal{\btheta} \|^2,
    \label{eq:5b}
\end{equation}
since $ \nabla_{\optimal{\btheta}} \tilde{\cL}  = \nabla_{\optimal{\btheta}} \cL + \nabla_{\optimal{\btheta}} \Delta_{\cL}  = \boldsymbol{0} + \nabla_{\optimal{\btheta}} \Delta_{\cL}  =  \nabla_{\optimal{\btheta}} \Delta_{\cL}   $. In addition, by applying the Cauchy-Schwartz inequality to the L.H.S of \Eqref{eq:5b} we obtain
\begin{equation}
  \|\nabla_{\optimal{\btheta}}\Delta_{\cL} \|  \|(\optimal{\btheta} - \tilde{\btheta})\|\geq  -(\nabla_{\optimal{\btheta}}\Delta_{\cL})^T (\tilde{\btheta} - \optimal{\btheta})  \geq \lambda \| \tilde{\btheta} - \optimal{\btheta} \|^2,
  \label{eq:6}
\end{equation}
and thus,
\begin{equation}\|\nabla_{\optimal{\btheta}}\Delta_{\cL} \|  \geq \lambda \| \tilde{\btheta} - \optimal{\btheta} \| \label{eq:7}.
\end{equation}

By definition of $\nabla_{\optimal{\btheta}}\Delta_{\cL}$ we can rewrite the above inequality as follows:
\begin{align}
\label{eq:7b}
   \|  \nabla_{\optimal{\btheta}}\Delta_{\cL} \| &= \| \frac{c}{m} \sum_{i=1}^m (\tilde{y}_i -\hat{y}_i) \nabla_{\optimal{\btheta}} h_{\optimal{\btheta}}(\bm{x}_i) \| \geq \lambda \|\tilde{\btheta} -\optimal{\btheta}\|^2.
\end{align}

Next, let $\rho_i = \hat{y}_i - \tilde{y}_i$, applying this substitution to the above and by triangle inequality it follows that:
\begin{align}
   \frac{|c|}{m} \sum^m_{i=1}| \rho_i| \|g_i\|   &\geq \frac{|c|}{m} \sum^m_{i=1}| \rho_i| \| \nabla_{\tilde{\btheta}} h_{\tilde{\btheta}}(\bm{x}_i)\|\\
   &\geq \|\frac{c}{m}\sum^m_{i=1}\rho_i \nabla_{\tilde{\btheta}} h_{\tilde{\btheta}}(\bm{x}_i) \|
   \geq \lambda \| \tilde{\btheta} - \optimal{\btheta} \|,
\end{align}
where the first inequality is due to definition of $g_{\bm{x}_i} = \max_{\btheta}\| \nabla_{\btheta} h_{\btheta}(\bm{x}_i) \|$ and the second inequality is due to the general triangle inequality .
Since $| \rho_i |$ is a Bernoulli random variable, in which $|\rho_i| = 1 $ w.p. $p^{\leftrightarrow}_{\bm{x}_i}$  and $|\rho_i| = 0$ w.p of $1-p^{\leftrightarrow}_{\bm{x}_i}$. Therefore $\mathbb{E}[ |\rho_i|]=p^{\leftrightarrow}_{\bm{x}_i}$. Thus, it follows that:

\begin{align}
   \mathbb{E}\big[ \frac{|c|}{m} \sum^m_{i=1}| \rho_i| \|g_{\bm{x}_i}\| \big] =\frac{|c|}{m} \sum^m_{i=1} p^{\leftrightarrow}_{\bm{x}_i} \|g_{\bm{x}_i}\|  \geq \lambda \mathbb{E}\big[ \| \tilde{\btheta} - \optimal{\btheta} \| \big] =\mathbb{E}\big[  \Delta_{\tilde{\btheta}} \big],
\end{align}
which concludes the proof.
\end{proof}

\begin{theorem}
\label{a:thm:4}
For a sample $\bm{x} \!\in\! \bar{D}$ let the teacher models 
outputs $f^i(\bm{x})$ be in agreement, $\forall i \in [k]$. 
The flipping probability $\fpx$ is given by 
\(
     \fpx = 1 - \Phi(\frac{k}{\sqrt{2} \sigma}),
\)
where $\Phi(\cdot)$ is the CDF of the std.~Normal distribution 
and $\sigma$ is the standard deviation in the Gaussian mechanism.
\end{theorem}
For simplicity of exposition Theorem \ref{a:thm:4} considers binary classifiers, i.e., $\mathcal{Y} = \{0,1\}$. The argument, however, can be trivially extended to generic $C$-classifiers. 

 \begin{proof}
 
 By assumption, for any given sample $\bm{x}$, all teachers agree in their predictions, so w.l.o.g., assume $k$ teachers output label $0$, while none of them outputs  label $1$.  
Next, let $\psi, \psi' \sim \mathcal{N}(0, \sigma^2)$ be two independent Gaussian random variables which are added to true voting counts, $k$ and $0$, respectively. 
The associated flipping probability is:
\begin{align}
     \fp{\bm{x}} & = \Pr\left[ 
    \tilde{\textsl{v}}(\bm{T}(\bm{x}))  \neq \textsl{v}(\bm{T}(\bm{x}))\right]= \Pr( k+ \psi \leq 0 + \psi')= \Pr( \psi' - \psi \geq k)\\  &= 1 - \Pr(\psi - \psi'\leq k),
\end{align}
since $\psi, \psi'$ are two independent Gaussian random variable with zero mean and  standard deviation of $\sigma$. 
Therefore, $\psi' - \psi \sim \mathcal{N}(0, 2\sigma^2)$. Thus:

$$\Pr(\psi - \psi' \leq k) = \Pr( \mathcal{N}(0, 2\sigma^2) \leq k) = \Phi(\frac{k}{\sqrt{2}\sigma}).$$

Hence, the flipping probability will be: $    \fp{\bm{x}}= 1- \Phi(\frac{k}{\sqrt{2} \sigma})$.
\end{proof}




\begin{corollary}[Theorem~\ref{a:thm:3}]
\label{a:cor:1}
Let $\bar{f}_{\btheta}$ be a \emph{logistic regression} classifier. Its 
expected model deviation is upper bounded as:
\begin{equation}
\label{app:eq:8}
 \mathbb{E}\left[ \Delta_{\tilde{\btheta}} \right] 
    \leq 
    \frac{1}{m\lambda} \left[ \sum_{\bm{x} \in \bar{D}} \fp{\bm{x}} \| \bm{x} \| \right]. 
 \end{equation}
\end{corollary}

\begin{proof}
The loss function $\ell(\bar{f}_{\btheta}(\bm{x}), y)$ of a logistic regression classifier with binary cross entropy loss can be rewritten as follows:
\begin{align}
    \ell(\bar{f}_{\btheta}(\bm{x}), y) & = -y \log(\frac{1}{1+\exp(-\btheta^T \bm{x})}) 
    - (1-y)   \log(\frac{\exp(-\btheta^T \bm{x})}{1+\exp(-\btheta^T \bm{x})})\\
    & = y\log(\exp(-\btheta^T \bm{x})) - \log\big(\frac{\exp(-\btheta^T \bm{x})}{1+\exp(-\btheta^T \bm{x})}\big)\\
    & = y(-\btheta^T \bm{x})   - \log\big(\frac{\exp(-\btheta^T \bm{x})}{1+\exp(-\btheta^T \bm{x}}\big).
\end{align}
Hence, $\ell(\cdot)$ is decomposable by Definition \ref{def:1} with $h_{\btheta}(\bm{x})=-\btheta^T \bm{x}$, $c=1$ and $z(h) =-\log(\frac{\exp(h)}{1+\exp(h)}) $. 

Applying Theorem \ref{a:thm:3} with $G^{\max}_{\bm{x}} = \max_{\btheta} \|\nabla_{\btheta} h_{\btheta}(\bm{x}) \| = \max_{\btheta} \| \nabla_{\btheta}  -\btheta^T \bm{x} \| = \| \bm{x}\|$, and $c=1$, gives the intended result. 

\end{proof}


\begin{corollary}[Theorem~\ref{thm:3}]
\label{a:cor:2}
Given the same settings and assumption of Theorem \ref{thm:3}, it follows:
\begin{equation}
\label{eq:9}
    \mathbb{E}\left[ \Delta_{\tilde{\btheta}}^2 \right] 
    \leq \frac{|c|^2}{m \lambda^2} \left[ \sum_{\bm{x} \in \bar{D}} p^{\leftrightarrow 2}_{\bm{x}} \| G_{\bm{x}}^{\max}\|^2 \right].
\end{equation}
\end{corollary}

\begin{proof}
First, by Theorem \ref{a:thm:3} we obtain an upper bound for $ \mathbb{E}\left[ \Delta_{\tilde{\btheta}}^2 \right] $ as follows:

\begin{align}
     \mathbb{E}\left[ \Delta_{\tilde{\btheta}}^2 \right]  \leq \frac{c^2}{\lambda^2} \left[ \frac{1}{m}\sum_{\bm{x} \in \bar{D}} p^{\leftrightarrow}_{\bm{x}} \| G^{\max}_{\bm{x}}\| \right]^2.
     \label{eq:second}
\end{align}

Applying the sum of squares inequality on the R.H.S.~of  \Eqref{eq:second} we obtain:
\begin{align}
     \frac{c^2}{\lambda^2} \left[ \frac{1}{m}\sum_{\bm{x} \in \bar{D}} p^{\leftrightarrow}_{\bm{x}} \| G^{\max}_{\bm{x}}\| \right]^2 \leq  \frac{c^2}{\lambda^2} \left[ \frac{1}{m}  p^{\leftrightarrow 2}_{\bm{x}} \| G^{\max}_{\bm{x}}\|^2 \right],
\end{align}
which concludes the proof.

\end{proof}


\section{Privacy Analysis}
\label{app:pate_privacy}
This section provides the privacy analysis for the original PATE model and the proposed mitigation solution.  In PATE with the noisy-max scheme presented in \Eqref{eq:noisy_max} of the main paper (also called GNMAX), the privacy budget is used for releasing the voting labels $ \tilde{\textsl{v}}(\bm{T}(\bm{x}_i))$ (a.k.a.~hard labels) for each of the $m$ public data samples $\bm{x}_i \in \bar{D}$ according to:

\begin{equation}
    \tilde{\textsl{v}}(\bm{T}(\bm{x}_i)) \!=\! \argmax_c \{ \#_c(\bm{T}(\bm{x}_i)) \!+\! \cN(0, \sigma^2)\}
\end{equation}

The proposed mitigation solution, instead, releases privately the voting counts $(\#_c(\bm{T}(\bm{x}_i)) \!+\! \cN(0, \sigma^2))_{c=1}^C$ and uses these noisy counts to construct the \emph{soft-labels}, see Equation (11).  

Using an analogous analysis as that provided in \cite{papernot2018scalable}, adding or removing one individual sample $\bm{x}$ from any disjoint partition $D_i$ of $D$ can change the voting count vector by at most two. This value of the query deviation is obtained by GNMAX \cite{papernot2018scalable}. 
Therefore the privacy cost for releasing hard labels or soft-labels is equivalent. 

Next, this section provides the privacy computation $\epsilon$ given by Gaussian mechanism which adds Gaussian noise with standard deviation $\sigma$ to the voting counts.

The privacy analysis of PATE with hard or soft-labels is based on the concept of Renyi differential privacy (RDP) \cite{Mironov_2017}. In either implementations, the process uses the Gaussian mechanism to add independent Gaussian noise to the voting counts. The following Proposition \ref{prop:3} from \cite{papernot2018scalable} derives the privacy guarantee for GNMAX.

\begin{proposition} 
\label{prop:3}
The GNMAX aggregator with private Gaussian noise $\mathcal{N}(0, \sigma^2)$ satisfies $(\gamma,\nicefrac{\gamma}{\sigma^2}) $-RDP for all $\gamma \geq 1$.
\end{proposition}

Since the GNMAX mechanism is applied on $m$ public data samples from $\bar{D}$, the total privacy loss spent to provide the private labels is derived by the following composition theorem.

\begin{theorem}[Composition for RDP]
\label{app:rdp_comp}
If a mechanism $\cM$ consists of a sequence of adaptive mechanisms $\cM_1, \cM_2, \ldots, \cM_m$ such that for any $i \in [m]$, $\cM_i$ guarantees $(\gamma, \epsilon_i)$-RDP, then $\cM$ guarantees $(\gamma, \sum_{i=1}^m \epsilon_i)$-RDP.
\end{theorem}

Based on Theorem \ref{app:rdp_comp} and Proposition \ref{prop:3}, PATE satisfies $(\gamma, \nicefrac{m\gamma}{\sigma^2})$-RDP.  PATE also satisfies $(\epsilon, \delta)$-DP by the following theorem.

\begin{theorem}[From RDP to DP]
\label{thm:rdp_2_dp}
If a mechanism $\cM$ guarantees $(\gamma, \epsilon)$-RDP, then $\cM$ guarantees $(\epsilon +  \frac{\log \nicefrac{1}{\delta}}{\gamma-1}, \delta)$-DP for any $\delta \in (0,1)$.
\end{theorem}

As a result of Theorem \ref{thm:rdp_2_dp}, PATE (with either hard or soft labels) satisfies $( \nicefrac{m\gamma}{\sigma^2} +\frac{\log \nicefrac{1}{\delta}}{\gamma -1}, \delta)$-DP.

\section{Experimental Analysis (Ext)}
\label{app:exp_ext}
This section reports detailed information about the experimental setting
as well as additional results conducted on the Income, Bank, Parkinsons, Credit Card and UTKFace datasets. 

\subsection{Setting and Datasets}

\smallskip\noindent\textbf{Computing Infrastructure} 
All of our experiments are performed on a distributed cluster equipped with Intel(R) Xeon(R) Platinum 8260 CPU @ 2.40GHz and 8GB of RAM.

\smallskip\noindent\textbf{Software and Libraries}  
All models and experiments were written in Python 3.7. All neural network classifier models in our paper were implemented in Pytorch 1.5.0. 

The Tensorflow Privacy package was also employed for computing the privacy loss. 


\smallskip\noindent\textbf{Datasets}  
This paper evaluates the fairness analysis of PATE on the following four UCI datasets: \emph{Bank}, \emph{Income}, \emph{Parkinsons}, \emph{Credit card} and UTKFace dataset. 
A descriptions of each dataset is reported as follows:

\begin{enumerate}
    \item \textbf{Income} (Adult) dataset, where the task is to predict if an 
    individual has low or high income, and the group labels are defined by race: 
    \emph{White} vs \emph{Non-White} \cite{UCIdatasets}.

    \item \textbf{Bank} dataset, where the task is to predict if a user subscribes 
    a term deposit or not and the group labels are defined by age: 
    \emph{people whose age is less than vs greater than 60 years old} \cite{UCIdatasets}. 
    
    \item \textbf{Parkinsons} dataset, where the task is to predict if 
    a patient has total UPDRS score that exceeds the median value, 
    and the group labels are defined by gender: \emph{female vs male} \cite{article}.  
       
    \item \textbf{Credit Card} dataset, where the task is to predict if 
    a customer defaults a loan or not. The group labels are defined by gender:
    \emph{female vs male} \cite{creditdataset}. 
    
    \item \textbf{UTKFace} dataset, where the task is to predict the gender of a given facial image. The group labels are defined based on the following 9 age ranges: 0-10, 10-15, 15-20, 20-25, 25-30, 30-40, 40-50, 50-60, 60-120.  \cite{hwang2020fairfacegan}
\end{enumerate}

On each dataset we perform standardization to render all input features with zero mean and unit standard deviation. Each dataset was partitioned into three disjoint subsets: private set, public train,  and test set, as follows. We randomly select 75\% of the dataset to use as private data and the rest for public data. For the public data, $m=200$ samples are randomly selected to train the student model, and the rest of the data is used as a test  set to evaluate that model.

\smallskip\noindent\textbf{Models' Setting} 

To visually show how tight the upper bound from Corollary \ref{a:cor:1} is, the paper uses a logistic regression model with 1000 runs to estimate the expected model deviation $\mathbb{E}\left[ \Delta_{\tilde{\btheta}} \right] =\mathbb{E}\left[ \| \tilde{\btheta} -\optimal{\btheta}\| \right] $. 

For other  experiments, the paper uses a neural network with with two hidden layers and nonlinear ReLU activations for both the ensemble and student models. All reported metrics are an average of 100 repetitions, used to compute the empirical expectations. The batch size for stochastic gradient descent is fixed to 32 and the learning rate is $\eta = 1e-4$. 

\subsection{Upper bound of the expected model deviation} 
\label{sec:D2}

The following provides empirical results on Corollary \ref{cor:1} on four benchmark datasets. As indicated in this corollary, the expected model deviation is bounded by $ \frac{1}{m\lambda} \left[ \sum_{\bm{x} \in \bar{D}} \fp{\bm{x}} \| \bm{x} \| \right]$. To visualize how tight the bounds are we report the RHS and LHS values of  \Eqref{eq:8} on different datasets. We run with two settings: $k = 20, \lambda = 20$ in Figure \ref{fig:cor1_bound}  and $k= 200, \lambda = 100$ in Figure \ref{fig:cor1_bound2}.

\begin{figure}
\centering
\begin{subfigure}[b]{0.2\textwidth}
\includegraphics[width = 1.0\linewidth]{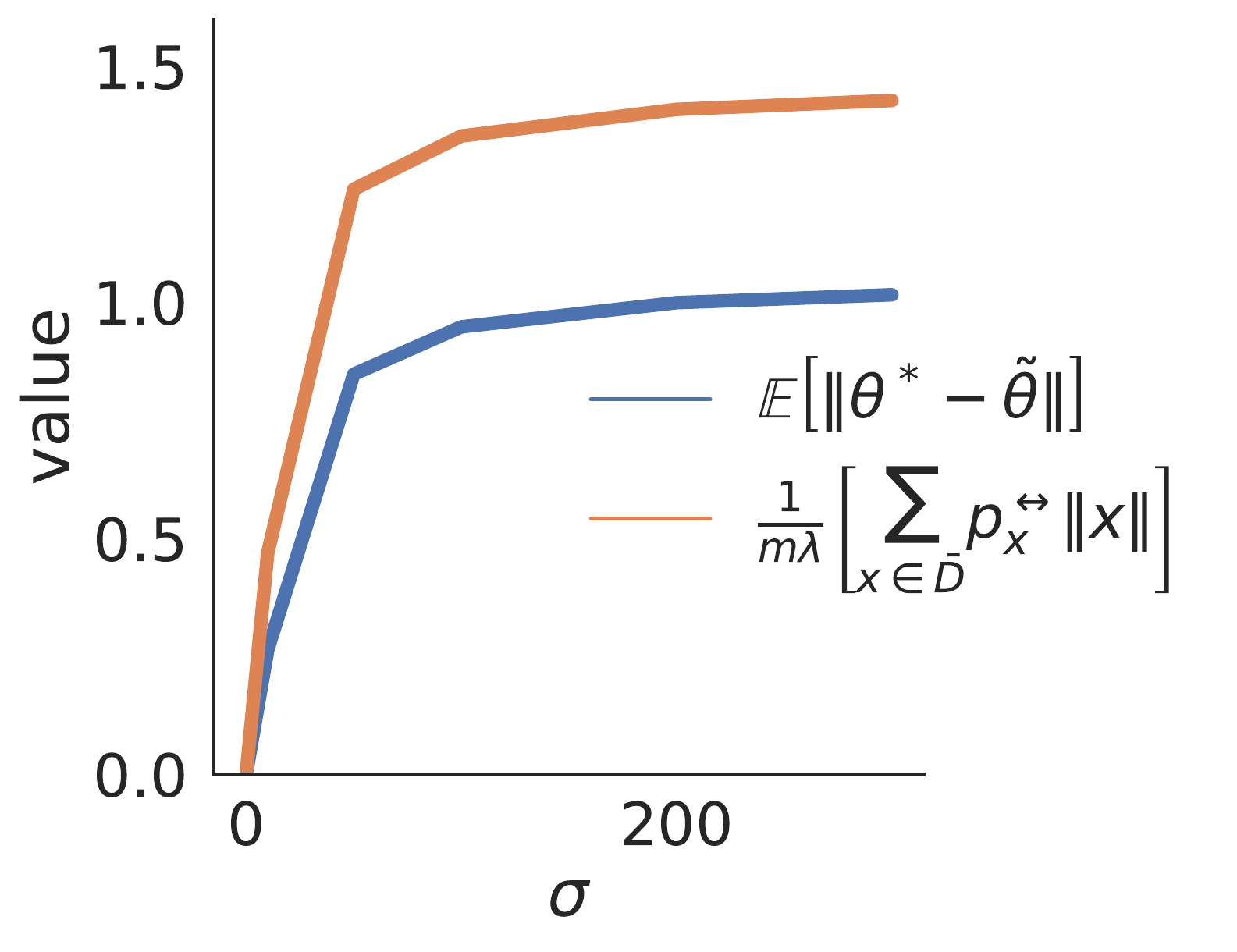}
\caption{Bank dataset}
\end{subfigure}
\begin{subfigure}[b]{0.2\textwidth}
\includegraphics[width = 1.0\linewidth]{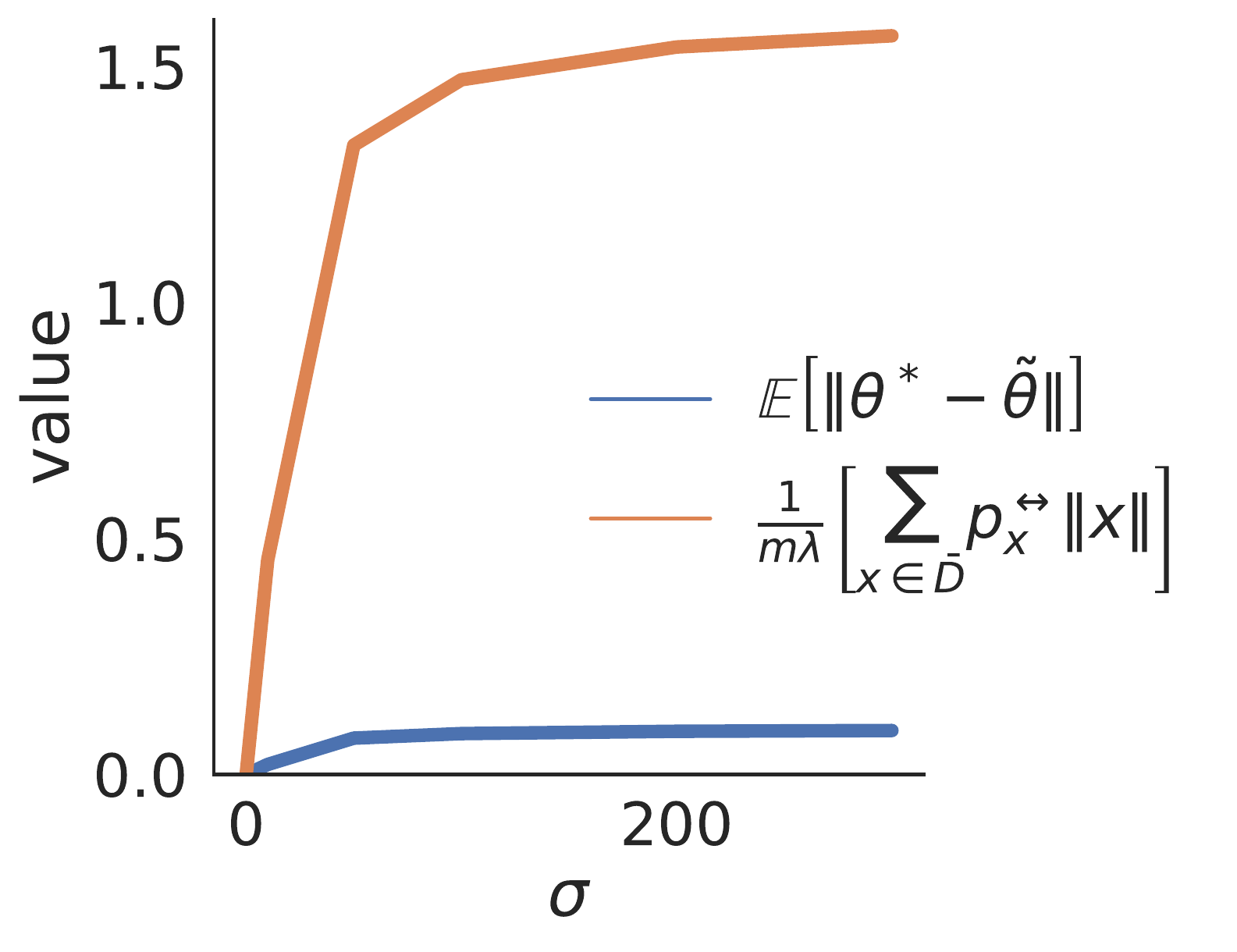}
\caption{Credit card dataset}
\end{subfigure}
\begin{subfigure}[b]{0.2\textwidth}
\includegraphics[width = 1.0\linewidth]{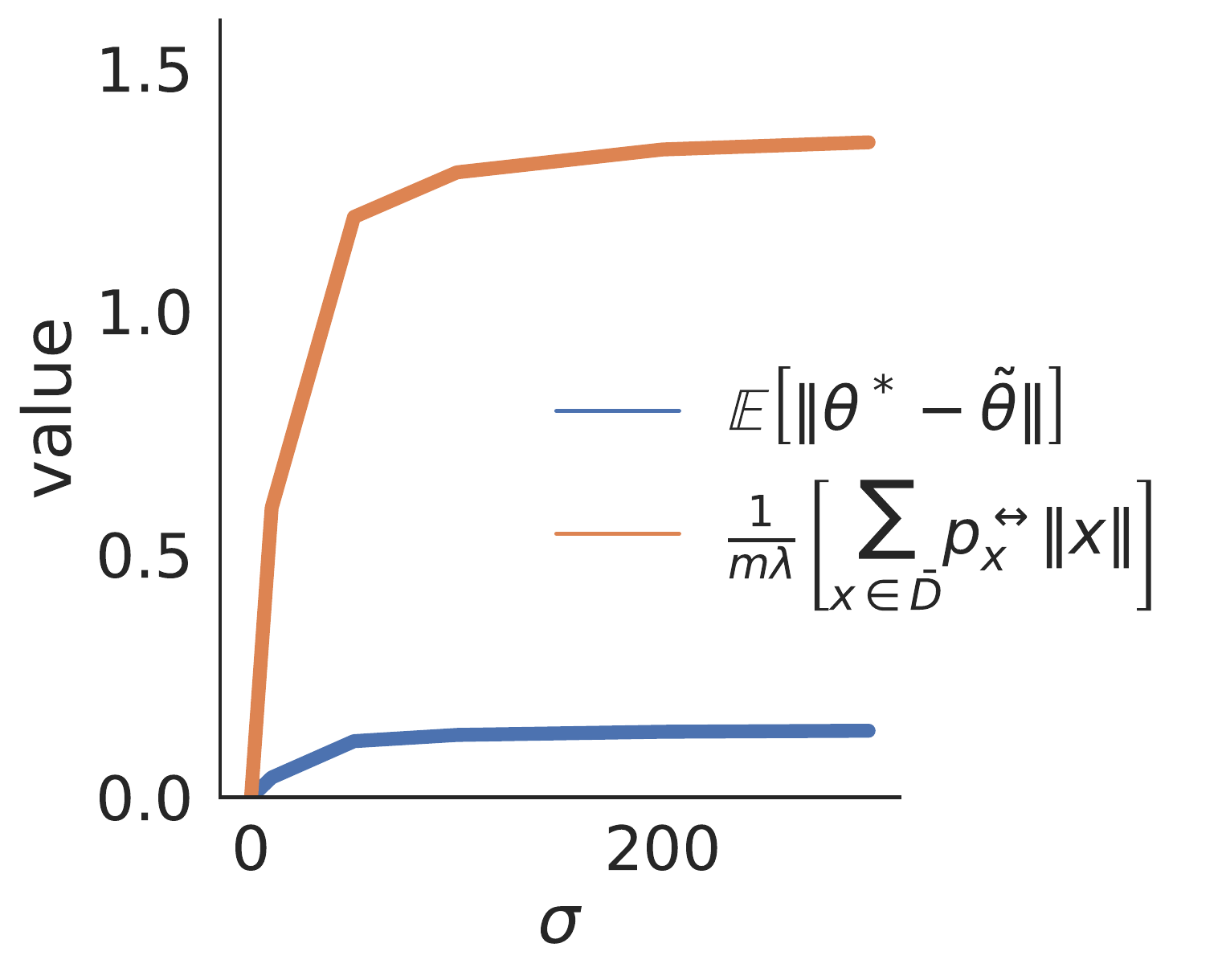}
\caption{Income dataset}
\end{subfigure}
\begin{subfigure}[b]{0.2\textwidth}
\includegraphics[width = 1.0\linewidth]{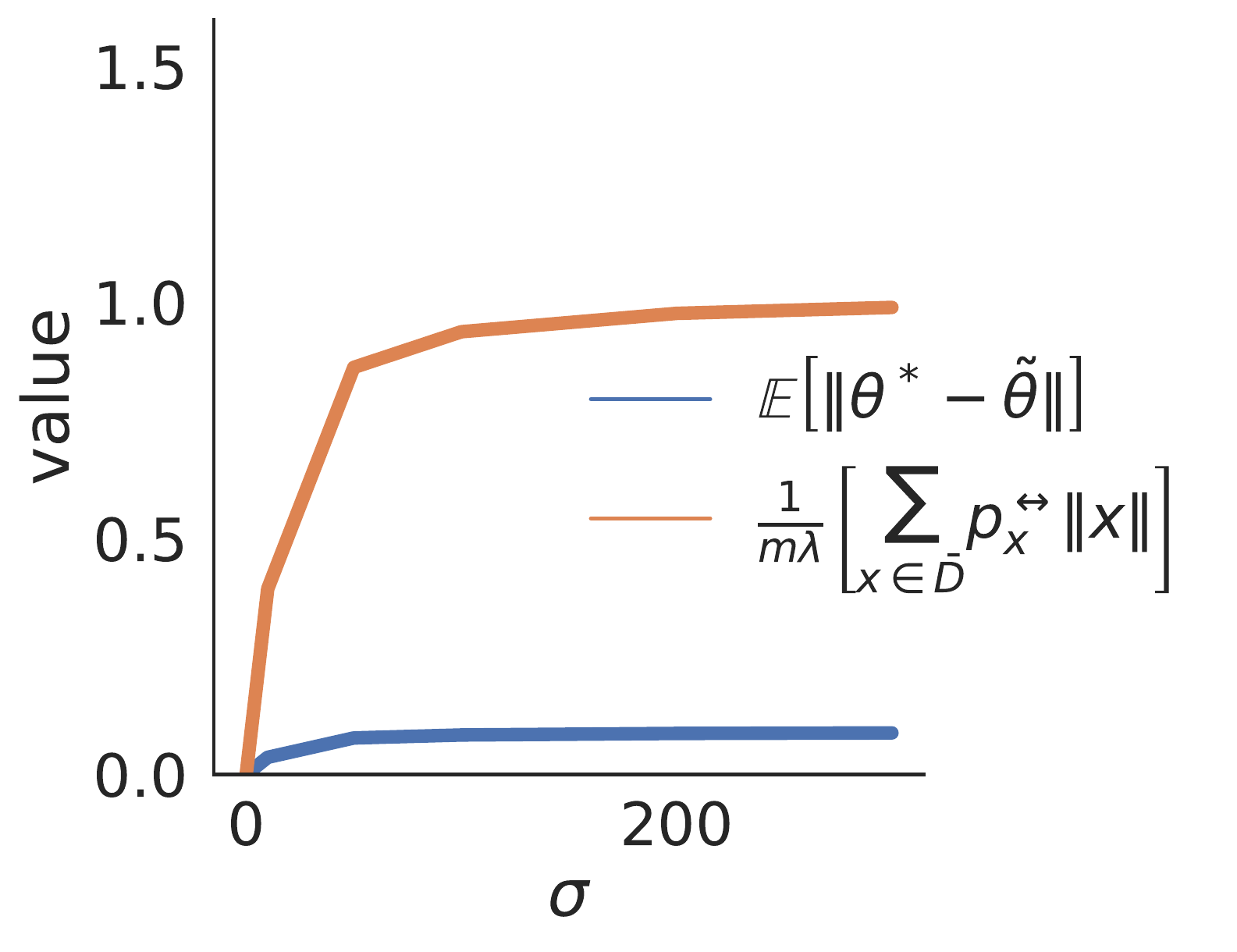}
\caption{Parkinsons dataset}
\end{subfigure}
\caption{Upper bound of the expected model deviation on 4 datasets with $\lambda = 20,k = 20$.}
\label{fig:cor1_bound}
\end{figure}

\begin{figure}
\centering
\begin{subfigure}[b]{0.2\textwidth}
\includegraphics[width = 1.0\linewidth]{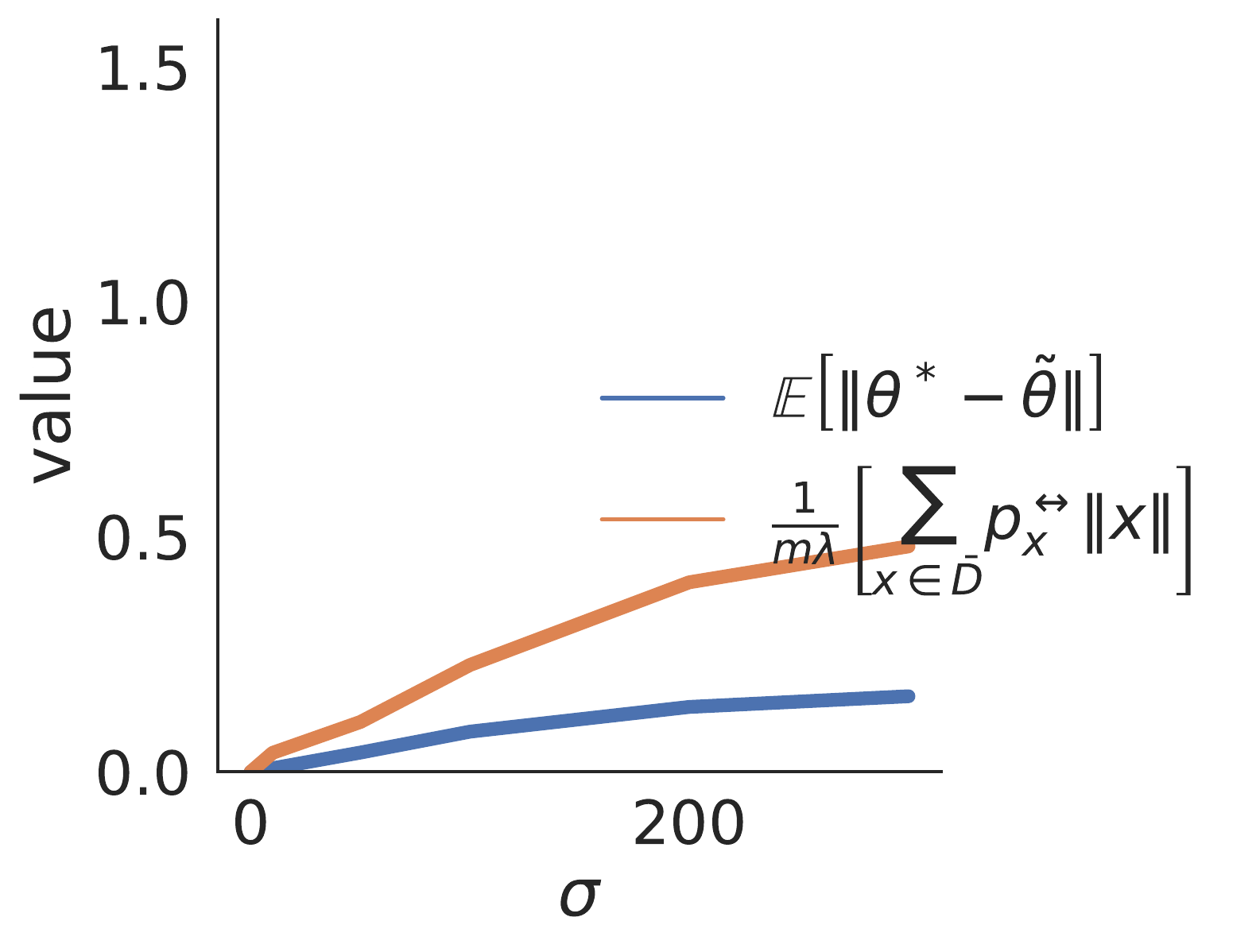}
\caption{Bank dataset}
\end{subfigure}
\begin{subfigure}[b]{0.2\textwidth}
\includegraphics[width = 1.0\linewidth]{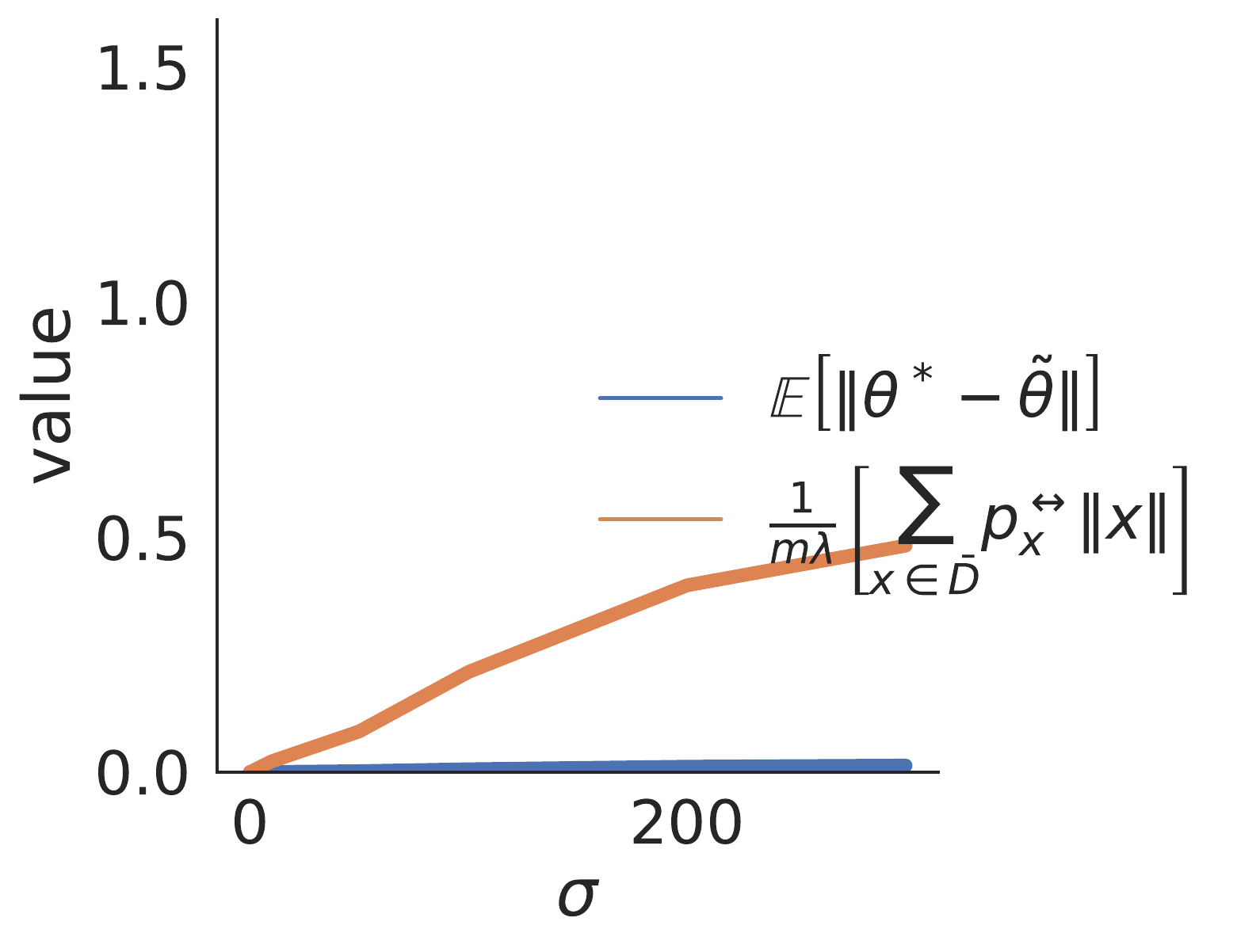}
\caption{Credit card dataset}
\end{subfigure}
\begin{subfigure}[b]{0.2\textwidth}
\includegraphics[width = 1.0\linewidth]{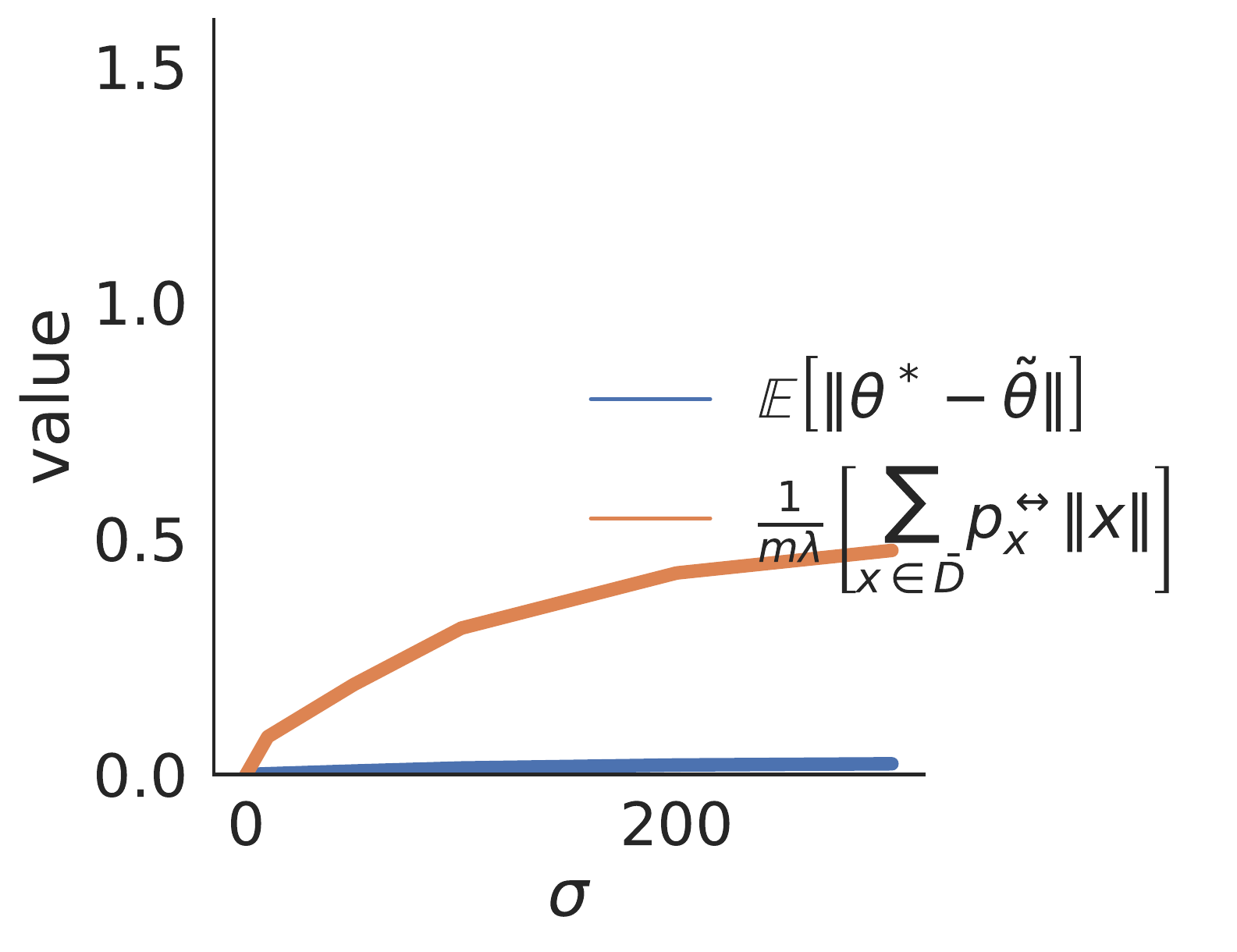}
\caption{Income dataset}
\end{subfigure}
\begin{subfigure}[b]{0.2\textwidth}
\includegraphics[width = 1.0\linewidth]{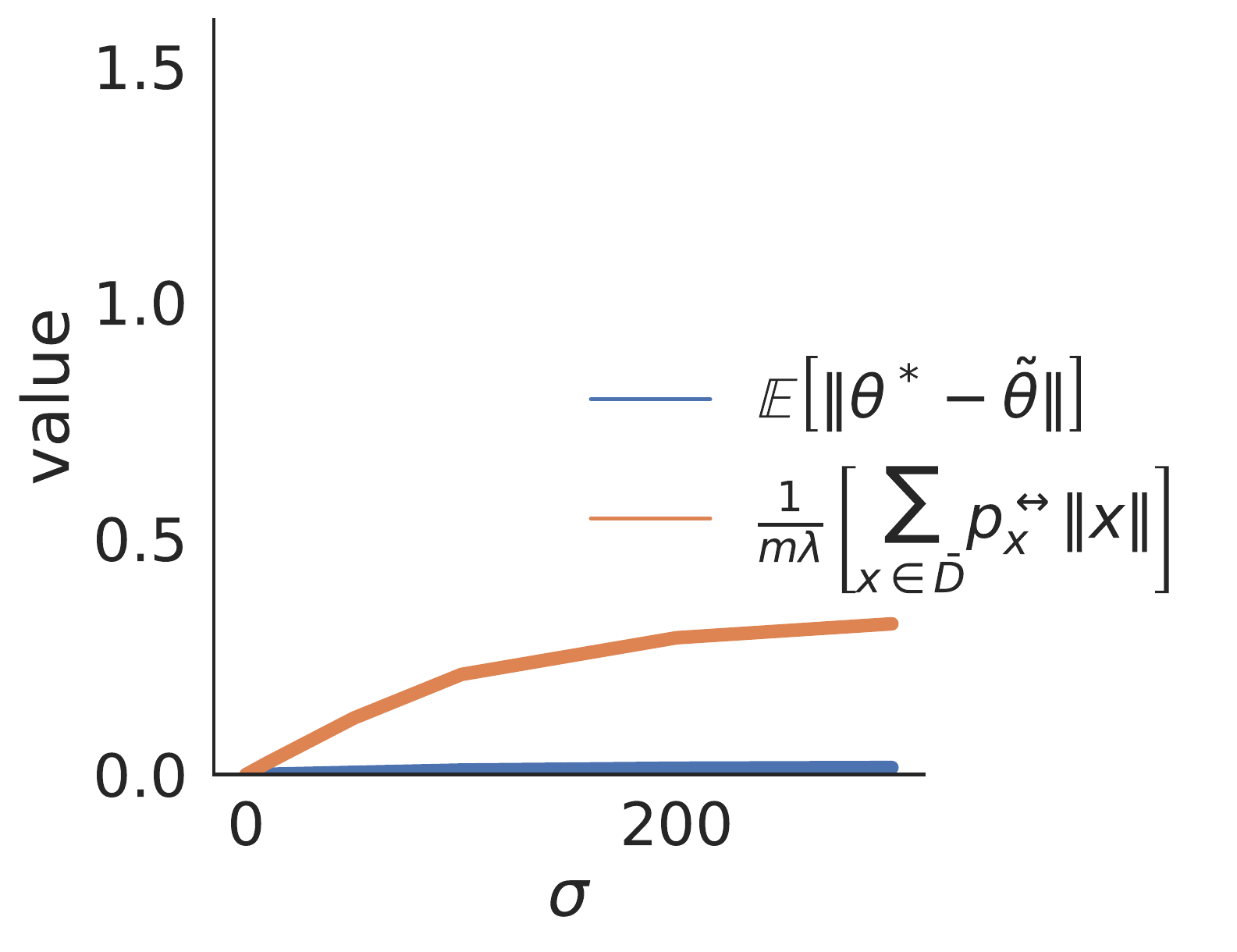}
\caption{Parkinsons dataset}
\end{subfigure}
\caption{Upper bound of the expected model  deviation on 4 datasets with $\lambda = 100, k=200$.}
\label{fig:cor1_bound2}
\end{figure}

\subsection{The impact of regularization parameter} 
This section provides further empirical supports regarding impact of the regularization parameter $\lambda$ to the accuracy and fairness trade-off. As seen from Theorem \ref{thm:3}, increasing $\lambda$ reduces the model deviation which in turns decreases the  group excessive risk $R(\bar{D}_{\leftarrow a})$ by Theorem \ref{thm:3} from the main text. On the other hand, large regularization can intuitively impacts negatively to the model accuracy. This was verified empirically in Figure \ref{app:fig:lambda_effect} which shows how model deviation(left), excessive risk difference between two groups (middle) and utility(right) vary according to $\lambda$. 

\begin{figure}
\centering
\begin{subfigure}[b]{0.44\textwidth}
\includegraphics[width = 1.0\linewidth]{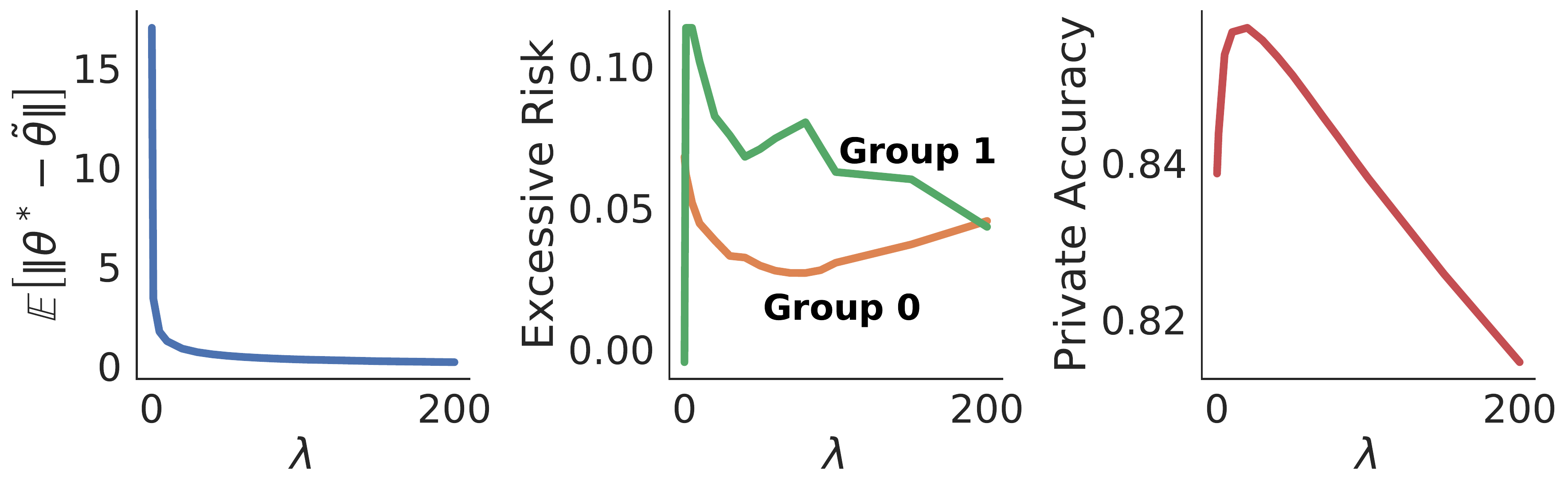}
\caption{Bank dataset}
\end{subfigure}
\begin{subfigure}[b]{0.44\textwidth}
\includegraphics[width = 1.0\linewidth]{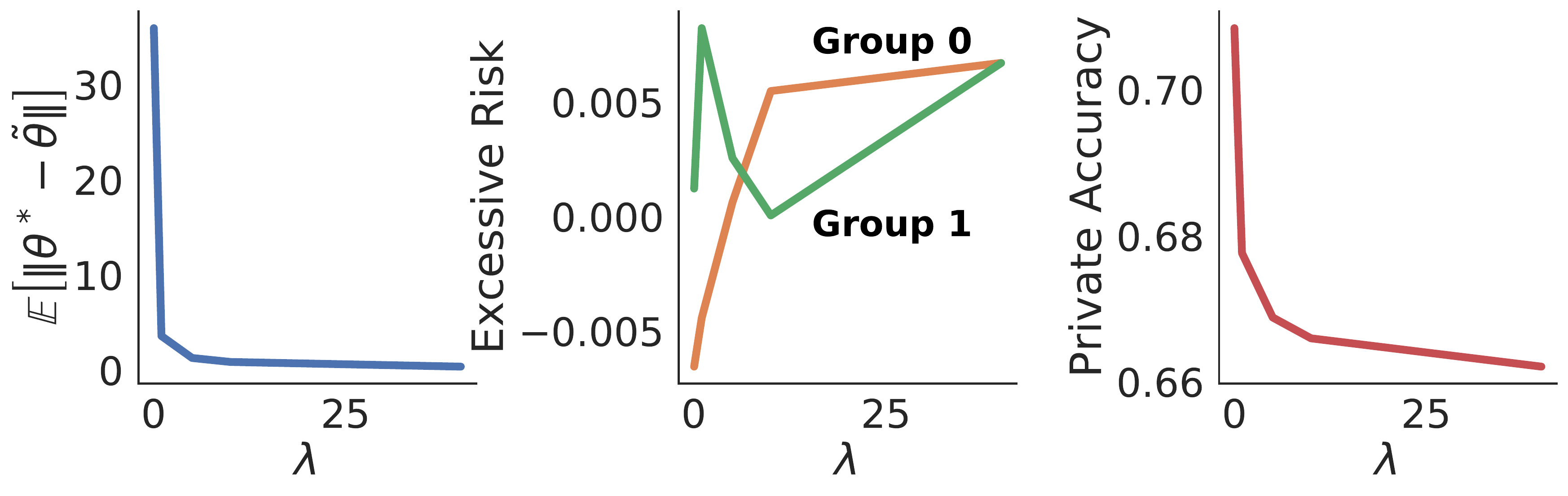}
\caption{Income dataset}
\end{subfigure}
\begin{subfigure}[b]{0.44\textwidth}
\includegraphics[width = 1.0\linewidth]{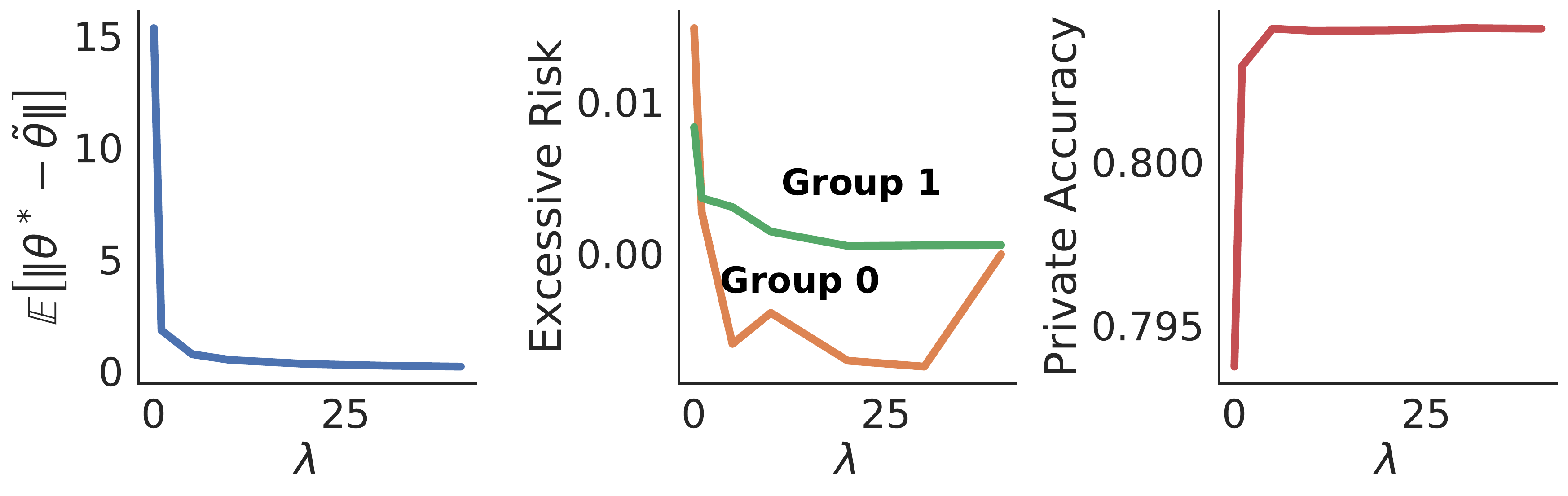}
\caption{Parkinsons dataset}
\end{subfigure}
\begin{subfigure}[b]{0.44\textwidth}
\includegraphics[width = 1.0\linewidth]{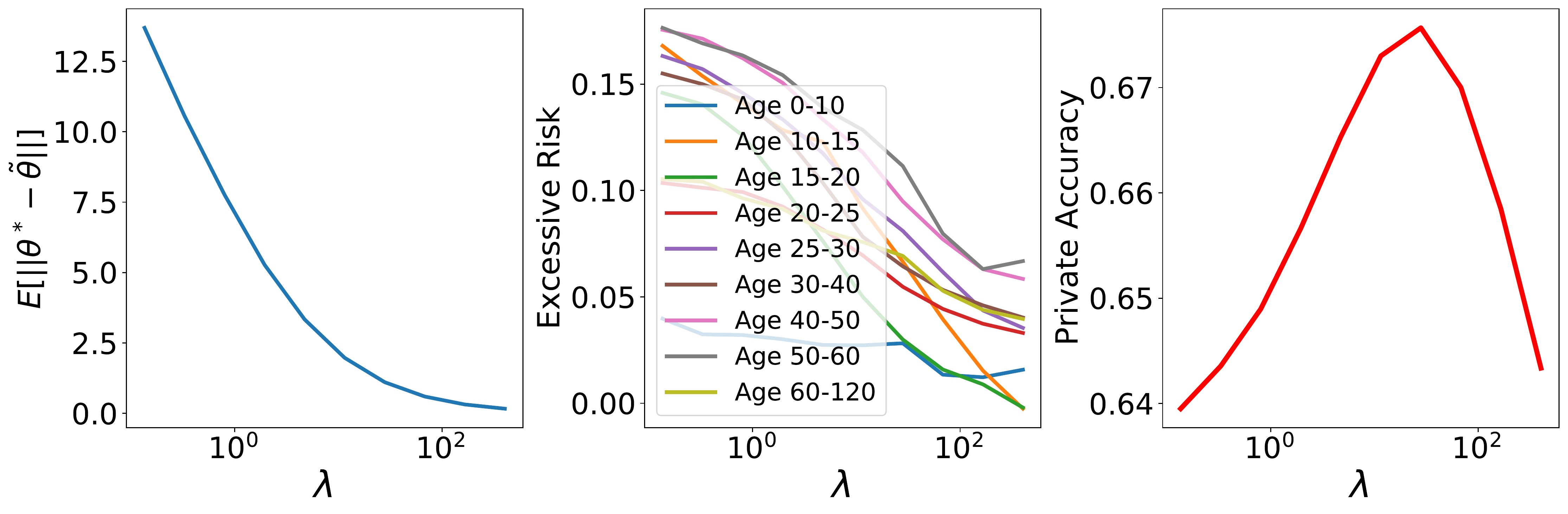}
\caption{UTKFace dataset}
\end{subfigure}
\caption{Expected model deviation (left), empirical risk (middle), and model accuracy (right) as a function of the  regularization. The experiments are performed with the following settings: $k =150, \sigma = 50$.}.
\label{app:fig:lambda_effect}
\end{figure}

\subsection{The impact of teachers ensemble size k} 
This section illustrates the effect of teacher ensemble sizes k to: 1) flipping probability $\fpx$, and 2) the trade-offs among model deviation $\mathbb{E}\left[ \Delta_{\tilde{\btheta}}\right]$, model's fairness and utilities. 

First, Theorem \ref{thm:4} from the main text shows that larger $k$ values correspond to smaller flipping probability $\fpx$. We provide more empirical evidence on other datasets and report the dependency between flipping probability with number of teachers $k$ in Figure \ref{fig:flip_prob_all}. It can be observed consistently on all datasets, the more number of teachers $k$, the smaller the flipping probability $\fpx$ over all samples $\bm{x}$ is.


\begin{figure}
\centering
\begin{subfigure}[b]{0.3\textwidth}
\includegraphics[width = 1.0\linewidth]{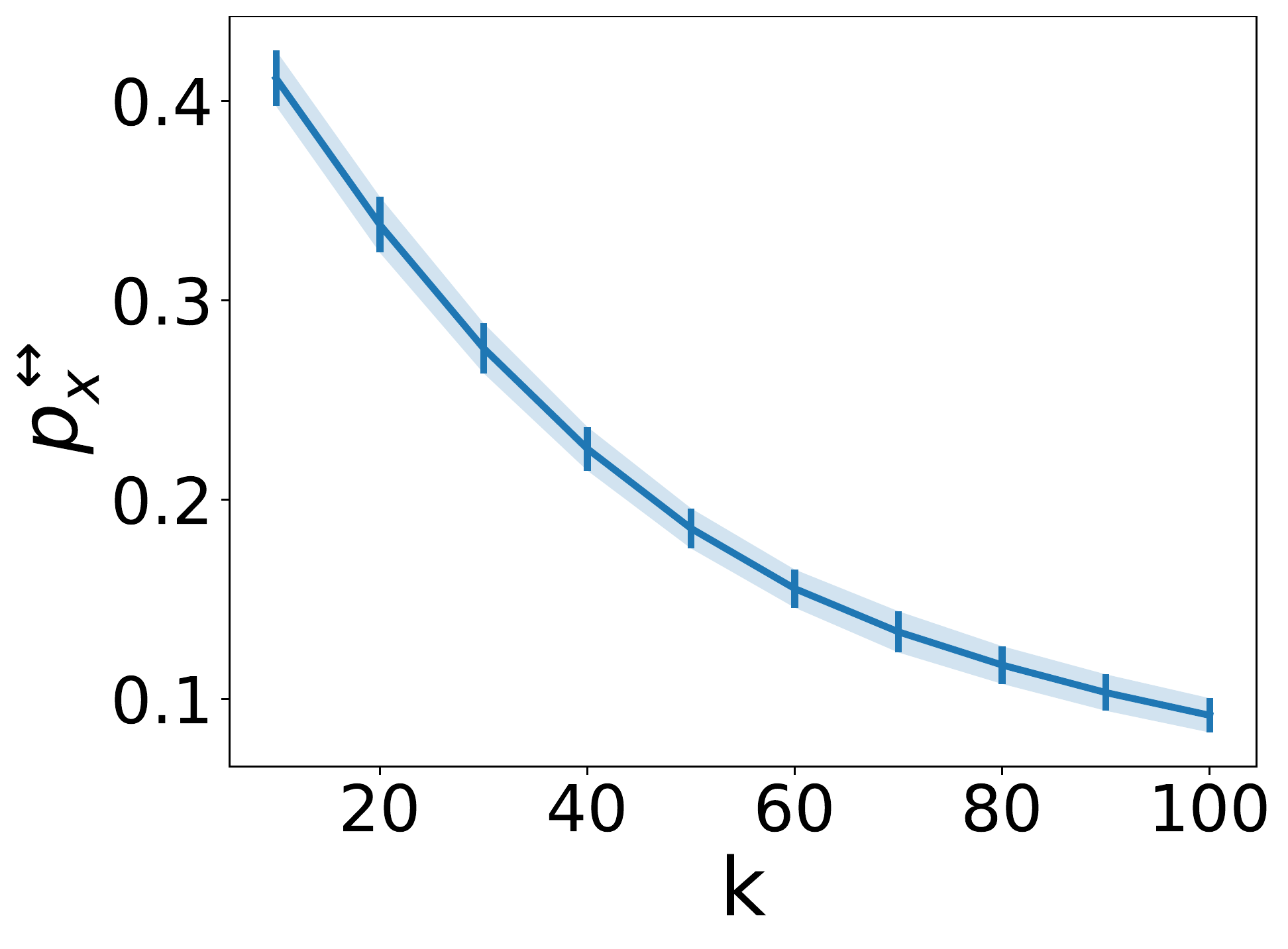}
\caption{Bank dataset}
\end{subfigure}
\begin{subfigure}[b]{0.3\textwidth}
\includegraphics[width = 1.0\linewidth]{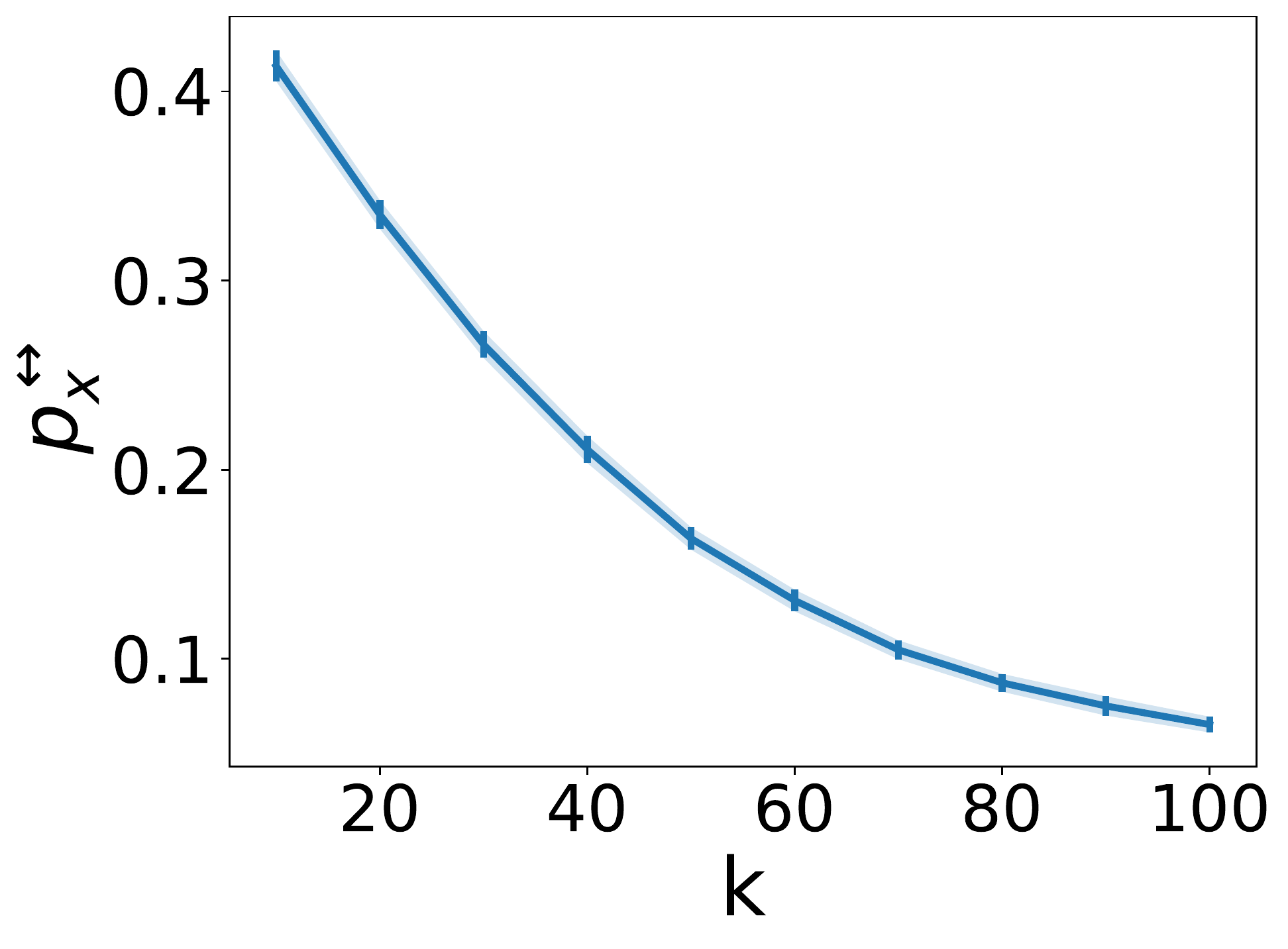}
\caption{Income dataset}
\end{subfigure}
\begin{subfigure}[b]{0.3\textwidth}
\includegraphics[width = 1.0\linewidth]{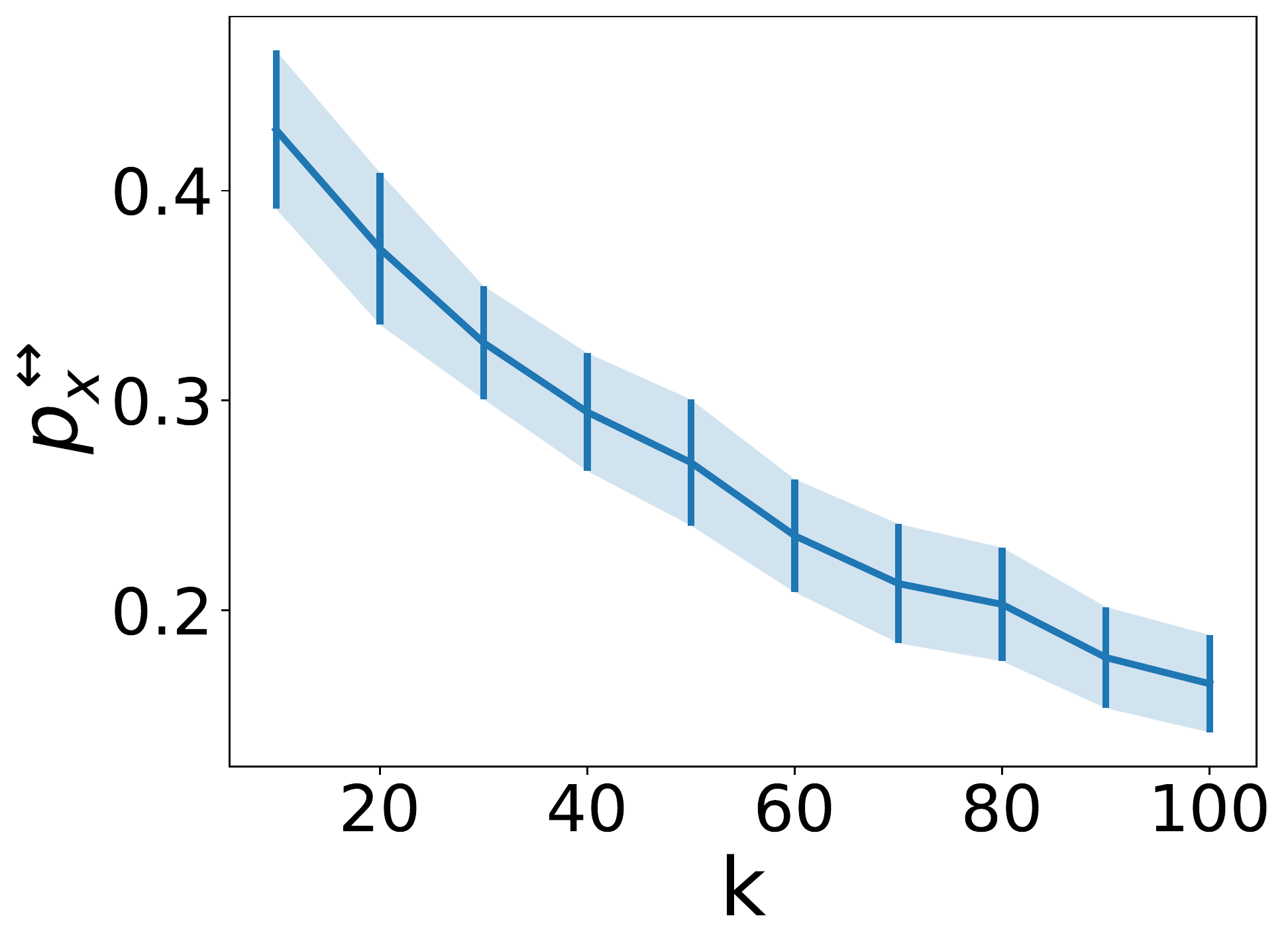}
\caption{Parkinsons dataset}
\end{subfigure}
\caption{ Average flipping probability $\fpx$  for samples $\bm{x} \in \bar{D}$ as a function of the ensemble size $k$. }
\label{fig:flip_prob_all}
\end{figure}

Second, regarding to the fairness analysis, similar to the previous subsection, we provide additional empirical supports on the effects of $k$ on the model deviation, the  difference between the group excessive risk, and the utility  of the PATE models. We report these metrics on the other three benchmarks datasets in Figure \ref{fig:k_effect}. A similar trend with the regularization parameter $\lambda$  also holds for the parameter $k$ here. When the parameter $k$ is increased to a large enough value, both model deviation and accuracy decreases, but the unfairness measured by the excessive risk difference between two groups reduces. This can be explained by looking again Figure \ref{fig:flip_prob_all} and Theorem \ref{thm:4} from the main text. A large number of teachers $k$ results to a  smaller flipping probability which in turns reduces the model deviation. By Theorem \ref{thm:3} a small model deviation can reduce the level of unfairness.


\begin{figure}
\centering
\begin{subfigure}[b]{0.44\textwidth}
\includegraphics[width = 1.0\linewidth]{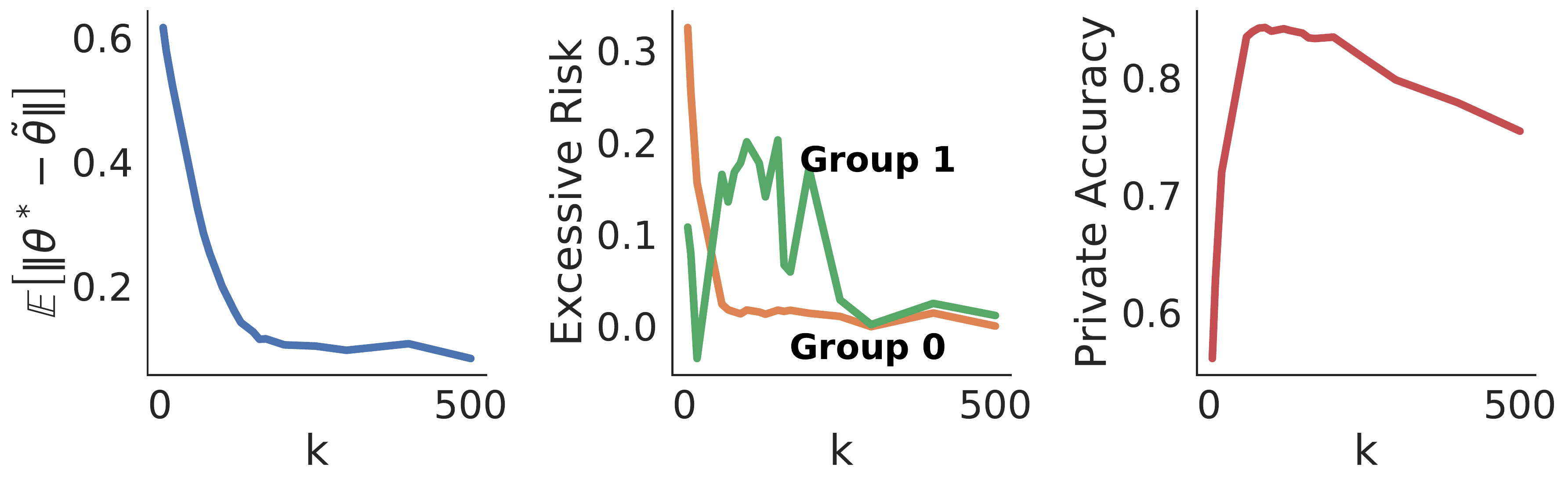}
\caption{Bank dataset}
\end{subfigure}
\begin{subfigure}[b]{0.44\textwidth}
\includegraphics[width = 1.0\linewidth]{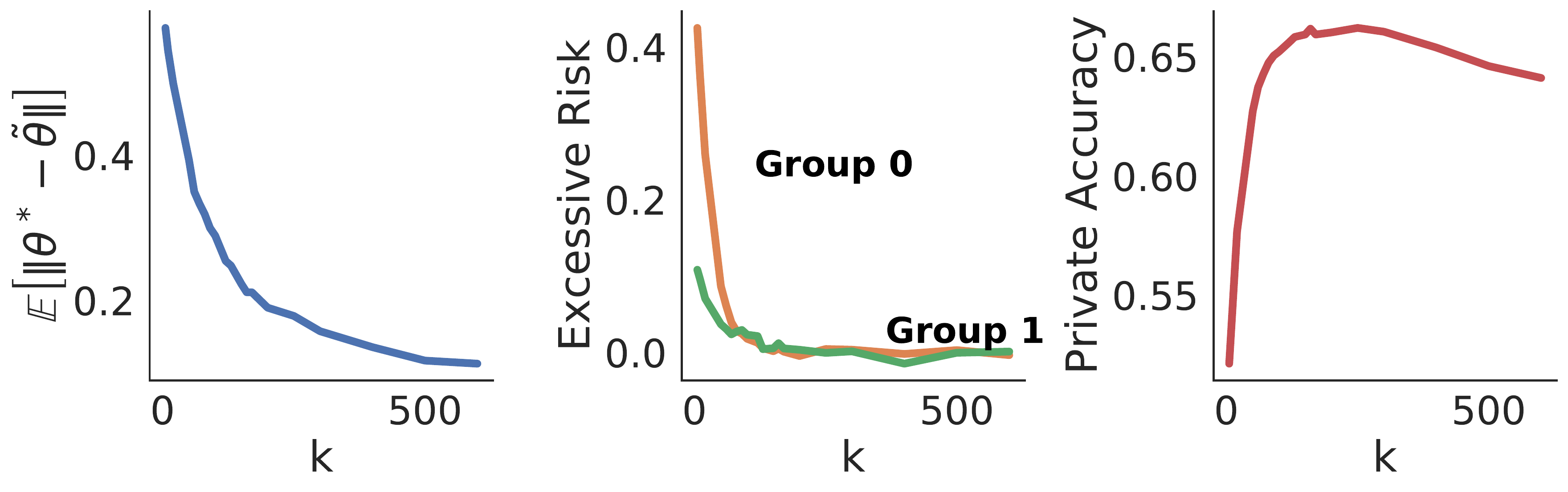}
\caption{Income dataset}
\end{subfigure}
\begin{subfigure}[b]{0.44\textwidth}
\includegraphics[width = 1.0\linewidth]{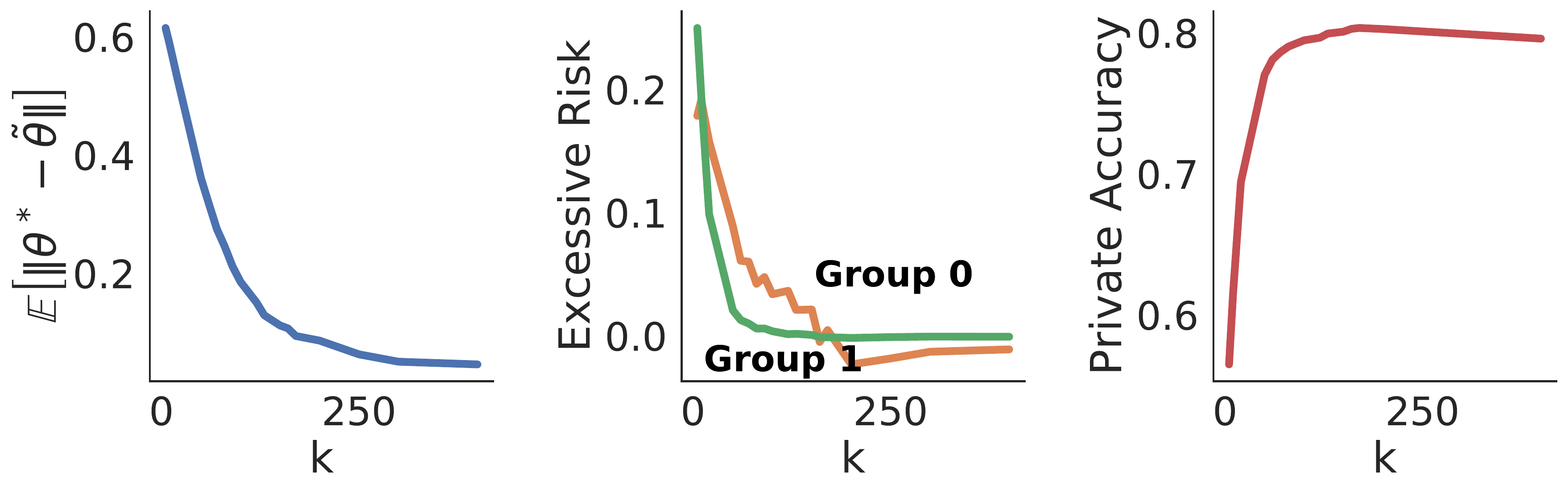}
\caption{Parkinsons dataset}
\end{subfigure}
\begin{subfigure}[b]{0.44\textwidth}
\includegraphics[width = 1.0\linewidth]{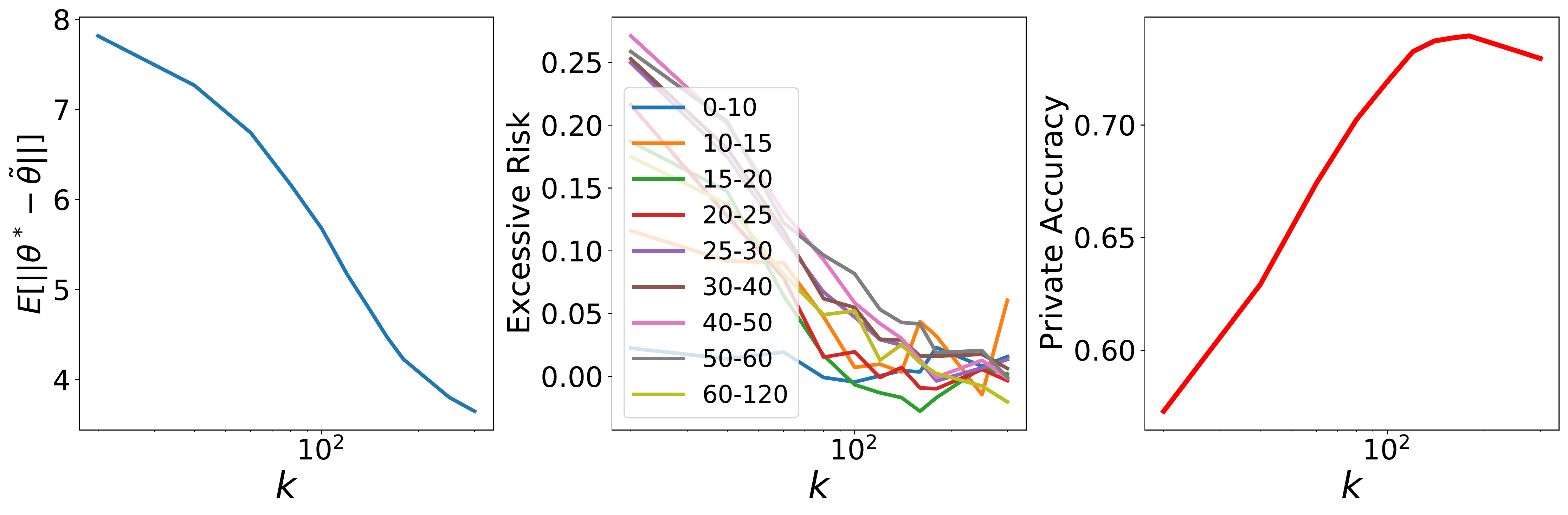}
\caption{UTKFace dataset}
\end{subfigure}
\caption{Expected model deviation (left), empirical risk (middle), and model accuracy (right) as a function of the ensemble size. The experiments are performed with the following settings: $\lambda = 100$, $\sigma = 50$.}.
\label{fig:k_effect}
\end{figure}

\subsection{The impact of the data input norm}
This section provides further experimental results regarding relation between (1) the input norm  with the private model deviation and (2) the input norm with its excessive risk.

Regarding the first relation, Corollary  \ref{cor:1} from the main text implies that the smaller the input norm $\|\bm{x}\|$ is the smaller the model deviation is. For each dataset, we then vary the range of the input norm $\| \bm{x} \| $ and report the associated values of the expected model deviation in Figure \ref{fig:inputnorm_vs_expected_diff}. It can be seen clearly from the Figure  \ref{fig:inputnorm_vs_expected_diff}, a monotone connection between input norm and the model deviation which verifies the statement from Corollary  \ref{cor:1}.

On the other hand, the input norm can affect the excessive risk  by Lemma \ref{a:thm:1} of the Appendix,  the individuals or group of individuals of large gradient norm can suffer from large excessive risk. In other words, the individuals of large data norm which are often observed at the tail of data can loose more accuracy. To confirm such claims, we report in Figure \ref{fig:corr_norm_all} the Spearman correlation between input norm and the excessive risk at individual levels. On all datasets, we can see obviously a positive relationship between data input norm and the excessive risk.

\begin{figure}
\centering
\begin{subfigure}[b]{0.3\textwidth}
\includegraphics[width = 1.0\linewidth]{images/impact_input_norm_income.pdf}
\caption{Income dataset}
\end{subfigure} 
\begin{subfigure}[b]{0.3\textwidth}
\includegraphics[width = 1.0\linewidth]{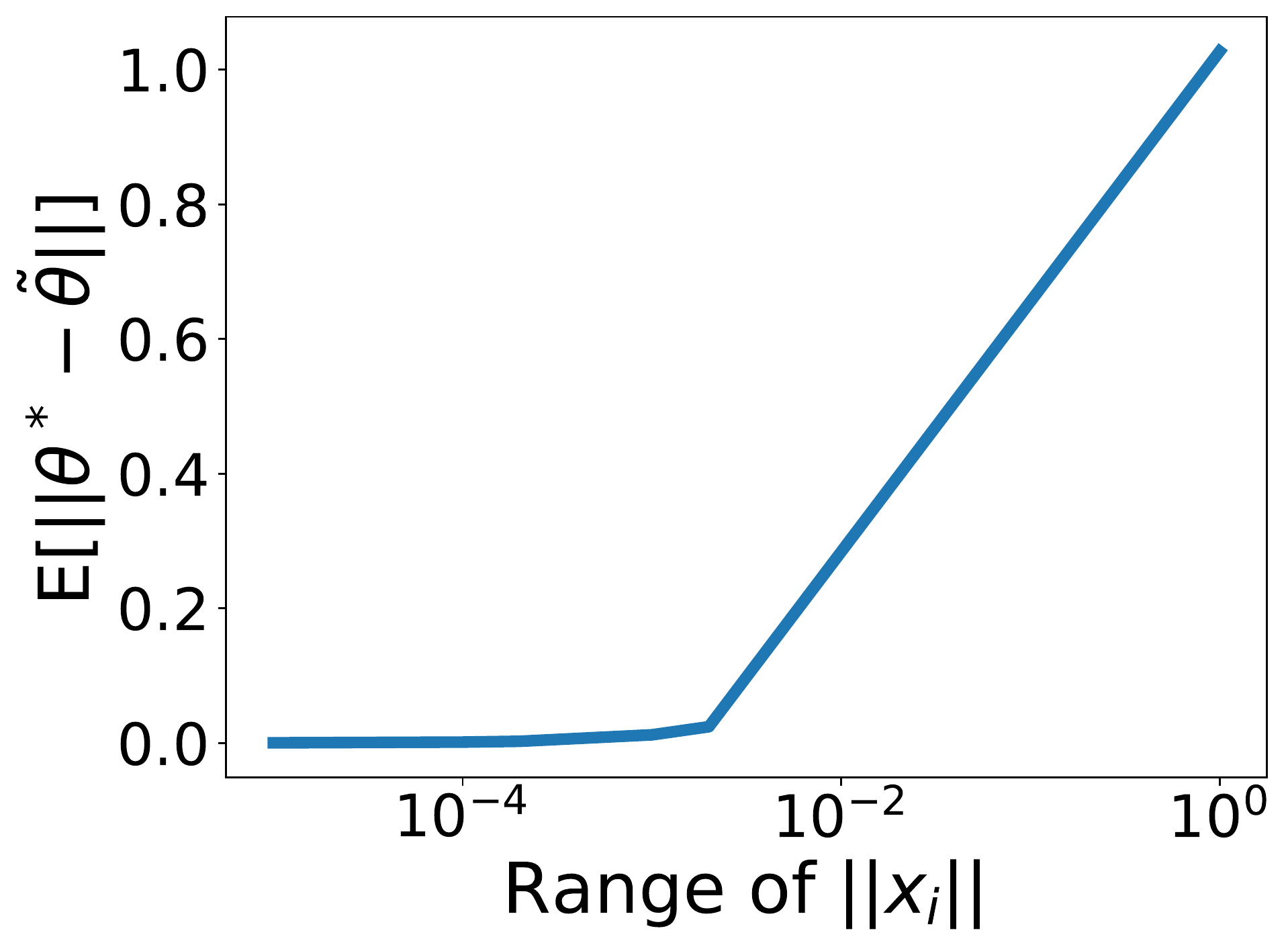}
\caption{Bank dataset}
\end{subfigure}
\begin{subfigure}[b]{0.3\textwidth}
\includegraphics[width = 1.0\linewidth]{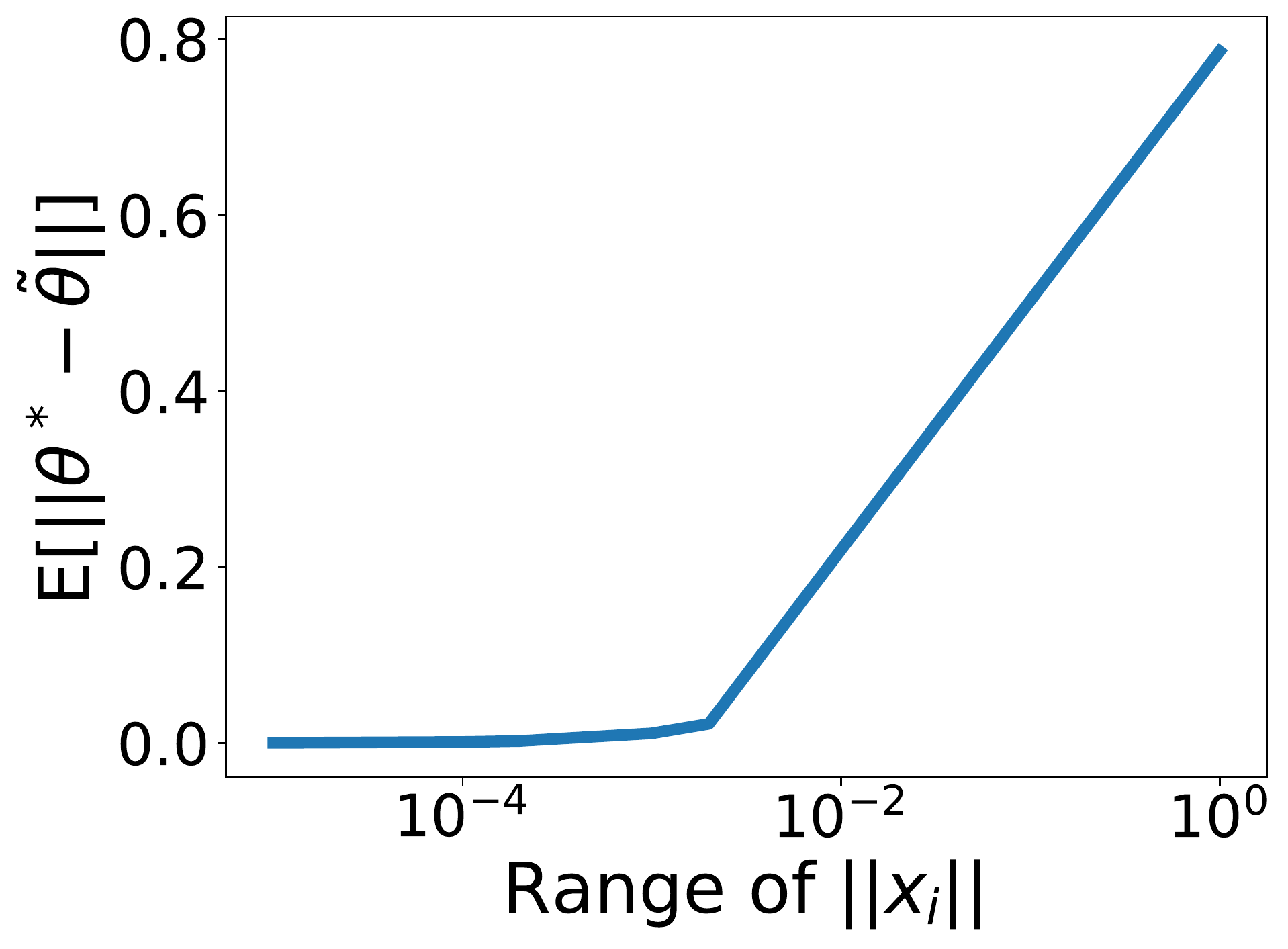}
\caption{Parkinsons dataset}
\end{subfigure}
\caption{Relation between input norm and model deviation.}
\label{fig:inputnorm_vs_expected_diff}
\end{figure}

\begin{figure}
\centering
\begin{subfigure}[b]{0.49\textwidth}
\includegraphics[width = 1.0\linewidth]{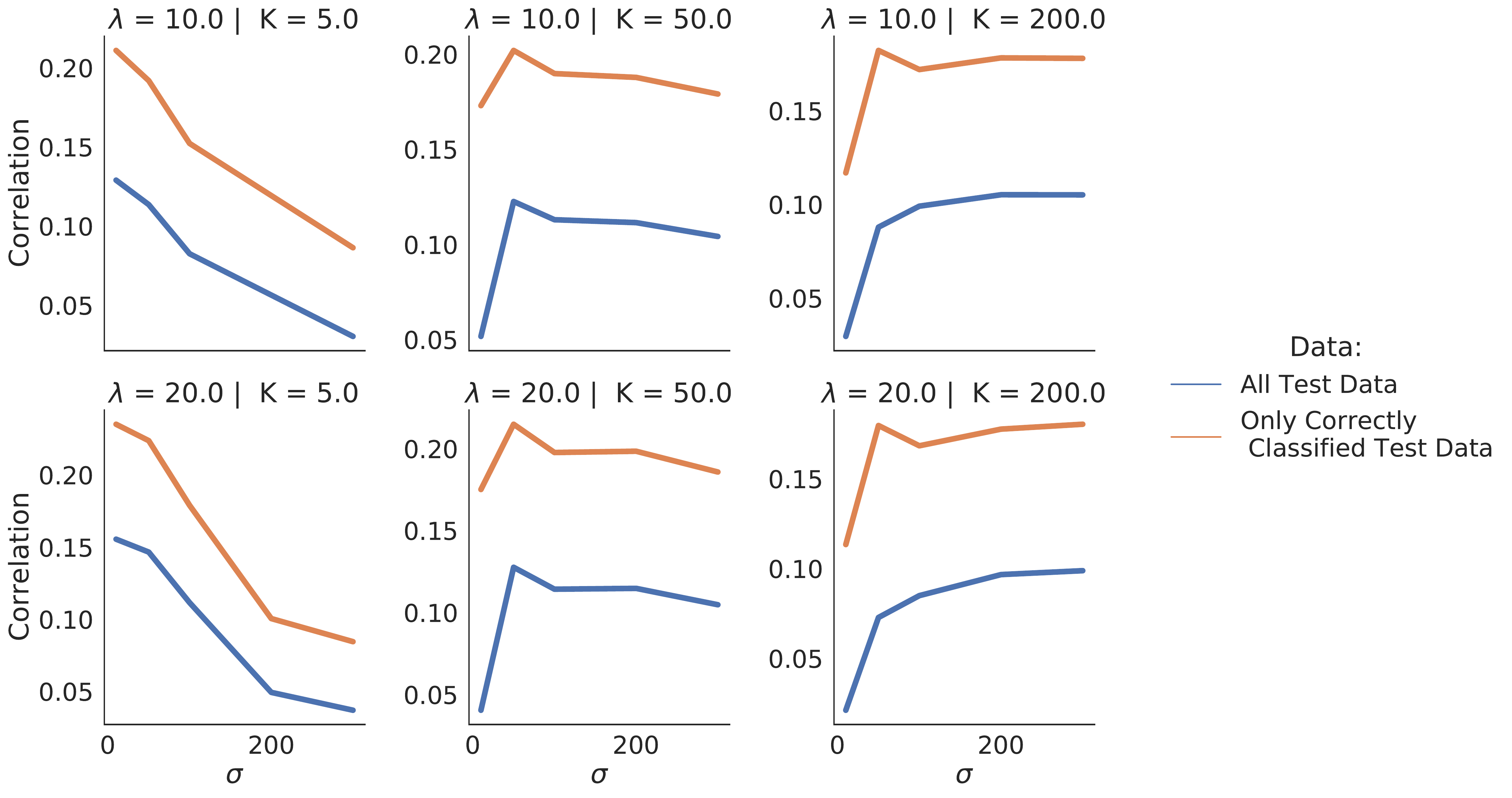}
\caption{Bank dataset}
\end{subfigure} 
\begin{subfigure}[b]{0.49\textwidth}
\includegraphics[width = 1.0\linewidth]{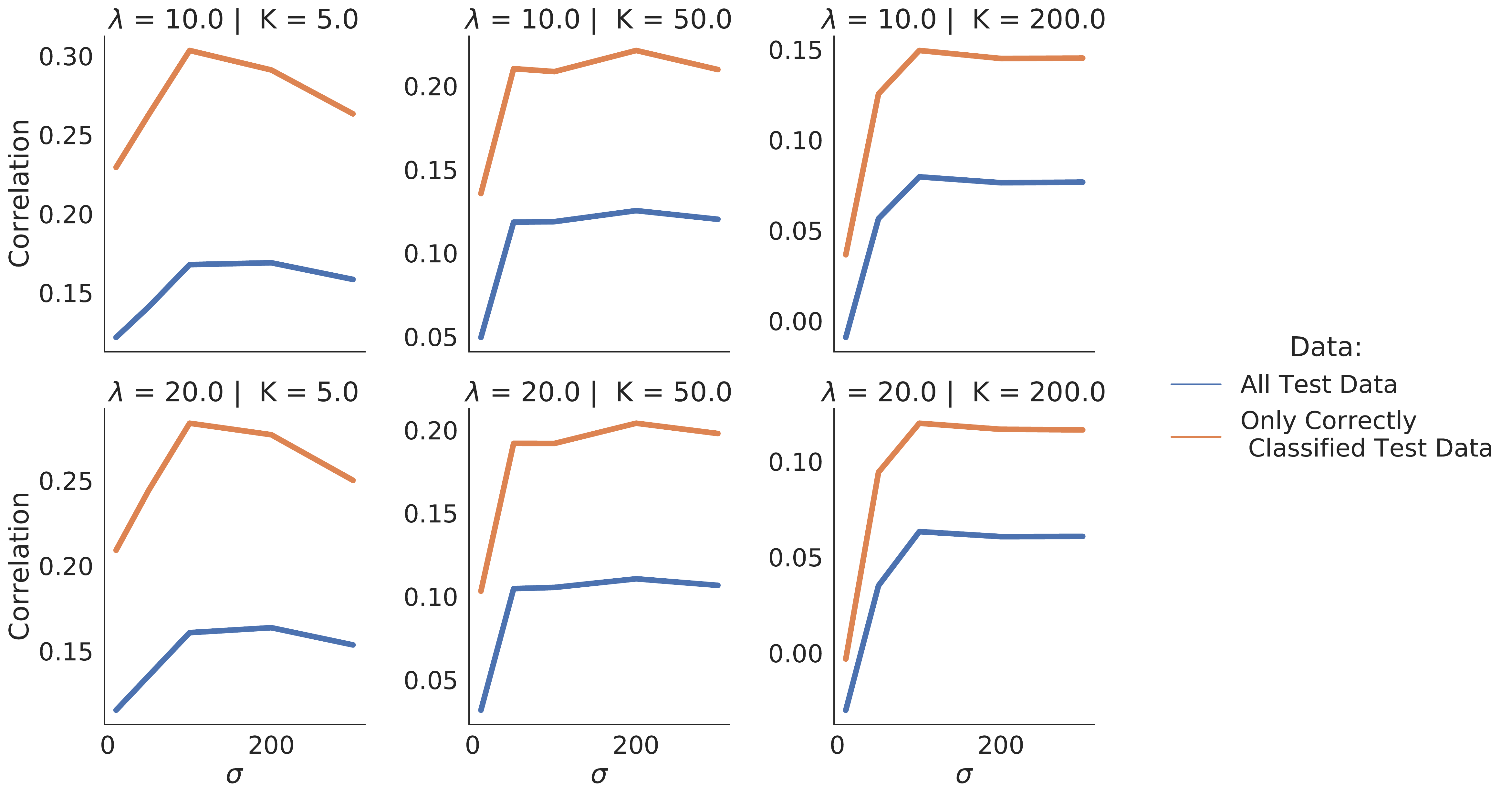}
\caption{Income dataset}
\end{subfigure}
\begin{subfigure}[b]{0.49\textwidth}
\includegraphics[width = 1.0\linewidth]{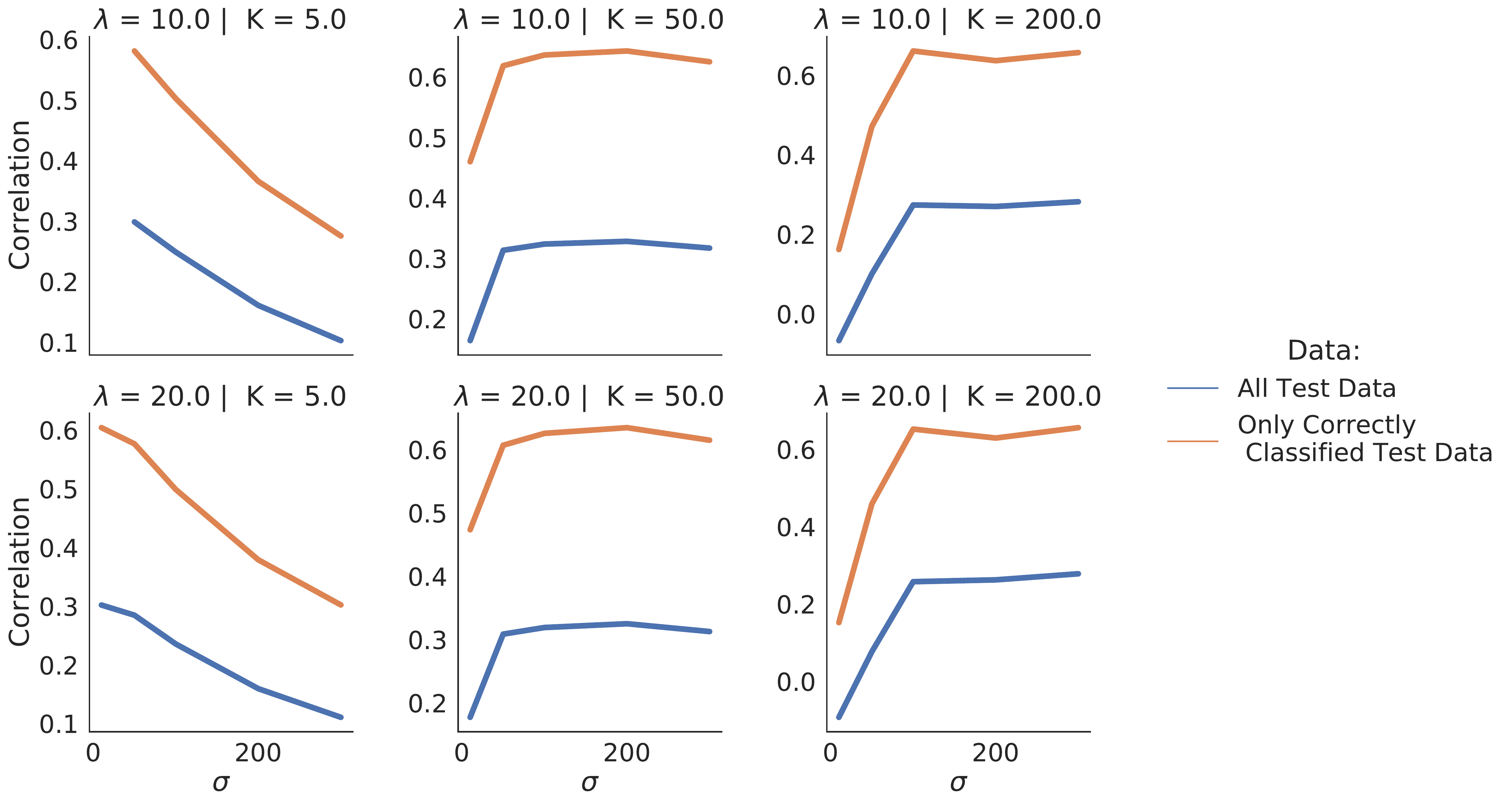}
\caption{Credit card dataset}
\end{subfigure}
\begin{subfigure}[b]{0.49\textwidth}
\includegraphics[width = 1.0\linewidth]{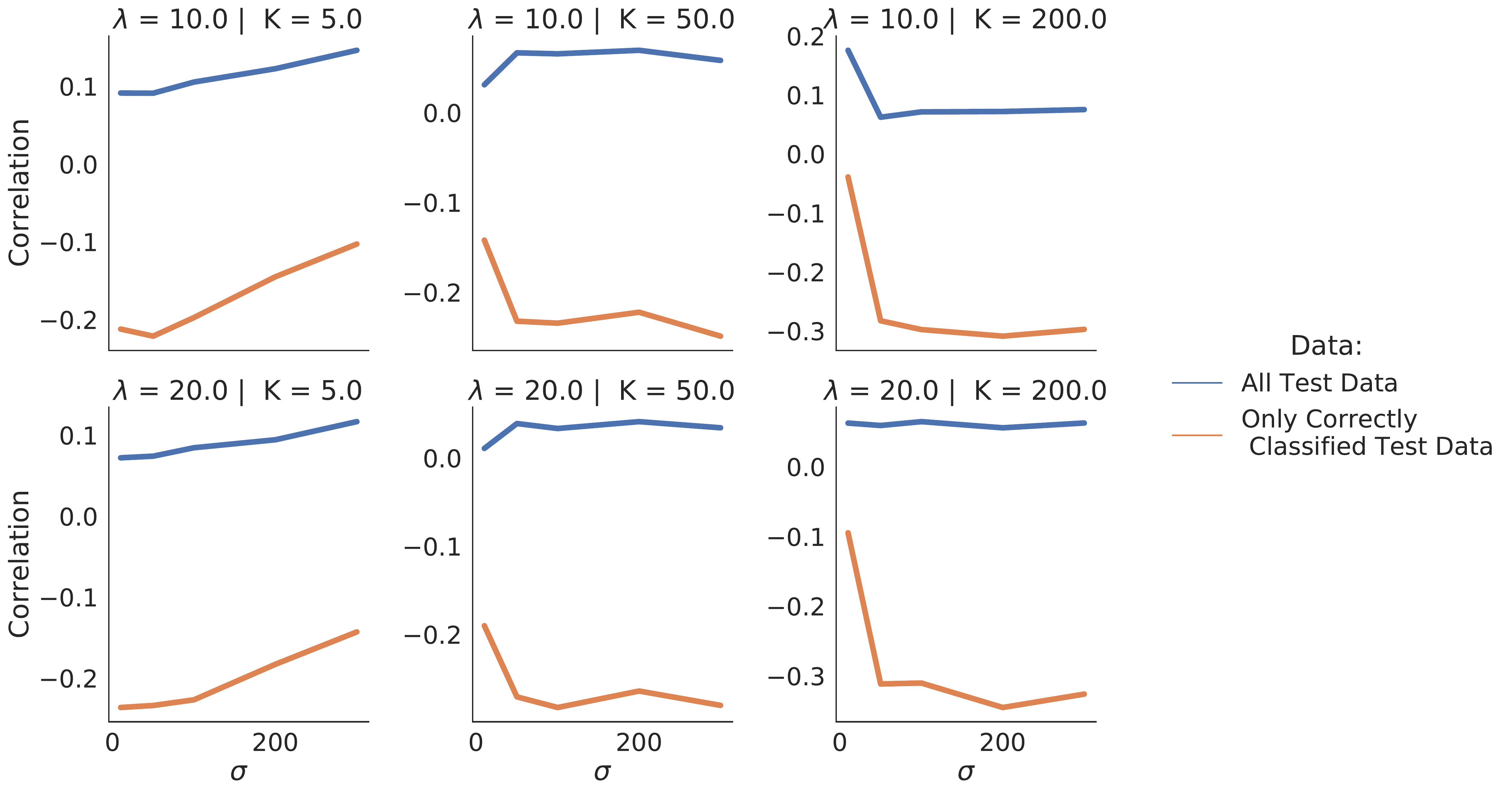}
\caption{Parkinsons dataset}
\end{subfigure}
\caption{Correlation between the excessive risk and input norm on 5 datasets. The experiments are performed with the following settings: $\lambda = 100$, $\sigma = 50, k = 150$.}
\label{fig:corr_norm_all}
\end{figure}

\subsection{Connection between input norm and smoothness parameter $\beta_a$}
\label{app:smoothness}
 It is noted that the smoothness parameter $\beta_{\bm{a}}$
captures the local flatness of the loss function of a particular group $a$. Consider the logistic regression classifier, then the smoothness parameter $\ell(f_{\btheta}(\bm{x}), y)$  for one particular data point is given by $\beta_x = 0.25 \| \bm{x} \|$ \cite{shi2021aisarah}. Recall the following important property of the smooth function: If  $L = \sum_{i} \ell_i $ and each $\ell_i$ is $\beta_i$-smooth then $L$ is $\max_i \beta_i$-smooth.  Because of that, the smoothness parameter $\beta_a$ for one particular group $a$ is given by: $\beta_a = 0.25  \max_{\bm{x} \in D_a}\| \bm{x} \|$


The above clearly illustrates the relationship between input norms $\|\bm{x} \|$ and 
the smoothness parameters $\beta_{\bm{a}}$. 


\subsection{Connection between input norm and gradient norm}
\label{app:sec_connection}
\begin{figure*}
    \centering
    \includegraphics[width=0.75\linewidth]{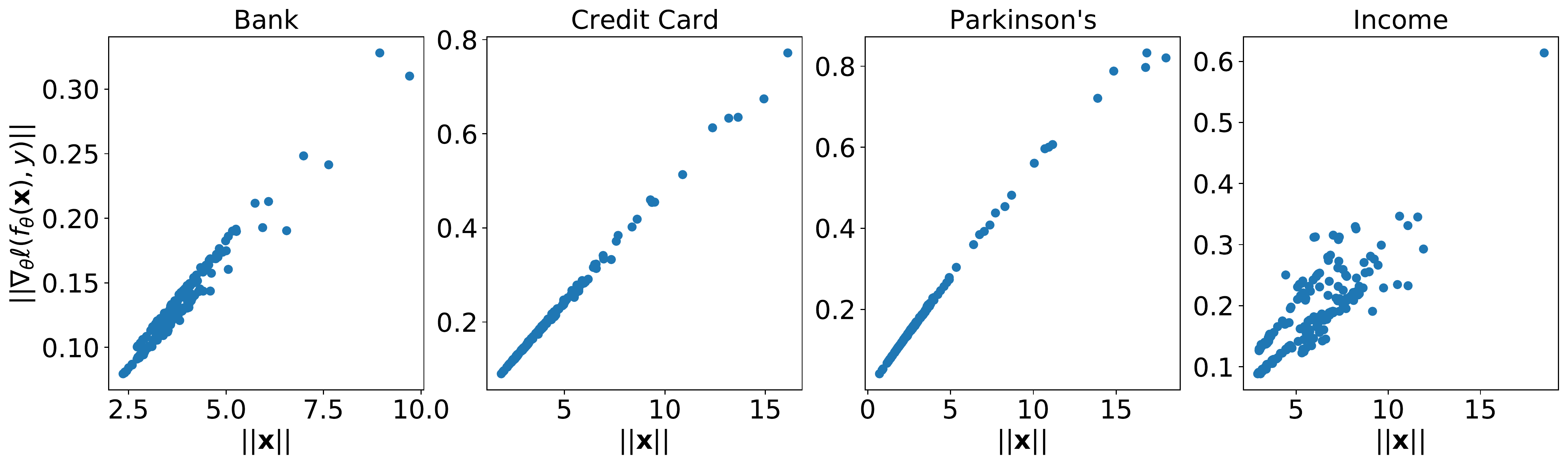}
    \caption{Relation Between Gradient Norm and Input Norm on all datasets.}
    \label{fig:grad_inp_corr_all}
\end{figure*}

In the main text, we have described, for logistic regression classifiers, there is a strong relation between the individual input norm $\| \bm{x}\|$  and their gradient norm at optimal parameter $  \| \nabla_{\optimal{\btheta}}  
\ell(\bar{f}_{\optimal{\btheta}}(\bm{x}),y)\|$. In this subsection, we extend the analysis for non-linear model. In particular, we show a similar connection between the gradient norm and the input norm for a neural network with a single hidden layer. We start by considering the following settings:
\paragraph{Settings} 

Consider a neural network model $\bar{f}_{\optimal{\btheta}}
 (\bm{x}) \defeq \textsl{softmax} \left(\optimal{\btheta}_1^T \tau(\optimal{\btheta}^T_2 \bm{x}) \right)$ 
 where $\bm{x} = (\bm{x}^i)_{i=1}^d $ is a $d$ dimensional input vector, the parameters $\optimal{\btheta}_2 \in \RR^{d \times H}, \optimal{\btheta}_1 \in \RR^{H \times C}$ 
 and the cross entropy loss $\ell(\bar{f}_{\optimal{\btheta}}(\bm{x}),y) = -\sum_{c=1}^C y_c \log 
 \bar{f}_{\optimal{\btheta}, c}(\bm{x})$ where $\tau(\cdot)$ 
 is a proper activation function, e.g., a sigmoid function.
 Let $\bm{O} = \tau(\optimal{\btheta}^T_2 \bm{x}) \in \RR^H$
 be the vector $(O_1, \ldots, O_H)$ of $H$ hidden nodes of the network. 
 Denote the variables $h_j = \sum_{i=1}^d \optimal{\btheta}_{2,j,i} \bm{x}^i$ as the $j$-th hidden unit 
 before the activation function. 
 Next, denote $\optimal{\btheta}_{1,j,k} \in \RR$ as the weight parameter that connects 
 the $j$-th hidden unit $h_j$ with the $c$-th output unit $\bar{f}_c$ and  
 $ \  \optimal{\btheta}_{2,i,j} \in \RR$ as the weight parameter that connects the $i$-th 
 input unit $\bm{x}^i$ with the $j$-th hidden unit $h_j$.

Given the settings above, we now show the dependency between gradient norm and input norm. First notice that we can decompose the gradients norm of this neural network into two layers as follows:
\begin{equation}
    \| \nabla_{\optimal{\btheta}}\ell(\bar{f}_{\optimal{\btheta}}(\bm{x}),y) \|^2 =  \| \nabla_{\optimal{\btheta}_1}\ell(\bar{f}_{\optimal{\btheta}}(\bm{x}),y) \|^2 + \| \nabla_{\optimal{\btheta}_2}\ell(\bar{f}_{\optimal{\btheta}}(\bm{x}),y) \|^2.
\end{equation}
We will show that $\nabla_{\optimal{\btheta}_2}\ell(\bar{f}_{\optimal{\btheta}}(\bm{x}),y) \| \propto  \|\bm{x}\|.$

Notice that: 
$$
\|\nabla_{\optimal{\btheta}_2}\ell(\bar{f}_{\optimal{\btheta}}(\bm{x}),y) \|^2 = \sum_{i,j} \| \nabla_{\optimal{\btheta}_{2,i,j}}\ell(\bar{f}_{\optimal{\btheta}}(\bm{x}),y) \|^2.
$$

Applying, Equation (14) from \citet{gradient_formula}, it follows that: 
\begin{equation}
   \nabla_{\optimal{\btheta}_{2,i,j}}\ell(\bar{f}_{\optimal{\btheta}}(\bm{x}),y) = 
   \sum_{c=1}^C \left(y_c - \bar{f}_{\optimal{\btheta},c}(\bm{x}) \right) \, 
   \optimal{\btheta}_{1,j,c}\left(O_j(1 - O_j) \right) \bm{x}^i,
\end{equation}
which highlights the dependency of the gradient norm 
$\| \nabla_{\optimal{\btheta}_2}\ell(\bar{f}_{\btheta}(\bm{x}),y) \|$ and the input norm $\| \bm{x}\|$. Figure \ref{fig:grad_inp_corr_all} provides an empirical evidence for this dependency on all four datasets used in our analysis. It can be seen clearly a strong positive correlation between input norm and the gradient norm at individual levels on all datasets from Figure \ref{fig:grad_inp_corr_all}.  

\subsection{Effectiveness of mitigation solution}

This subsection provides extended empirical results regarding the effectiveness of our proposed mitigation solution which was presented in Section \ref{sec:mitigation}. 

We report the comparison between training PATE with hard and soft labels when $k=20$ in Figure \ref{fig:mitigation_solution_K_20} and when $k=150$ in Figure \ref{fig:mitigation_solution_K_150}. These figures  again illustrate the effects of the proposed mitigating solution in terms of utility/fairness tradeoff on the private student model. The top
subplots of each figure show the group excessive risks $R(\bar{D}_{\leftarrow 0})$ and $R(\bar{D}_{\leftarrow 1})$ associated with two 
groups while the bottom subplots illustrate the accuracy of the model, at increasing of the privacy loss $\epsilon$. Recall that our mitigation solution does not require the availability of group labels during training. This challenging settings are of importance under the scenario when it is not feasible to collect or use protected features (e.g., under GDPR).


\begin{figure*}
\centering
\begin{subfigure}[b]{0.4\textwidth}
\includegraphics[width=\textwidth]{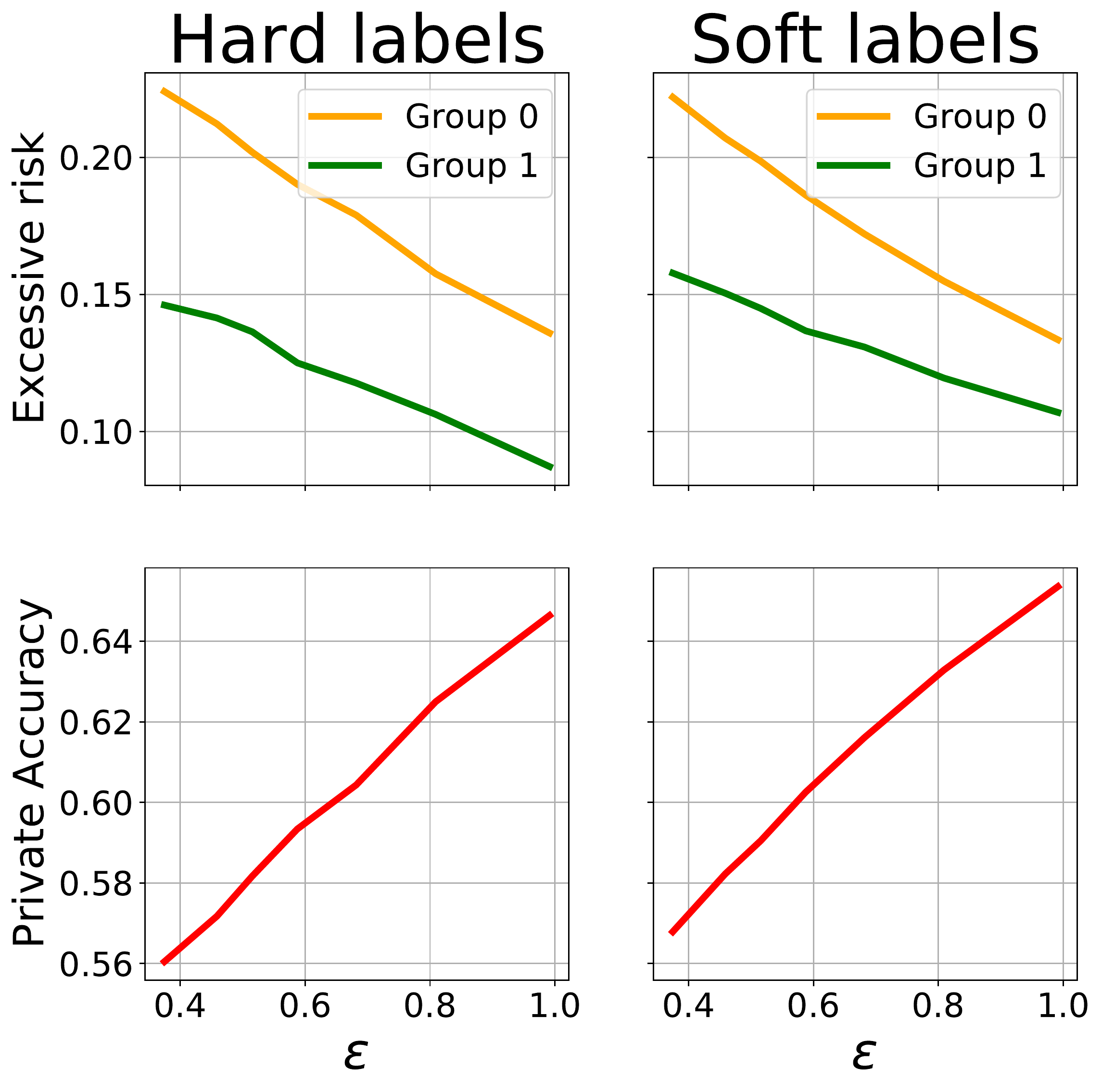}
\caption{}
\end{subfigure}
\begin{subfigure}[b]{0.4\textwidth}
\includegraphics[width=\linewidth]{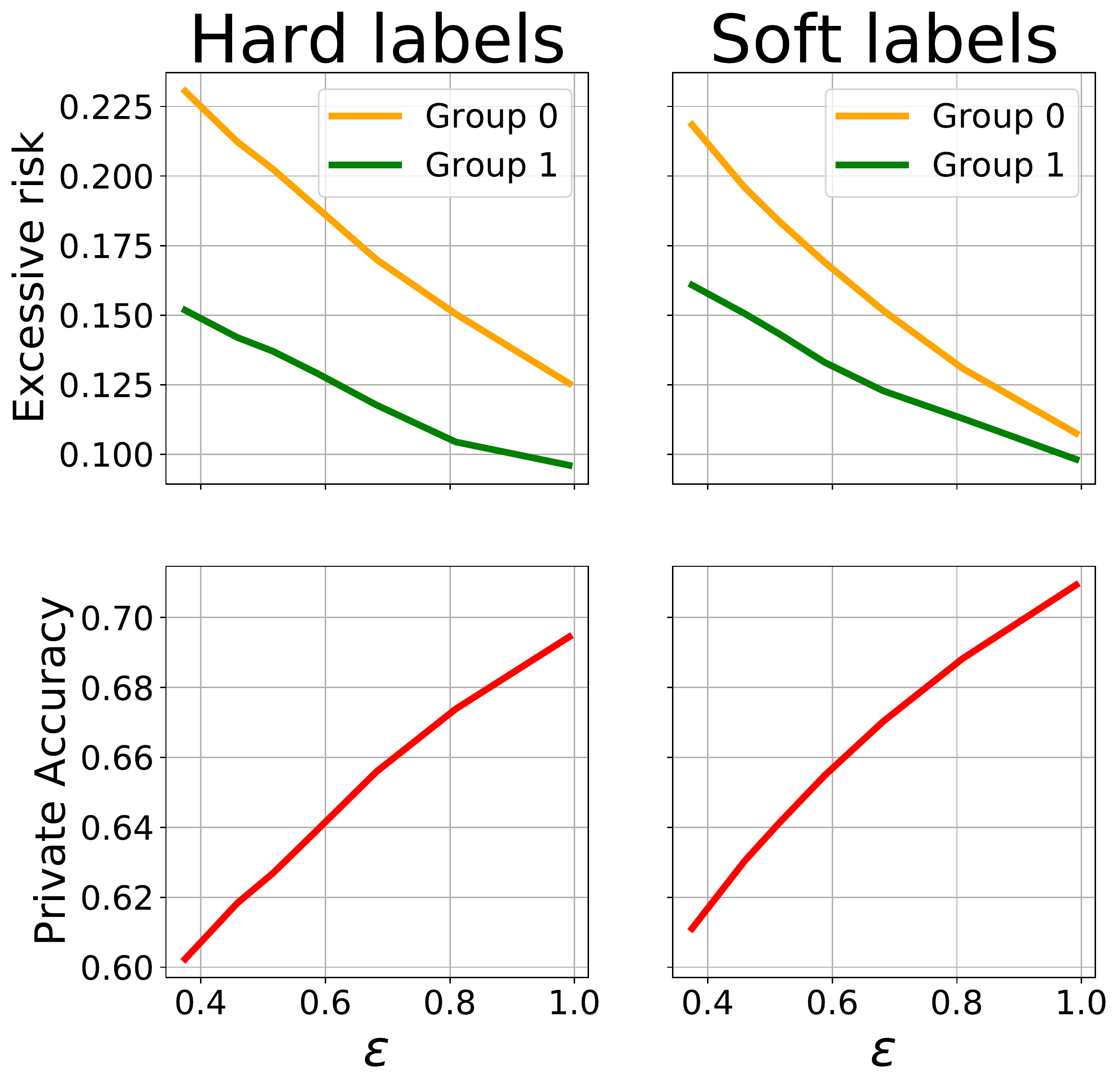}
\caption{}
\end{subfigure}
\begin{subfigure}[b]{0.4\textwidth}
\includegraphics[width=\linewidth]{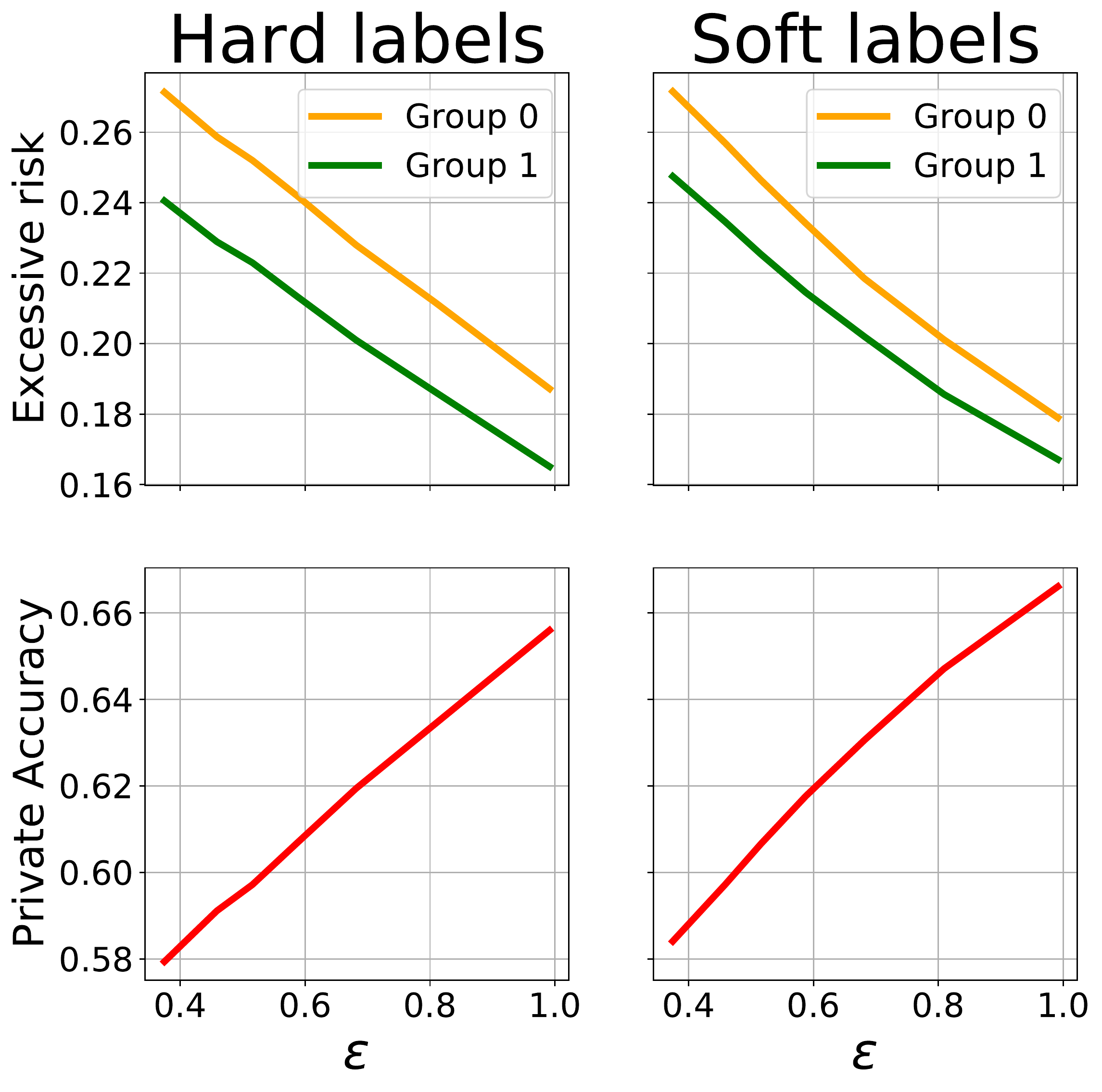}
\caption{}
\end{subfigure}
\begin{subfigure}[b]{0.4\textwidth}
\includegraphics[width=\linewidth]{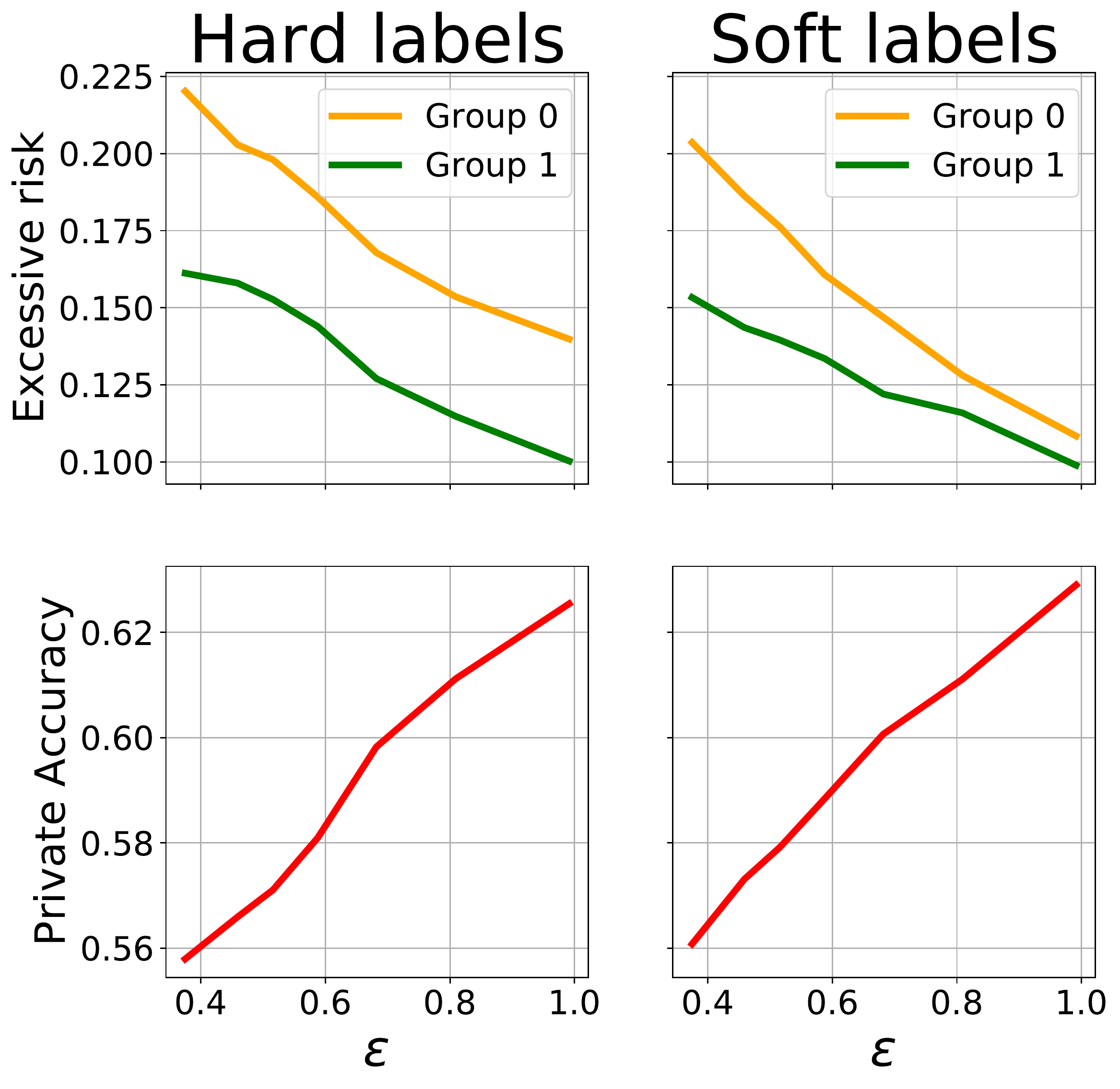}
\caption{}
\end{subfigure}
\begin{subfigure}[b]{0.6\textwidth}
\includegraphics[width=\linewidth]{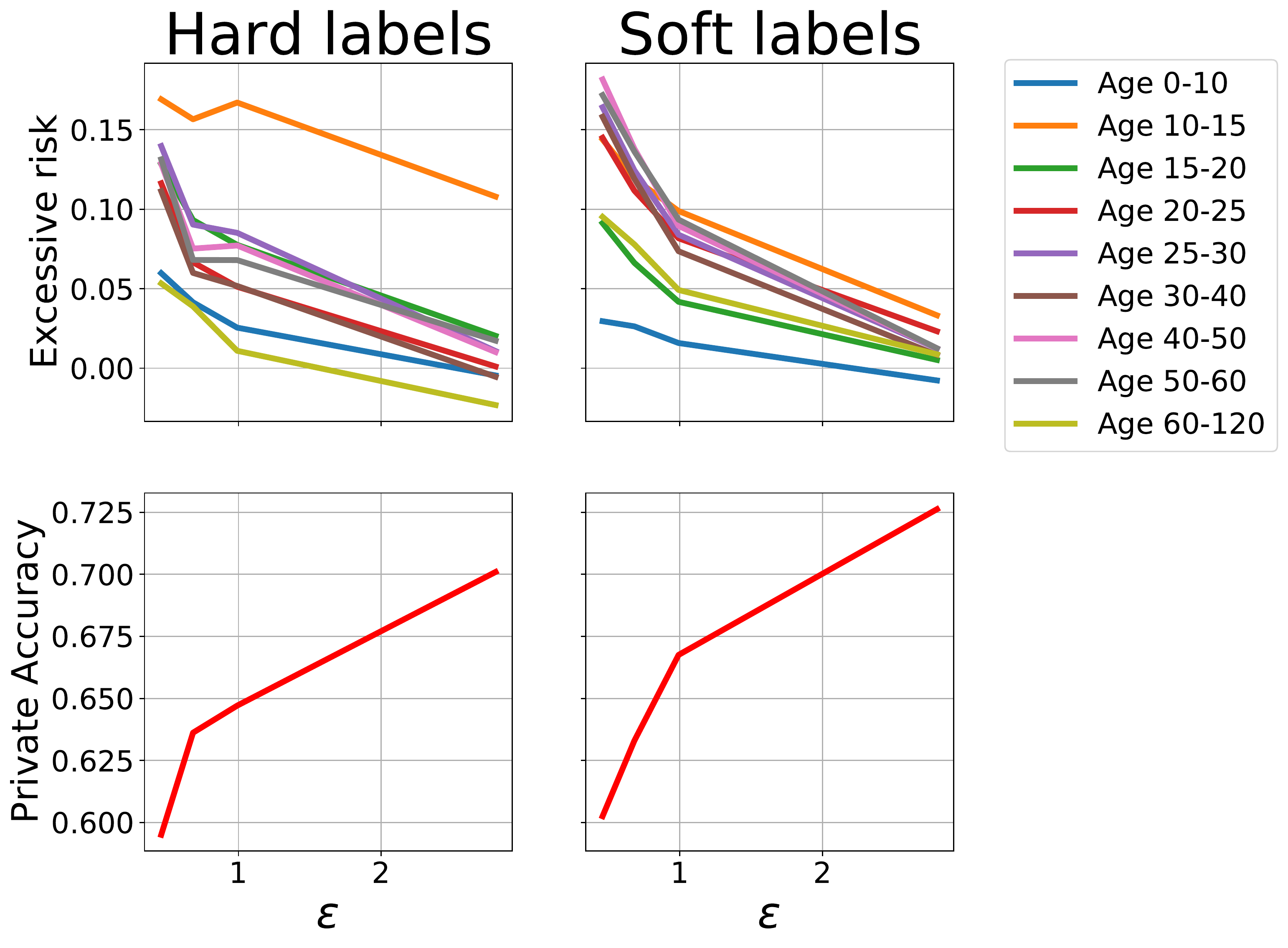}
\caption{}
\end{subfigure}
\caption{Comparison between training privately PATE with hard labels and soft labels in term of fairness (top subfigures) and utility(bottom subfigures) on (a) Bank, (b) Credit card, (c) Income (d) Parkinsons, (e) UTKFace dataset. Here for each dataset, the number of teachers $k=20$. }
\label{fig:mitigation_solution_K_20}
\end{figure*}

\begin{figure*}
\centering
\begin{subfigure}[b]{0.4\textwidth}
\includegraphics[width=\textwidth]{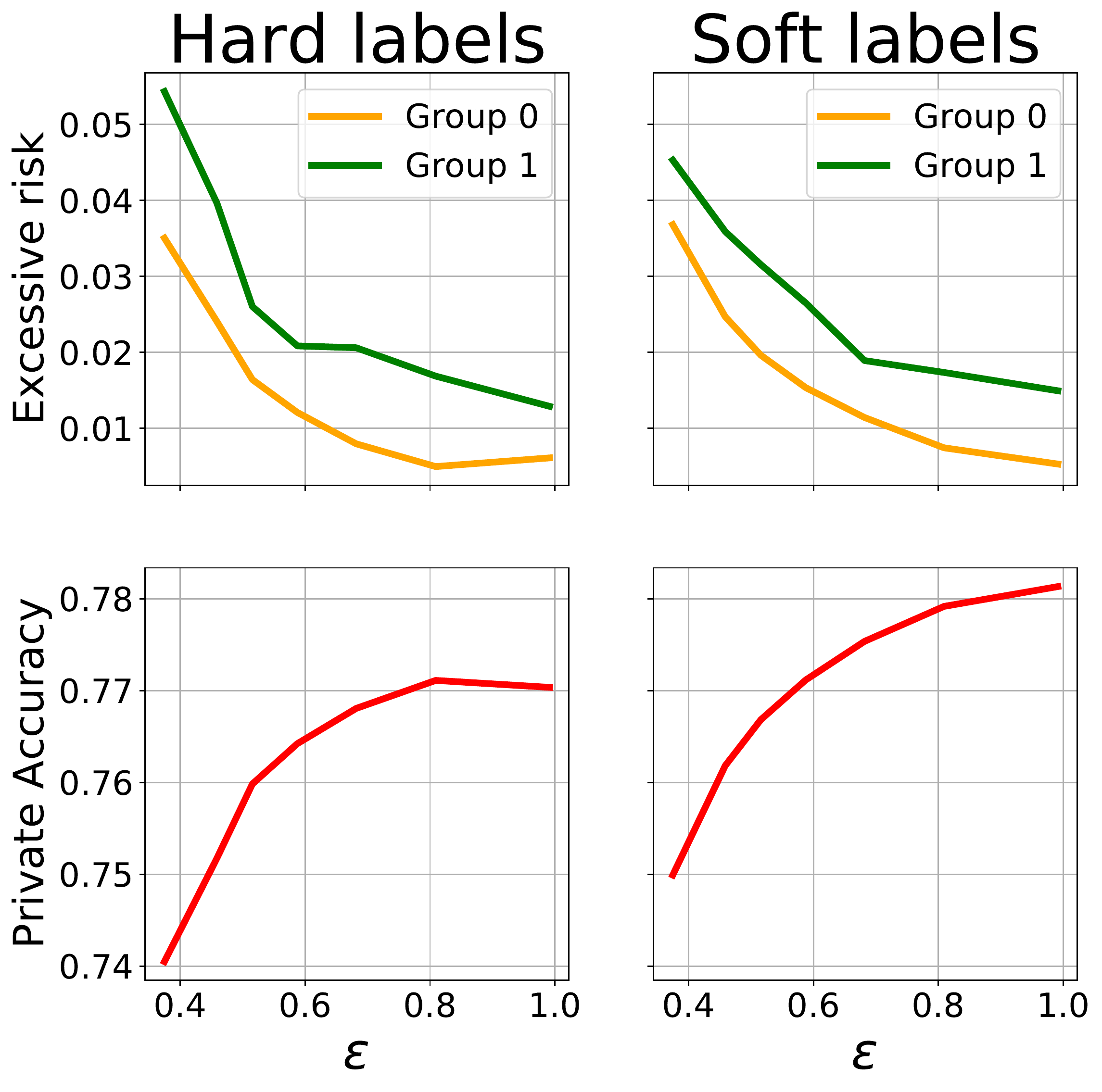}
\caption{}
\end{subfigure}
\begin{subfigure}[b]{0.4\textwidth}
\includegraphics[width=\linewidth]{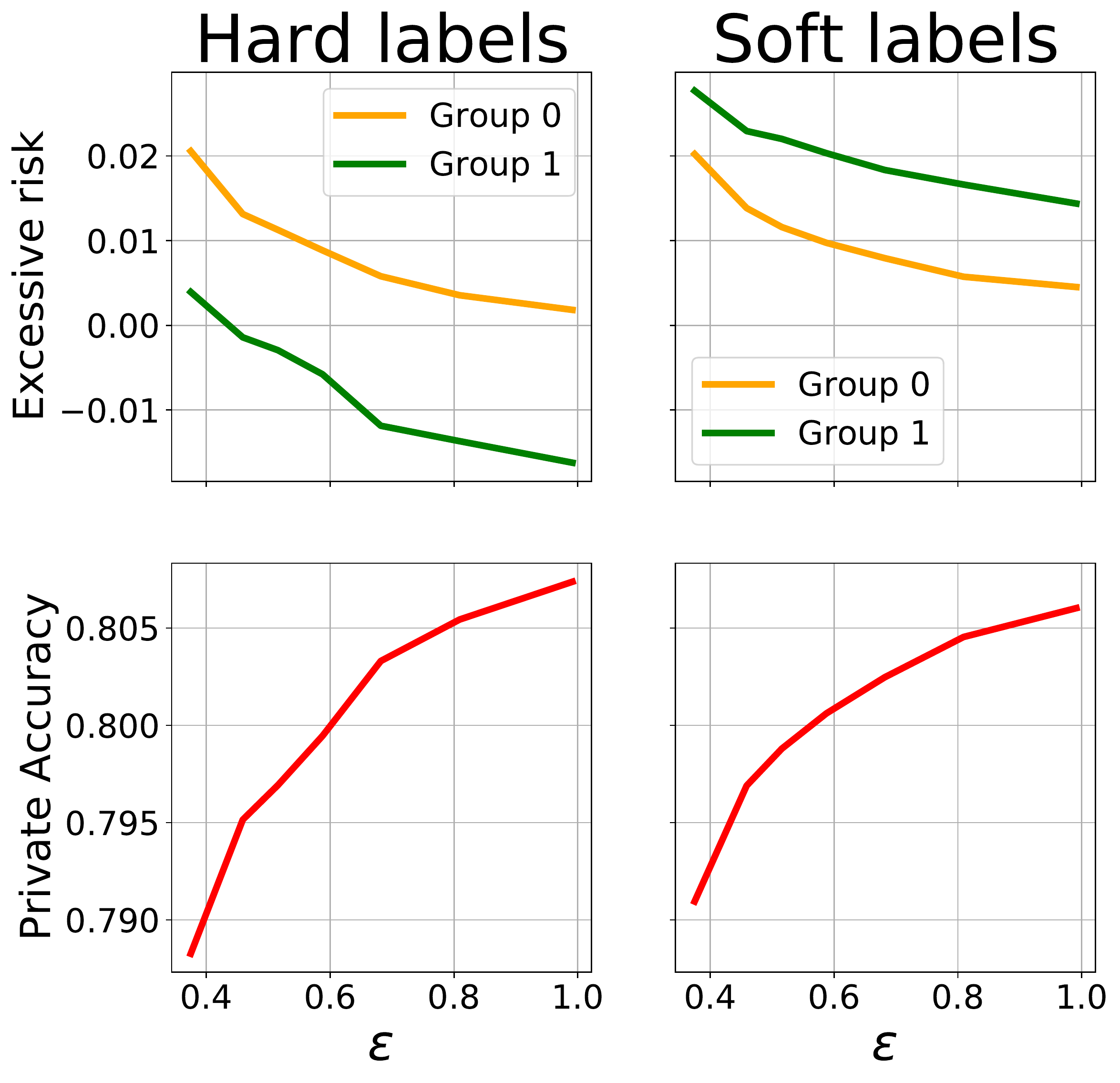}
\caption{}
\end{subfigure}
\begin{subfigure}[b]{0.4\textwidth}
\includegraphics[width=\linewidth]{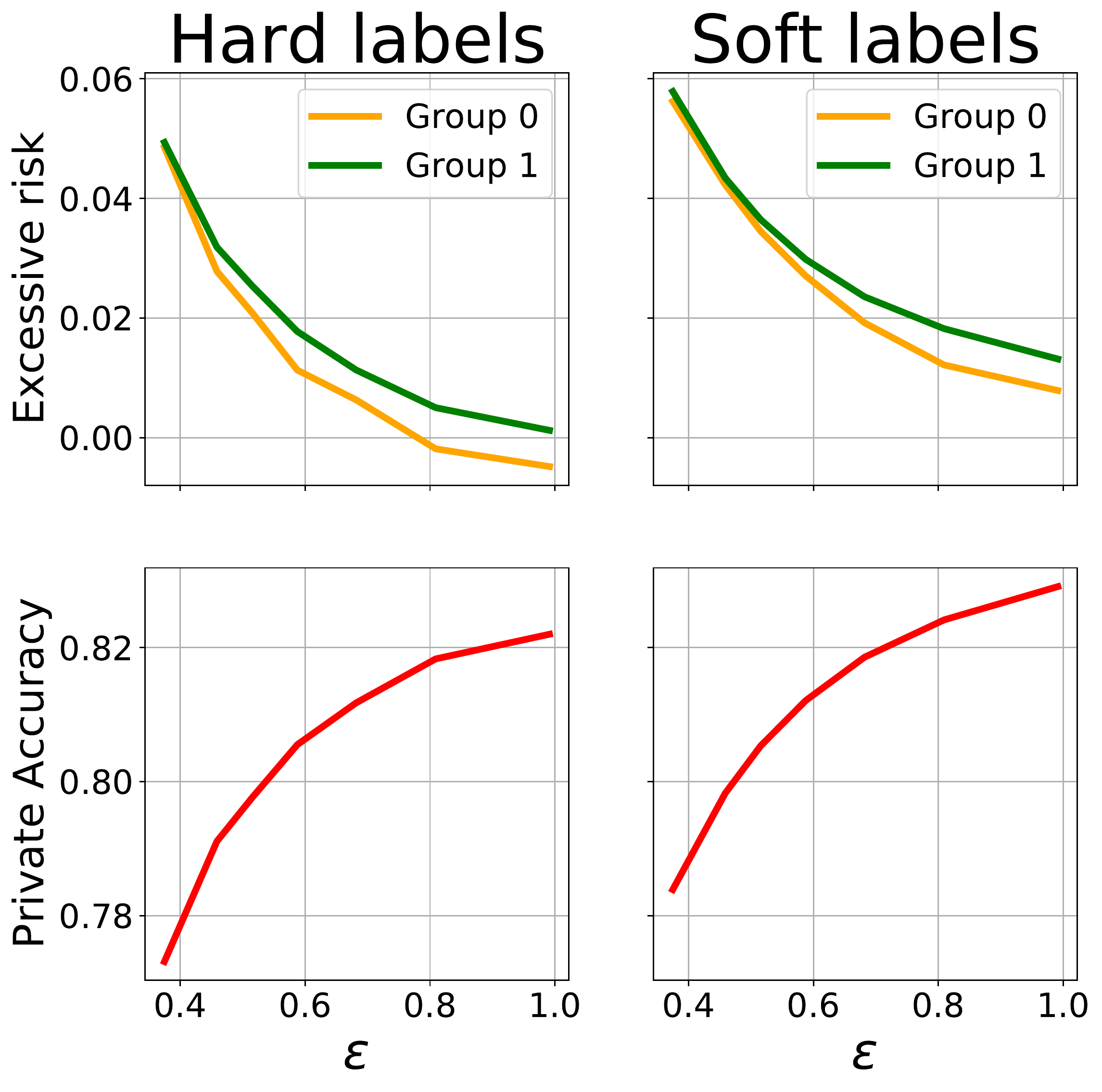}
\caption{}
\end{subfigure}
\begin{subfigure}[b]{0.4\textwidth}
\includegraphics[width=\linewidth]{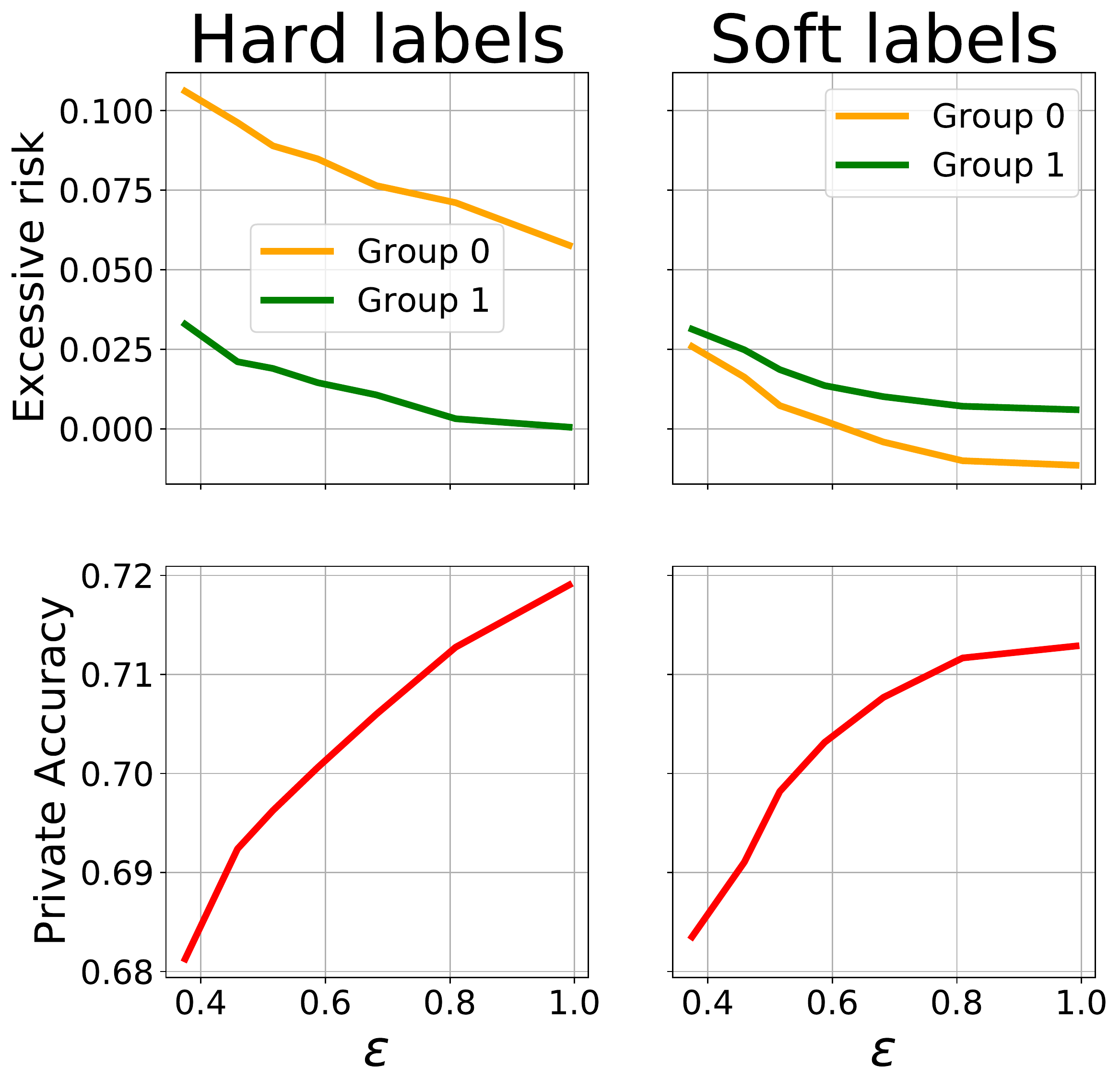}
\caption{}
\end{subfigure}

\caption{Comparison between training privately PATE with hard labels and soft labels in term of fairness (top subfigures) and utility(bottom subfigures) on (a) Bank, (b) Credit card, (c) Income, and (d) Parkinsons Here for each dataset, the number of teachers $k=150$.  }
\label{fig:mitigation_solution_K_150}
\end{figure*}

\newpage